\documentclass[11pt]{article}

\usepackage[bookmarks,colorlinks]{hyperref}
\usepackage{xcolor}
\definecolor{darkblue}{rgb}{0,0.22,0.66}
\hypersetup{citecolor=darkblue,linkcolor=darkblue,urlcolor=darkblue}

\usepackage[a4paper,margin=1in]{geometry}

\usepackage{amsthm}
\usepackage{amsmath,scalerel}
\usepackage{microtype}
\usepackage{needspace}

\usepackage[authoryear,sort]{natbib}
\bibpunct{(}{)}{;}{a}{,}{,} 

\usepackage{amsfonts,amssymb}
\usepackage[scr=rsfs]{mathalpha} 
\usepackage{mathtools}
\usepackage{bm}

\usepackage{xurl,xcolor}

\usepackage{enumitem}

\setlist[enumerate]{
  topsep=3pt,
  itemsep=3pt,
  parsep=0pt,
  partopsep=0pt,
  leftmargin=*,   %
  align=right,
  labelsep=0.600em,
  itemindent=0.00em , %
  labelindent=0.6em,
}

\setlist[itemize]{
  topsep=3pt,
  itemsep=3pt,
  parsep=0pt,
  partopsep=0pt,
  leftmargin=*,   %
  align=right,
  labelsep=0.600em,
  itemindent=0.00em , %
  labelindent=0.6em,
}

\usepackage{tabularx,multirow,array,booktabs}
\usepackage{colortbl}
\usepackage{graphicx}
\graphicspath{{figures/}}%
\usepackage{float}
\usepackage{caption}
\newtheoremstyle{paperplain}%
  {8pt plus 2pt minus 2pt}
  {8pt plus 2pt minus 2pt}
  {\itshape}
  {}
  {\bfseries}
  {.}
  {.5em}
  {}
\newtheoremstyle{paperdefinition}%
  {8pt plus 2pt minus 2pt}
  {8pt plus 2pt minus 2pt}
  {\normalfont}
  {}
  {\bfseries}
  {.}
  {.5em}
  {}
\newtheoremstyle{paperremark}%
  {6pt plus 2pt minus 1pt}
  {6pt plus 2pt minus 1pt}
  {\normalfont}
  {}
  {\itshape}
  {.}
  {.5em}
  {}

\theoremstyle{paperplain}
\newtheorem{theorem}{Theorem}[section]%

\newtheorem*{example*}{Example}
\newtheorem{corollary}[theorem]{Corollary} %
\newtheorem{lemma}[theorem]{Lemma}
\newtheorem*{lemma*}{Lemma}

\newtheorem{proposition}[theorem]{Proposition}
\theoremstyle{paperdefinition}
\newtheorem{definition}[theorem]{Definition}
\theoremstyle{paperremark}

\newtheorem{remark}[theorem]{Remark}
\newtheorem*{remark*}{Remark}
\usepackage{etoolbox,xparse}
\newcommand{\myMFabc}[4]{\expandafter#1\csname#3#4\endcsname{{#2{#4}}}}
\newcommand{\myMFcmd}[4]{\expandafter#1\csname#3#4\endcsname{{#2{\csname#4\endcsname}}}}

\newcommand{\MFabc}[3][\newcommand]{
    \def\doOld##1##2{\forcsvlist{\myMFabc{#1}{##1}{##2}}{#3}}
    \providecommand{\do}{do}
    \RenewDocumentCommand \do { >{\SplitList{,}} m } { \doOld##1 }
    \docsvlist{#2}
}
\newcommand{\MFcmd}[3][\newcommand]{
    \def\doOld##1##2{\forcsvlist{\myMFcmd{#1}{##1}{##2}}{#3}}
    \providecommand{\do}{do}
    \RenewDocumentCommand \do { >{\SplitList{,}} m } { \doOld##1 }
    \docsvlist{#2}
}

\usepackage{dsfont}
\newcommand{\bmzero}{{\bm{0}}}

\let\one\bbone
\MFabc{ {\mathfrak,frak}, }{T,u,U,o,O,E,m}
\MFabc{ {\mathbb, }, }{A,N,Z,R,C,Q}

\newcommand{\hatbm}[1]{\widehat{\bm{#1}}}
\newcommand{\tildebm}[1]{\widetilde{\bm{#1}}}
\newcommand{\bmcal}[1]{\bm{\mathcal{#1}}}
\newcommand{\caltilde}[1]{\mathcal{\widetilde{#1}}}
\newcommand{\calhat}[1]{\mathcal{\widehat{#1}}} 
\newcommand{\scrtilde}[1]{\widetilde{\mathscr{#1}\mspace{1mu}\mspace{-1mu}}}
\newcommand{\scrhat}[1]{\mathscr{\widehat{#1}\mspace{1mu}\mspace{-1mu}}}
\newcommand{\bmcalhat}[1]{\bm{\mathcal{\widehat{#1}}}}
\newcommand{\bmcaltilde}[1]{\bm{\mathcal{\widetilde{#1}}}}
\MFabc{ {\mathcal,cal}, {\mathscr,scr}, {\mathbb,bb}, {\bmcal,bmcal}, {\bmcal,calbm}, {\caltilde,caltilde}, {\caltilde,tildecal}, {\calhat,calhat}, {\calhat,hatcal}, {\scrtilde,scrtilde}, {\scrtilde,tildescr}, {\scrhat,scrhat}, {\scrhat,hatscr}, {\bmcalhat,bmcalhat}, {\bmcalhat,bmhatcal}, 
{\bmcalhat,calbmhat}, {\bmcalhat,calhatbm}, {\bmcalhat,hatbmcal}, {\bmcalhat,hatcalbm}, {\bmcaltilde,bmcaltilde}, {\bmcaltilde,bmtildecal}, {\bmcaltilde,calbmtilde}, {\bmcaltilde,caltildebm}, {\bmcaltilde,tildebmcal}, {\bmcaltilde,tildecalbm}}{A,B,C,D,E,F,G,H,I,J,K,L,M,N,O,P,Q,R,S,T,U,V,W,X,Y,Z}

\MFabc{{\bm,bm}, {\mathsf,sf}, {\widehat,hat}, {\widetilde,tilde}, {\hatbm,hatbm}, {\hatbm,bmhat}, {\tildebm,tildebm}, {\tildebm,bmtilde}}{a,b,c,d,e,f,g,h,i,j,k,l,m,n,o,p,q,r,s,t,u,v,w,x,y,z,A,B,C,D,E,F,G,H,I,J,K,L,M,N,O,P,Q,R,S,T,U,V,W,X,Y,Z}

\MFcmd{{\bm,bm}, {\widehat,hat}, {\widetilde,tilde}, {\tildebm, tildebm}, {\tildebm, bmtilde}, {\hatbm, hatbm}, {\hatbm, bmhat}}{alpha,beta,gamma,delta,epsilon,zeta,eta,theta,iota,kappa,lambda,mu,nu,xi,omicron,pi,rho,sigma,tau,upsilon,phi,chi,psi,omega,Alpha,Beta,Gamma,Delta,Epsilon,Zeta,Eta,Theta,Iota,Kappa,Lambda,Mu,Nu,Xi,Omicron,Pi,Rho,Sigma,Tau,Upsilon,Phi,Chi,Psi,Omega,varrho,varphi,vartheta,varepsilon,varsigma,ell}

\newcommand{\actdef}[1]{\expandafter\def\csname#1\endcsname{{\ensuremath{\mathtt{#1}}}}}
\forcsvlist{\actdef}{ReLU, LReLU, LeakyReLU, ELU, GELU, SiLU, Softplus, dGELU, dSiLU, dSoftplus, Tanh, Sigmoid, Arctan, Softsign, SRS, dSRS, Swish, dSwish, Mish, dMish, SELU, CELU, dSELU, Sin,SinLU, SinTU, PSinTU, sine, cosine, Sine, Cosine, EUAF}

\newlength{\myLength}

\def\ts{\textsf{\textup{T}}}
\let\ts\top

\newcommand{\proofstep}[2]{%
  \par\medskip
  \Needspace{4\baselineskip}%
  \noindent\textbf{Step #1: #2}\par\smallskip
}

\AtBeginEnvironment{theorem}{\Needspace{8\baselineskip}}
\AtBeginEnvironment{warrentheorem}{\Needspace{8\baselineskip}}
\AtBeginEnvironment{corollary}{\Needspace{7\baselineskip}}
\AtBeginEnvironment{lemma}{\Needspace{7\baselineskip}}
\AtBeginEnvironment{lemma*}{\Needspace{7\baselineskip}}
\AtBeginEnvironment{proposition}{\Needspace{8\baselineskip}}
\AtBeginEnvironment{definition}{\Needspace{7\baselineskip}}
\AtBeginEnvironment{example}{\Needspace{6\baselineskip}}
\AtBeginEnvironment{example*}{\Needspace{6\baselineskip}}
\AtBeginEnvironment{remark}{\Needspace{5\baselineskip}}
\AtBeginEnvironment{remark*}{\Needspace{5\baselineskip}}
\AtBeginEnvironment{proof}{\Needspace{4\baselineskip}}

\newenvironment{keywords}{\par \noindent\textbf{Keywords}:}{\par}

\usepackage[left,mathlines]{lineno}
\usepackage{refcount}

\definecolor{mylinenumbercolor}{HTML}{BEBEBE}

\makeatletter
\newcommand*\patchAmsMathEnvironmentForLineno[1]{%
	\expandafter\let\csname old#1\expandafter\endcsname\csname #1\endcsname
	\expandafter\let\csname oldend#1\expandafter\endcsname\csname end#1\endcsname
	\renewenvironment{#1}%
	{\linenomath\csname old#1\endcsname}%
	{\csname oldend#1\endcsname\endlinenomath}}%
\newcommand*\patchBothAmsMathEnvironmentsForLineno[1]{%
	\patchAmsMathEnvironmentForLineno{#1}%
	\patchAmsMathEnvironmentForLineno{#1*}}%
\patchBothAmsMathEnvironmentsForLineno{equation}%
\patchBothAmsMathEnvironmentsForLineno{align}%
\patchBothAmsMathEnvironmentsForLineno{flalign}%
\patchBothAmsMathEnvironmentsForLineno{alignat}%
\patchBothAmsMathEnvironmentsForLineno{gather}%
\patchBothAmsMathEnvironmentsForLineno{multline}%
\makeatother

\let\epsilon\varepsilon

\definecolor{mygray}{RGB}{230,230,230}
\definecolor{myorange}{HTML}{ff7f0e}

\let\cite\citep
\usepackage{doi}

\makeatletter
\long\def\@makefntext#1{\@setpar{\@@par\@tempdima \hsize 
		\advance\@tempdima-15pt\parshape \@ne 15pt \@tempdima}\par
	\parindent 2em\noindent \hbox to \z@{\hss{\textsuperscript{\@thefnmark}} \hfil}#1}
\newlength\aftertitskip     \newlength\beforetitskip
\newlength\interauthorskip  \newlength\aftermaketitskip
\def\maketitle{\par
	\begingroup
	\def\thefootnote{\color{black}\fnsymbol{footnote}}
	\def\@makefnmark{\hbox to 0pt{$^{\@thefnmark}$\hss}}
	\@maketitle \@thanks
	\endgroup
	\setcounter{footnote}{0}
	\let\maketitle\relax \let\@maketitle\relax
	\gdef\@thanks{}\gdef\@author{}\gdef\@title{}\let\thanks\relax}

\def\@startauthor{\noindent \normalsize\bf}
\def\@endauthor{}
\def\@starteditor{\noindent \small {\bf Editor:~}}
\def\@endeditor{\normalsize}
\def\@maketitle{\vbox{\hsize\textwidth
		\linewidth\hsize \vskip \beforetitskip
		{\begin{center} \Large\bf \@title \par \end{center}} \vskip \aftertitskip
		{\def\and{\unskip\enspace{\rm and}\enspace}%
			\def\addr{\small\it}%
            \def\email{\hfill\small\ttfamily}%
			\def\name{\normalsize\bf}%
			\def\AND{\@endauthor\rm\hss \vskip \interauthorskip \@startauthor}
			\@startauthor \@author \@endauthor}
		\vskip \aftermaketitskip
}}
\makeatother
\def\aff{{\mathsf{Aff}}}

\colorlet{blue}{black}

\DeclareMathOperator{\rank}{rank}
\DeclareMathOperator{\VCdim}{VCdim}
\DeclareMathOperator{\Pdim}{Pdim}
\DeclareMathOperator{\Arch}{Arch}
\DeclareMathOperator{\med}{med}
\DeclareMathOperator{\width}{width}
\DeclareMathOperator{\depth}{depth}

\hypersetup{
  pdftitle={Sharp Approximation Rates for Neural Networks with Affine Latent Parameterizations},
  pdfauthor={Shijun Zhang},
}

\title{Sharp Approximation Rates for Neural Networks
with Affine Latent Parameterizations}

\author{\name Shijun Zhang
  \email shijun.zhang@polyu.edu.hk\\
  \addr Department of Applied Mathematics\\
  \addr Hong Kong Polytechnic University
}

\begin{document}
\maketitle

\begin{abstract}
	Many parameter-efficient methods generate the parameters of a large neural
	network from a low-dimensional latent representation. Given an architecture
	$\Phi$ with $P_\Phi$ parameter slots, we write
	$\bm{\theta}_f=\mathcal{G}(\bm{\xi}_f)$, where
	$\mathcal{G}\colon\mathbb{R}^M\to\mathbb{R}^{P_\Phi}$ is a parameter generator
	and $\bm{\xi}_f\in\mathbb{R}^M$ is a latent representation of the target
	function $f$. The architecture $\Phi$ and the generator $\mathcal{G}$ are
	shared across the entire target class, while each target $f$ is represented by its
	own latent vector $\bm{\xi}_f$, with
	$\Phi_{\mathcal{G}(\bm{\xi}_f)}$ approximating $f$. This framework encompasses
	hypernetworks, low-dimensional parameterizations, parameter-efficient
	adaptation, and model compression. Understanding the tradeoff between the
	latent dimension $M$ and the network budget $P$ is therefore fundamental to
	characterizing the expressive efficiency of these methods. We study this
	tradeoff for affine generators and fully connected ReLU architectures. More
	precisely, optimizing jointly over architectures $\Phi$ satisfying
	$P_\Phi\leq P$ and affine generators $\mathcal{G}:\mathbb{R}^M\to \mathbb{R}^{P_\Phi}$, we prove that the optimal
	worst-case uniform approximation error over the unit ball of
	$\alpha$-H\"older functions on $[0,1]^d$, where $0<\alpha\leq1$, has the sharp
	order
	$
	\bigl(P\min\{M,P\}\bigr)^{-\alpha/d}.
	$
	In particular, our result shows that even a fixed-dimensional latent
	space suffices to achieve vanishing approximation error as the network budget
	increases.
\end{abstract}


\section{Introduction}\label{sec:introduction}

Modern neural networks are often heavily overparameterized: the number of
weights used at inference can far exceed the number of degrees of freedom
needed to select a useful model for a particular task.  If the full parameter
vector can be generated from a much smaller latent vector, one can reduce
the storage required for each task and the dimension of the optimization
problem without shrinking the deployed architecture.  This separation is
especially natural when a single generator and network architecture are reused
across many targets, tasks, or environments.

A substantial body of empirical work indicates that this mechanism can be
effective.  Particularly direct examples include training in a fixed parameter
subspace and reconstructing a full parameter vector from a seeded linear
expansion \cite{li2018intrinsic,nooralinejad2023pranc}.  Related methods based
on weight tying, structured transforms, hypernetworks, and low-rank updates are
reviewed in Section~\ref{sec:context}.  Collectively, these studies suggest that
the relevant complexity is not described by a single parameter count: one must
distinguish the amount of target-dependent information from the size and
complexity of the fixed mechanism that decodes that information.

To formalize this separation, consider an architecture $\Phi$ with $P_\Phi$
weight and bias slots and a parameter generator
$\calG:\R^M\to\R^{P_\Phi}$.  The architecture and generator are fixed for an
entire target class.  For each target $f$, only a latent vector
$\bmxi_f\in\R^M$ is selected, and the resulting function is produced by
\begin{equation*}
 \bmxi_f\in\R^M
 \xmapsto{\quad  \calG\quad }
 \bmtheta_f=\calG(\bmxi_f)\in\R^{P_\Phi}
 \xmapsto{\quad  \Phi\quad }
 \bigl[\bmx\mapsto\Phi_{\bmtheta_f}(\bmx)\bigr].
\end{equation*}
Here $\Phi_{\bmtheta_f}$ denotes the function realized by $\Phi$ with complete
parameter vector $\bmtheta_f$, and $\bmx$ denotes its input.
We call $\R^M$ the \emph{latent parameter space} and $\bmxi_f$ the
target-dependent latent vector.  Because $\bmxi_f$ is the only quantity chosen
separately for each target, it is also the only target-dependent trainable
vector in the model.  Its $M$ coordinates carry information about the target,
whereas the $P_\Phi$ slots determine the size of the deployed decoder.
Counting only $M$ ignores the decoding mechanism, while counting only $P_\Phi$
treats all deployed parameters as if they varied independently with the
target.  A meaningful approximation theory must therefore account for both
resources and keep their quantifiers separate.

This accounting of two resources is meaningful only if the generator is
restricted.  If $\calG$ is arbitrary and its complexity is not charged, then
the latent dimension $M$ alone imposes essentially no expressive constraint.
Indeed, because $\R$ and $\R^{P_\Phi}$ have the same cardinality, an
unrestricted generator with $M=1$ can be chosen to map onto the full parameter
space.  Its generated family then coincides with the family obtained by
varying all $P_\Phi$ network parameters independently.
Continuity alone does not remove this degeneracy.  A continuous surjection
from $\R$ onto $\R^{P_\Phi}$ can be formed, for example, by joining closed
Peano curves whose images cover successively larger parameter cubes.  When
$P_\Phi>1$, no such space-filling map can be locally Lipschitz, because the
image of a locally Lipschitz map from $\R$ has Hausdorff dimension at most one.
For a finite target collection, even smoothness is insufficient: coordinatewise
Lagrange interpolation produces a polynomial curve through any prescribed
finite set of realizing parameter vectors.  Thus neither continuity for a
generator shared across a class nor smoothness for a finite task collection
makes the latent dimension alone a meaningful complexity measure.  A theory
of nonlinear generators must also control, for example, the description or
parameter complexity of $\calG$, its computational size, regularity,
stability, and the precision of the latent code.

We therefore study the simplest structured specialization of this framework:
an affine generator,
\begin{equation}\label{eq:intro-affine-map}
 \calG=\calA,
 \qquad
 \calA(\bmxi)=\bmA\bmxi+\bma.
\end{equation}
Here $\bmA\in\R^{P_\Phi\times M}$ and $\bma\in\R^{P_\Phi}$ are fixed across
the target class.
After the homogeneous lifting
$\widehat{\bmxi}:=(\bmxi,1)\in\R^{M+1}$, the map is linear in
$\widehat{\bmxi}$, since
$\calA(\bmxi)=[\bmA,\bma]\widehat{\bmxi}$.  This model includes training in a
fixed subspace, fixed weight sharing, and affine expansions based on seeds or
fixed transforms.  Although hypernetworks, Mapping Networks, and many low-rank
adaptation schemes are nonlinear in their latent variables, they motivate the
same question concerning both budgets.  The affine case is already nontrivial
and imposes a transparent geometric restriction: all generated parameter vectors
lie in a single affine subspace of dimension at most
$\min\{M,P_\Phi\}$.  Any expressivity beyond this affine dimension must
therefore be supplied by the fixed ReLU decoder.  Throughout the paper, an
\emph{affine latent parameterization} means precisely the shared map in
\eqref{eq:intro-affine-map}, which takes a latent vector to the
complete parameter vector of the realized network.  Figure~\ref{fig:parameter-flow}
illustrates this setup.

\begin{figure}[ht]
\centering
\includegraphics[width=0.850\textwidth]{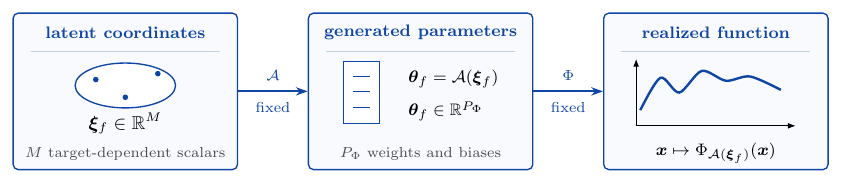}
\caption{An affine latent parameterization.  Only the $M$-dimensional latent
vector $\bmxi_f$ changes with the target.  The shared affine map $\calA$
generates all $P_\Phi$ network parameters, and the shared architecture $\Phi$
realizes the approximation.}
\label{fig:parameter-flow}
\end{figure}

The closest predecessors in approximation theory treat the two endpoints of
our resource model separately.  Approximation results for very deep networks
control the complete network size but allow every parameter to depend on the
target
\cite{yarotsky18a,shijun:optimal:rate:in:width:and:depth}.  At the opposite
endpoint, intrinsic parameter approximation controls the number of scalars
that depend on the target while allowing the fixed decoder to grow without an
independent parameter budget; repeated composition of a shared ReLU block
provides a complementary fixed-dimensional parameter-sharing construction
\cite{shijun:intrinsic:parameters,shijun:RCNet}.
Neither endpoint simultaneously restricts the latent dimension and the number
of generated parameter slots.  Our goal is to connect these resource models
and determine the sharp joint dependence on the two budgets.  This leads to
the central question:
\begin{quote}
For a single affine latent parameterization shared by a whole function class,
how does the optimal worst-case approximation error depend jointly on the
latent dimension $M$ and the budget $P$ for all generated weight and bias slots?
\end{quote}

Our main result gives a sharp answer in both resources, uniformly over the
target class.  We fix the pair $(\Phi,\calA)$ before the target is selected,
count every dense weight and bias slot of the generated network, and allow only
the latent vector $\bmxi_f$ to depend on $f$.  For the unit ball of $\alpha$-H\"older
functions on $[0,1]^d$, where $0<\alpha\le1$, once $P$ exceeds the explicit
dimension-dependent threshold in Theorem~\ref{thm:upper} and
$\min\{M,P\}\ge4$, the optimal worst-case error is, up to constants depending
only on $d$ and $\alpha$,
\[
 [P\min\{M,P\}]^{-\alpha/d} .
\]
Thus, when $M\le P$, the two resources multiply and the rate is
$(PM)^{-\alpha/d}$.  When $M\ge P$, the affine image is limited by the ambient
parameter dimension, and the rate saturates at $P^{-2\alpha/d}$.  The upper
bound is achieved by fully connected ReLU networks whose width depends only on
$d$ and whose depth grows at most linearly with $P$.  The matching lower bound
holds for every admissible fully connected ReLU architecture with at most $P$
parameter slots, without any further restriction on width or depth.  Thus the
result characterizes the best use of one affine latent space shared across the
target class, rather than the performance of a particular construction.
Figure~\ref{fig:PM-regimes} summarizes the two regimes.

\begin{figure}[ht]
\centering
\includegraphics[width=0.60\textwidth]{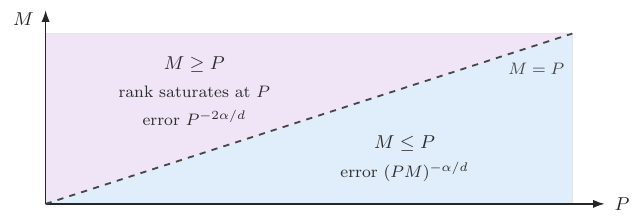}
\caption{The sharp joint law in $P$ and $M$ in the nondegenerate regime of
Theorem~\ref{thm:upper}.  For $M\le P$, the effective scale is
$PM$; in particular, each fixed $M\ge4$ yields the rate $P^{-\alpha/d}$.  For
$M\ge P$, the affine image is limited by the ambient parameter dimension, and
the rate saturates at $P^{-2\alpha/d}$.}
\label{fig:PM-regimes}
\end{figure}

The regime of fixed latent dimension is especially noteworthy.  If
$M=M_0\ge4$ is independent of $P$, then, whenever the hypotheses of
Theorem~\ref{thm:upper} hold and $P\ge M_0$,
\[
 [P\min\{M_0,P\}]^{-\alpha/d}
 =M_0^{-\alpha/d}P^{-\alpha/d}.
\]
Consequently, a constant number of target-dependent coordinates is sufficient
for the worst-case error to converge algebraically to zero as the shared
decoder grows.  This decay is driven entirely by growth of the decoder budget,
which is shared across targets, rather than by growth of the trainable
dimension.
Although this rate is slower than the saturated rate $P^{-2\alpha/d}$ available
when $M\ge P$, it shows that a genuinely low-dimensional affine latent space
can support substantial and provably optimal approximation power.

The main contributions can be summarized as follows.
\begin{enumerate}[label={\textup{(\roman*)}}]
  \item We introduce a minimax formulation uniform over the target class in
  which the architecture and affine latent map are fixed before the target is
  selected, and the latent dimension $M$ and the parameter budget $P$ for the
  generated slots are accounted for separately.
  This formulation isolates target-dependent information from shared decoding
  complexity.

  \item We construct fully connected ReLU networks of fixed width depending
  only on the dimension.  Under the stated nondegenerate conditions, these
  networks attain the rate
  $[P\min\{M,P\}]^{-\alpha/d}$.  The construction tracks the depth, every dense
  parameter slot, and all dimension-dependent constants.  More generally, it
  yields a modulus-of-continuity estimate for arbitrary continuous targets and
  remains effective even when $M$ is fixed.

  \item We prove a matching lower bound over all fully connected architectures
  with input dimension $d$ and at most $P$ dense parameter slots, without
  any additional width or depth restriction.  The main ingredient is a
  pseudo-dimension estimate for ReLU networks with affine parameter tying,
  derived from affine rank reduction and polynomial sign pattern counting; a
  H\"older bump packing argument then converts this capacity bound into the
  matching approximation lower bound.
\end{enumerate}

The upper and lower bounds rely on complementary ingredients.  The upper bound
combines a serialized affine spline loader, fixed binary prefix extraction, a
spatial address network, and median boundary repair.  For the converse,
substituting
$\bmtheta=\calA(\bmxi)$ reduces the effective number of variables to
the rank of $\bmA$, which is at most $\min\{M,P_\Phi\}$.  A layerwise
semialgebraic argument then
bounds the pseudo-dimension, and localized H\"older bumps convert this capacity
bound into the matching approximation lower bound.  The constructive uniform
upper bound also implies the same estimate in $L^p([0,1]^d)$ for every
$1\le p<\infty$.  The matching lower bound proved here is for the uniform
norm.
The stated rate relies on exact real arithmetic: a latent coordinate may
contain a long binary stream, and the selection $f\mapsto\bmxi_f$ is
discontinuous.  The theorem therefore concerns representation power rather
than numerical stability or optimization;
Section~\ref{sec:exact-real-scope} discusses this scope in detail.

The remainder of the paper is organized as follows.
Section~\ref{sec:context} reviews empirical evidence for affine and nonlinear
latent parameterizations, relates the problem to existing approximation and
capacity theory, and explains the interpretation and limitations of the
joint law in $P$ and $M$.  Section~\ref{sec:model} introduces the network model,
the parameter budget classes, and the main upper and lower bounds.
Section~\ref{sec:upper-proof} gives the constructive proof of the upper bound
for continuous functions, from the fixed-width modules through the final
choice of the integer budgets.  Section~\ref{sec:lower-proof} proves the
matching H\"older lower bound through pseudo-dimension and bump packing
arguments.  Finally, Section~\ref{sec:conclusion} summarizes the conclusions
and discusses directions for further work.

\section{Related work, interpretation, and scope}\label{sec:context}

This section reviews the most closely related empirical and theoretical work,
then interprets the joint law in $P$ and $M$ and clarifies its scope.
Sections~\ref{sec:affine-literature} and
\ref{sec:nonlinear-generators} distinguish affine parameterizations from
their nonlinear relatives.  Section~\ref{sec:approximation-coding-capacity}
places the result within approximation and capacity theory, while
Sections~\ref{sec:storage-decoding} through \ref{sec:exact-real-scope} discuss
resource accounting, optimization geometry, and the assumptions on exact real
arithmetic.  None of this background is used as an assumption in the proofs.

\subsection{Affine latent parameterizations: parameter prediction,
weight tying, and subspace training}
\label{sec:affine-literature}

The examples in this subsection fall into three broad categories: parameter
prediction and weight tying, training in a fixed linear subspace, and training
in a subspace learned jointly for a single task.  For our theorem, the relevant
distinction is whether the same affine map is fixed before the target function
is selected.

\subsubsection*{Parameter prediction and weight tying}
An early parameter-prediction study learned a small subset of the weights and
reconstructed the remainder from linear predictors; in its best reported case,
more than $95\%$ of the weights were predicted without loss of accuracy
\cite{denil2013predicting}.  This finding supports substantial redundancy in
trained weight tensors, but the predictors and selected coordinates are
task-dependent, and the work gives no uniform approximation guarantee over a
function class.

HashedNets impose a more rigid form of sharing: a signed hash maps many virtual
connections to a smaller collection of trainable scalars
\cite{chen2015hashed}.  Once the virtual weights are enumerated, this
construction is exactly a sparse affine map, each of whose rows contains one
signed nonzero entry.  Existing theory addresses approximation by random
linear sketches on well-conditioned low-dimensional input manifolds and local
recovery for a one-layer hashed model \cite{lin2019hashingtheory}; it does not
give the H\"older-class joint minimax rate considered here.
Frequency-sensitive hashing first transforms convolutional filters and then
shares their spectral coefficients \cite{chen2016freshnets}.  Thus it is close
to the present affine viewpoint, but its objective is empirical compression of
specific convolutional models rather than a worst-case rate for a prescribed
function class.

Beyond direct prediction and tying, structured expansions can reduce both
storage and multiplication cost.  Fastfood supplies a fixed structured random
transform built from diagonal matrices, permutations, and Hadamard transforms
\cite{le2013fastfood}.  Adaptive Deep Fried layers instead learn products of
several diagonal factors and are generally nonlinear jointly in those factors
\cite{yang2015deepfried}.  Discrete cosine transform (DCT) layers with sparse
corrections add a fixed dense transform to a very sparse trainable component
and retain useful accuracy even at extreme trainable sparsities
\cite{price2021subspaceoffset}.  FourierFT reconstructs a dense update from
selected Fourier coefficients \cite{gao2024fourierft}.  Once the transform and
coefficient locations are fixed, the latter two constructions are affine
parameter maps.  Their transforms, sparsity patterns, and statistical tasks
are prescribed differently, however, and none proves a joint minimax law in
the number of adjustable coefficients and the number of realized network
slots.

\subsubsection*{Fixed and learned subspaces}
Training in fixed random subspaces optimizes
$\bmtheta=\bmW_{\mathrm{sub}}\bmxi+\bmb_{\mathrm{sub}}$ and defines the
empirical intrinsic dimension as the smallest code dimension that attains a
chosen fraction of the performance of the full model \cite{li2018intrinsic}.
Experiments show that networks
with hundreds of thousands of weights can sometimes be optimized through only
hundreds or thousands of coordinates.  Subsequent experiments demonstrate
that a single reused random basis can impair optimization, while redrawing the
basis at every step and using separate projections for different modules can
improve both accuracy and speed \cite{gressmann2020random}.  Because the basis
changes during training, this method does not use a single fixed map $\calA$
and therefore lies outside our minimax model.  Structured Fastfood projections
for fine-tuning language models also reveal small intrinsic dimensions
specific to each task and suggest that larger pretrained models can require
fewer effective coordinates \cite{aghajanyan2021intrinsic}.

PRANC uses the exact linear expansion
$\bmtheta=\sum_{m=1}^M\xi_m\bmtheta_m$, with basis vectors regenerated from
pseudorandom seeds \cite{nooralinejad2023pranc}.  It is therefore one of the
closest empirical instances of the present algebraic model.  Balanced
multicomponent and multilayer networks reduce the number of trainable
parameters through a structured component decomposition
\cite{ZZZZ-24-MMNN}.  Their Fourier variant represents each component as a
trainable linear combination of fixed random sine bases within a low-rank
architecture \cite{ZZZZ-25-FMMNN}.  These models further illustrate that the
degrees of freedom depending on the target can be much fewer than the
parameters in the realized network, but they do not establish a minimax law
over an entire function class in the two budgets $(M,P)$.

Other coordinate-restricted approaches include FreezeNet and training
restricted to BatchNorm parameters
\cite{wimmer2020freezenet,frankle2021batchnorm}; these methods correspond to
coordinate projections rather than designed dense affine maps.
Sparsity inducing training with $\ell_1$ regularization for convolutional
networks addresses another form of parameter reduction
\cite{xu2020}.  In that setting, however, the active sparse structure is
learned for the task rather than generated by one affine map fixed across a
target class.

A complementary line of expressivity theory shows that training only
normalization scales and biases
in a sufficiently large random host can reconstruct smaller target networks,
with extensions to convolutional and residual architectures
\cite{giannou2023normalization,burkholz2024batchnorm}.  Those theorems exploit
the special normalization architecture and address target-network
reconstruction rather than uniform approximation of a H\"older ball by one
prescribed affine map.

Learned solution subspaces provide a further comparison.  Lines, curves, and
simplices of highly accurate solutions show that useful low-dimensional
regions can be optimized jointly with a target network
\cite{wortsman2021subspaces}.  In that setting, however, the endpoints defining
the subspace are themselves trainable objects specific to the task, whereas
our affine image is fixed before any target in the H\"older ball is selected.
Taken together, these works address task performance, optimization, or
reconstruction of prescribed target networks; they do not characterize the
worst-case approximation error jointly in $(M,P)$.

\subsection{Nonlinear generators and parameter-efficient adaptation}
\label{sec:nonlinear-generators}

We now turn from affine maps to generators that are typically nonlinear.  The
general mechanism remains
\[
 \calG:\R^M\to\R^{P_\Phi},
 \qquad
 \bmtheta=\calG(\bmxi),
\]
but $\calG$ need not be affine.  Hypernetworks, Mapping Networks, and most
low-rank adaptation methods therefore motivate the question of how
approximation depends on both budgets, although they do not belong to the
affine class studied here.  As explained in the Introduction, a sharp minimax
theory for such models would require a separate complexity or regularity budget
for $\calG$; otherwise $M$ alone does not control the generated family.

Hypernetworks generate the weights of a primary network from conditioning
information or a compact code and have been applied to recurrent and image models
\cite{ha2017hypernetworks}.  Theoretical work compares their parameter
complexity with that of direct embedding models and identifies settings in
which modular weight generation is advantageous
\cite{galanti2020modularity}.  Infinite-width analysis shows that widening only
the generator does not automatically produce benign kernel dynamics, whereas
a joint infinite-width limit yields an explicit hyperkernel
\cite{littwin2020infinitehypernetworks}.  Hypernetworks conditioned on graphs can
predict millions of weights for architectures unseen during training in a
single forward pass \cite{knyazev2021parameterprediction}.  These works study
learned nonlinear generators, conditional families, or optimization dynamics.
The generator may be frozen after training, but its map $\calG$ remains
nonlinear and is learned for a particular task.

Mapping Networks provide a particularly direct recent example of a nonlinear
parameter generator \cite{sen2026mapping}.  A compact trainable latent vector modulates
nontrainable, orthogonally initialized base matrices inside the mapping
machinery.  Composing these modulated layers produces a nonlinear map from the
latent vector to all parameters of a target model.  The reported experiments
cover image classification, detection of manipulated faces, semantic
segmentation, recurrent forecasting, and adaptation of a
pretrained model.  For one reported classifier, full training uses $537{,}994$
parameters, while latent variants use only a few thousand coordinates and
obtain competitive or improved accuracy.  In the reported Cityscapes
experiment, a target model with $1{,}734{,}803$ parameters is driven by roughly
eight thousand latent coordinates: pixel accuracy increases from $93.21\%$ to
$97.92\%$, while mean intersection-over-union changes from $0.4957$ to
$0.4623$ (and to $0.4823$ for the layerwise variant)
\cite{sen2026mapping}.  These results directly motivate the question involving
both budgets, but their generator is nonlinear.  The existence theorem based
on manifolds, the local solvability result, and the experiments do not
establish a worst-case approximation rate over a function class in $(M,P)$.
A closely related nonlinear compression model uses a frozen random generator
to constrain the complete parameter vector to a prescribed low-dimensional
nonlinear manifold \cite{thrash2025mcnc}.  Its experiments support the broader
$\calG(\bmxi)$ framework, but its generator is nonlinear and random, and its
analysis does not establish the class-wide affine minimax law considered here.

Low-rank adaptation provides another important class of nonlinear
parameterizations.
LoRA represents an update as a product of two thin matrices
\cite{hu2022lora}.  VeRA, NOLA, and RandLoRA further reuse fixed random
matrices or random bases and train scaling or combination coefficients
\cite{kopiczko2024vera,koohpayegani2024nola,albert2025randlora}.  Ordinary
two-factor LoRA is nonlinear jointly in its two trainable factors.  Even when
each NOLA factor is a linear combination of fixed bases, multiplying the two
factors again makes the complete update nonlinear in the combined code.  A
variant in which one factor is fixed, or a direct linear expansion of the
complete update in a fixed basis, would instead define an affine map from
$\R^M$ to $\R^{P_\Phi}$.  Kernel analyses explain why low-dimensional update
subspaces can preserve fine-tuning dynamics in suitable regimes
\cite{malladi2023kernel}.  Separately, theoretical work characterizes
reconstruction of target networks by low-rank updates and studies the landscape
in neural tangent and local Polyak--\L{}ojasiewicz regimes
\cite{zeng2024expressivelora,jang2024lorantk,liu2025loralandscape}.
These results depend on pretrained weights, low-rank factorization, or local
optimization assumptions; they do not yield the global function class rate
proved below.

Gradient low-rank projection provides a useful contrast: it compresses
gradients and optimizer state while retaining full-rank trainable weights
\cite{zhao2024galore}.  It is therefore an optimization-memory method rather
than a low-dimensional parameter map
$\calG:\R^M\to\R^{P_\Phi}$.

\subsection{Approximation theory, coding, and capacity}
\label{sec:approximation-coding-capacity}

Three theoretical viewpoints organize the comparison: exact codes with an
unbudgeted decoder, networks whose complete parameter vector may depend on the
target, and capacity bounds based on polynomial sign patterns.  Our theorem
draws on all three while budgeting both the code and the decoder.

Classical universal approximation theorems establish density on compact sets
\cite{Cybenko1989ApproximationBS,HORNIK1989359,HORNIK1991251}, while
quantitative approximation theory relates the error to width, depth,
smoothness, and the total number of network parameters
\cite{yarotsky2017,yarotsky18a}.  Related developments include sparse-network
and Sobolev-norm estimates
\cite{PETERSEN2018296,B_lcskei_2019,2019arXiv190207896G,10.3389/fams.2018.00014}, approximation rates
for broad classes of activation functions \cite{SIEGEL2020313}, and
high-order rates for shallow networks \cite{SIEGEL20221}.  Recent work gives
nearly optimal rates for shallow $\operatorname{ReLU}^k$ networks on Sobolev
spaces through Radon-transform techniques
\cite{doi:10.1137/24M1686693}.  Approximation spaces generated by deep networks
are characterized in \cite{2019arXiv190501208G}, and further quantitative
constructions appear in
\cite{shijun:NonlineArpprox,shijun:Characterized:by:Numer:Neurons,
shijun:smooth:functions,shijun:optimal:rate:in:width:and:depth}.

A closely related approximation result separates parameters that depend on the
target from parameters fixed for an entire function class
\cite[Theorem~2.2 and the discussion immediately following it]{shijun:intrinsic:parameters}.
For a target on $[0,1]^d$ with Lipschitz constant $\lambda$, it
constructs a ReLU network with $n+2$ intrinsic parameters and uniform error
at most $5\lambda\sqrt{d}2^{-n}$, and it extends the construction to general
continuous functions.  The decoder size is allowed to grow as required and is
not independently budgeted.  The present problem asks how that
exponential dependence on the code length changes once the decoder is also
subject to the budget $P$.  Repeated composition of a single ReLU block of
fixed size gives a complementary form of parameter sharing
\cite{shijun:RCNet}.  After the composition is unrolled, reuse of the same
block parameters becomes an affine tying map from a fixed number of independent
coefficients to a number of network slots proportional to the number of
repetitions.  The resulting $O(r^{-1/d})$ error for Lipschitz targets is
consistent with the fixed $M$ regime here, but that work does not determine the
joint dependence on arbitrary $(M,P)$ or prove a matching lower bound.

At the opposite endpoint, approximation theory for very deep networks with
joint width and depth budgets controls the
entire network size while allowing every parameter to depend on the
target \cite{yarotsky18a,yarotsky:2019:06,
shijun:optimal:rate:in:width:and:depth,
shijun:Characterized:by:Numer:Neurons,shijun:smooth:functions}.
The joint law in this paper interpolates between these two resource models.
Qualitative bounded-width universality also shows that width can be bounded
solely in terms of the input dimension when depth is allowed to grow
\cite{NIPS2017_7203}.  Such a theorem neither requires the parameter vector to
lie in one fixed affine image nor quantifies the joint code and slot budgets;
those are the constraints responsible for the rate studied here.
Other constructions of fixed width use depth as a progressive refinement
mechanism.  Multigrade ReLU networks successively reduce the residual
\cite{shijun:mgdl:relu:decay}, while shared architectures with intermediate
readouts achieve geometrically finer layerwise approximation scales
\cite{shijun:Geometric:Layer-wise:rates:sinReLU}.  These results clarify the
approximation role of depth, but they do not impose the affine latent budget
studied here.
Another line interprets deep architectures through dynamical systems and
derives approximation results from flows, interpolation, and controllability
\cite{WeinanE2017dynamicalsystems,LiLinShen2023DynamicalSystems,
doi:10.1137/23M1599744,2026arXiv260315363C}.  This viewpoint has also been
developed for invariant target classes
\cite{2022arXiv220808707L}.

A complementary computability result constructs a single fixed narrow
recurrent ReLU network that receives an encoding of a computable target
function as part of its input and approximates it after a controlled number of
recurrent iterations
\cite{bournez2025universal}.  That result emphasizes a universal decoder and
relates its iteration count to computational complexity.  Here, by contrast,
the code does not enter as an ordinary network input: it affinely generates the weights and
biases of a finite feedforward architecture, and both the code dimension $M$
and the resulting dense parameter count $P$ are optimized explicitly.

Random feature approximation provides another relevant affine special case:
the features are fixed and only the output coefficients vary.
Quantitative results are available for random ReLU features, random neural
networks and reservoirs, and networks with sampled weights
\cite{hsu2021randomrelu,gonon2023random,bolager2023sampling}.  The essential
differences are that those results typically average over a random draw, use
$L^2$ target classes or classes of Barron type, and do not exploit a decoder
of length proportional to $P$ to obtain a joint $PM$ law.  Lottery ticket
results instead encode the target through a discrete pruning mask
\cite{malach2020lottery,pensia2020subsetsum};
their information model is combinatorial rather than based on an affine code
in exact real arithmetic.
More broadly, Barron spaces and function spaces induced by flows, together
with population risk estimates for networks with two layers, provide a related
function space perspective
\cite{2019arXiv190608039E,e2020representation,Weinan2019}.

The theory of nonlinear widths distinguishes arbitrary, continuous, and stable
maps from targets to parameters
\cite{devore_1998,Devore89optimalnonlinear,Ingrid,cohen2020optimal}.
Very deep ReLU networks can exceed stable nonlinear width rates by using
discontinuous encodings of high precision
\cite{yarotsky18a,yarotsky:2019:06,shijun:optimal:rate:in:width:and:depth}.
For shallow neural-network variation spaces, sharp approximation rates,
metric entropy bounds, and $n$-width estimates provide a complementary
complexity-theoretic perspective \cite{SiegelXu2024SharpBounds}.
The upper construction below lies in this exact real regime.  For the lower
bound, classical counting of polynomial sign patterns for real parameters,
together with later nearly tight bounds for piecewise linear networks,
provides the relevant capacity tools
\cite{warren1968signpatterns,goldberg1995bounding,Bartlett98almostlinear,
JMLR:v20:17-612,anthony_bartlett_1999}.
We adapt these arguments to affine parameter tying, so the effective variable
count is $\rank(\bmA)$ rather than the ambient number $P_\Phi$ of parameter
slots.

\subsection{Resource accounting and endpoint regimes}
\label{sec:storage-decoding}

Having situated the theorem relative to the most closely related models and
theoretical results, we now interpret its two resource budgets.  In the
nondegenerate regime covered by the main theorem, with
$4\le M\le P$ and $P$ sufficiently large, the sharp H\"older rate is
$(PM)^{-\alpha/d}$.  The construction
explains this product directly: a constant fraction of the $M$ latent
coordinates store independent streams of quantized function increments, while
a fixed decoder of depth proportional to $P$ reads a number of digits
proportional to $P$ from each stream.  Consequently, it reconstructs a number
of grid values proportional to $PM$.

The result also connects two established endpoint theories.  When $M$ is
comparable to $P$, the rate is $P^{-2\alpha/d}$, matching the optimal rate for
very deep networks of fixed width established in
\cite[Theorems~1(a) and~2(b)]{yarotsky18a} and
\cite[Theorem~1.1]{shijun:optimal:rate:in:width:and:depth}.  When $M\ge4$ is
fixed and $P$ is sufficiently large, the rate becomes $P^{-\alpha/d}$ and
quantifies how a latent vector of fixed dimension over the real numbers can be
decoded by a shared decoder of increasing depth, in the spirit of
\cite{shijun:intrinsic:parameters}.  Once $M\ge P$, the affine rank is bounded
by the ambient parameter dimension, which explains the saturation encoded by
$\min\{M,P\}$.

The coefficients of the affine map are shared across the target class rather
than selected separately for each target.  In the absence of additional
structure, $\bmA$ and $\bma$ contain
$P_\Phi M+P_\Phi$ fixed scalar entries.  These entries are not charged to the
target-dependent trainable parameter count $M$, so the theorem does not
automatically imply compression of total storage.  Charging them, imposing
sparsity or fast transforms, or regenerating them from short random seeds
would define different resource models, as in several empirical constructions
discussed above.  A recent preprint studies a seeded deployment model with finite bit precision in
which the stored artifact consists of a short integer seed and a quantized
latent vector; the seed regenerates the fixed basis and initialization center
\cite{dhayalkar2026kilobyte}.  Its resource accounting and
experiments complement the minimax theorem over exact real parameters here,
but it does not provide a matching approximation law over a function class.
Within the deployed network, however, every weight and bias slot, including
slots fixed at zero, is counted in $P_\Phi$.

\subsection{Optimization geometry and conditioning}
\label{sec:optimization-geometry}

Beyond resource accounting, the affine map also induces a particular
optimization geometry.  The approximation problem itself does not prescribe
an algorithm for finding $\bmxi_f$, but the affine parameterization makes this
geometry explicit.  Let
\[
 \calA(\bmxi)=\bmA\bmxi+\bma
 \quad\text{and}\quad
 \widetilde\calL(\bmxi):=\calL\bigl(\calA(\bmxi)\bigr),
\]
where $\calL(\bmtheta)$ is a twice continuously differentiable loss in the
full parameter space and $\ts$ denotes transpose.  The chain rule gives
\begin{equation*}
 \nabla_{\bmxi}\widetilde\calL(\bmxi)
 =\bmA^\ts
   \nabla_{\bmtheta}\calL\bigl(\calA(\bmxi)\bigr),
 \qquad
 \nabla_{\bmxi}^2\widetilde\calL(\bmxi)
 =\bmA^\ts
   \nabla_{\bmtheta}^2\calL\bigl(\calA(\bmxi)\bigr)
   \bmA.
\end{equation*}
The pullback Hessian is singular whenever
$r_{\calA}:=\rank(\bmA)<M$; hence its ordinary condition number is not
meaningful until redundant latent directions are removed.  If
$r_{\calA}=0$, the affine family contains one parameter vector and there is no
latent condition number.  Otherwise, choose an orthonormal basis for
$\ker(\bmA)^\perp$ and place its vectors in the columns of
$\bmV_{\calA}\in\R^{M\times r_{\calA}}$.  Define the nonredundant coordinates
$\bmxi_{\calA}^{\mathrm{eff}}:=\bmV_{\calA}^\ts\bmxi$ and the effective
matrix
$\bmA_{\calA}^{\mathrm{eff}}:=\bmA\bmV_{\calA}$.  Because
$\bmxi-\bmV_{\calA}\bmV_{\calA}^\ts\bmxi\in\ker(\bmA)$, one has
$\bmA\bmxi=\bmA_{\calA}^{\mathrm{eff}}
\bmxi_{\calA}^{\mathrm{eff}}$.  Thus
$\bmtheta=\bmA_{\calA}^{\mathrm{eff}}
\bmxi_{\calA}^{\mathrm{eff}}+\bma$, and the Euclidean metric on
$\bmxi_{\calA}^{\mathrm{eff}}$ is inherited from the nonredundant directions of
the original latent vector.  For a local quadratic model with Hessian $\bmH$,
define
\[
 \bmH_{\calA}^{\mathrm{eff}}
 :=(\bmA_{\calA}^{\mathrm{eff}})^\ts
 \bmH\bmA_{\calA}^{\mathrm{eff}}.
\]
When this matrix is positive definite, the local condition number in these
effective coordinates is
\[
 \operatorname{cond}_2(\bmH_{\calA}^{\mathrm{eff}})
 :=\frac{\lambda_{\max}(\bmH_{\calA}^{\mathrm{eff}})}
         {\lambda_{\min}(\bmH_{\calA}^{\mathrm{eff}})}.
\]
Because a nonorthogonal
change of coordinates within the same affine image can alter this condition
number, the coordinate system must be specified.  Thus dimensional reduction
alone neither guarantees nor precludes better conditioning.  This conclusion
is consistent with empirical comparisons of random bases, general analyses of
reparameterization, and studies of low-rank optimization landscapes
\cite{gressmann2020random,kristiadi2023geometry,liu2025loralandscape}.
The reparameterization result in \cite{kristiadi2023geometry} concerns
invertible coordinate changes; by contrast, a rank-deficient $\calA$ also
restricts the model to a proper affine subspace.
Constructing an affine map that is simultaneously optimal for approximation
and well conditioned for a prescribed loss is a separate problem.

\subsection{Exact real arithmetic, discontinuity, and scope}
\label{sec:exact-real-scope}

We next turn to a separate limitation of the upper bound: its use of exact real
arithmetic.  The construction stores finite but increasingly long binary
streams in exact real coordinates.  Their useful bit length grows with the
depth and hence with $P$.  For a stream of $D$ decoded bits, the fixed
prefix gate used later contains a coefficient of size $2^{D+1}$, so the
dynamic range of some decoder weights shared across targets also grows
exponentially with the number of decoded bits.
The selection map $f\mapsto\bmxi_f$ uses quantization and is discontinuous.
These features are standard in the exact real regime with superconvergent
approximation rates
\cite{yarotsky18a,shijun:intrinsic:parameters,
shijun:optimal:rate:in:width:and:depth,shijun:RCNet,
shijun:three:layers,shijun:Characterized:by:Numer:Neurons,
ZZZZ-25-FMMNN,shijun:floor:relu,WANG2025107258}.
If every latent coordinate were restricted to $b$ bits, only $2^{Mb}$ codes
would be available, and a standard H\"older packing or entropy argument
would impose an additional information obstruction of order
$(Mb)^{-\alpha/d}$.  A recent bit-complexity framework likewise emphasizes
that parameter count and finite-precision information complexity are distinct
resource measures \cite{2026arXiv260801357M}.  The main theorem therefore
concerns expressivity over real parameters, not bit complexity, stability
under perturbations or noise, or the behavior of gradient descent.  Weight
constraints define another complementary resource model: upper and lower
approximation bounds are known for norm-constrained ReLU networks on smooth
function classes \cite{JIAO2023249}.  Here we count all weight and bias slots
but impose no uniform bound on their magnitudes.  These comparisons are not
used in any proof below.

A second scope restriction is architectural: we consider only fully connected
ReLU networks.  Both the architectural class and the activation are
substantive restrictions.  Approximation and
universality for convolutional architectures have their own corresponding
theories \cite{Bao2019ApproximationAO,2022arXiv221114047L,ZHOU2019}.  The role of the
activation is also reflected in general activation-dependent approximation
rates \cite{SIEGEL2020313}, activation-dependent spectral bias
\cite{hong2022activation,cao2019towards}, and constructions combining ReLU, sine, and
exponential activations on H\"older classes
\cite{doi:10.1137/21M144431X}.
Richer nested ReLU
architectures can obey different laws relating the parameter count to the
error \cite{shijun:net:arc:beyond:width:depth}.  With specially designed
continuous activations, even networks of fixed size can be universal
\cite{shijun:arbitrary:error:with:fixed:size}, while transfer results extend
many ReLU approximation constructions to broad activation families
\cite{shijun:2023:beyond:ReLU:to:diverse:actfun}.  None of these results gives
the joint affine minimax law studied here.

Finally, the constructive theorem uses width depending only on the dimension,
while its depth grows with $P$ and supplies sequential decoding time.  The lower
bound allows every fully connected architecture with at most $P$ dense
parameter slots under the convention of Section~\ref{sec:notation}.  No theorem
is claimed for upper bounds at fixed depth, convolutional architectures, or
nonlinear generators.  The upper construction applies to every continuous
function on the cube, while matching optimality is proved for the normalized
H\"older class in the uniform norm.  The affine map is also chosen
constructively; the result does not assert that a random affine subspace
attains the same worst-case rate.

In summary, the literature reviewed above motivates the general parameter map
$\calG$, while the remainder of the paper isolates a precise setting: one
affine latent parameterization shared across the target class, latent vectors
over exact real numbers, and two independently charged budgets $(M,P)$.  The
minimax class imposes no explicit width bound, whereas the optimal upper
construction has width depending only on $d$.

\section{Problem formulation and main results}\label{sec:model}

With the setting and scope now fixed, we formulate the precise minimax
problem.  Section~\ref{sec:notation} introduces the notation, architecture
classes, affine latent families, and target class; Section~\ref{sec:main}
states the matching upper and lower bounds.

\subsection{Notation and model}\label{sec:notation}

Let $\R$ and $\N^+$ denote the sets of real numbers and positive integers,
respectively.  For any set $X$, the notation $\R^X$ denotes the set of
functions from $X$ to $\R$.  If $m\in\N^+$ and $X\subseteq\R^m$ has the
subspace topology inherited from $\R^m$, then $C(X)$ denotes the space of
real-valued continuous functions on $X$.

Bold lowercase and uppercase letters denote column vectors and matrices,
respectively, and $\ts$ denotes transpose.  Thus $(a_1,\ldots,a_q)$ is a
column vector, whereas $[a_1,\ldots,a_q]$ is a row vector.  The symbol
$\bmzero$ denotes a zero vector or matrix, with its dimension inferred from
context.  For $q\in\N^+$, $\bmI_q$ is the $q\times q$ identity matrix, and
$\bm{e}_j$ is the $j$th coordinate vector whenever its ambient dimension is
clear.  We use $\ker(\bmA)$, $\operatorname{Range}(\bmA)$, and $\rank(\bmA)$
for the kernel, range, and rank of a matrix.  The expression
$\operatorname{diag}(\bmA_1,\ldots,\bmA_J)$ denotes the corresponding block
diagonal matrix, and $\operatorname{vec}$ denotes vectorization in the fixed
order specified below.

Throughout the paper, $\varrho(t):=\max\{t,0\}$ denotes the ReLU activation
and is applied componentwise to vectors.  For $t\in\R$, set
\[
 t_+:=\max\{t,0\}=\varrho(t),
 \qquad
 t_-:=\max\{-t,0\}=\varrho(-t),
 \qquad
 t=t_+-t_-.
\]
For $\bmz=(z_1,\ldots,z_m)\in\R^m$, the symbols $\lVert\bmz\rVert_2$ and
$\lVert\bmz\rVert_\infty:=\max_{1\le j\le m}\lvert z_j\rvert$ denote the
Euclidean and maximum norms, respectively.  For $1\le p<\infty$ and
$f:[0,1]^d\to\R$, write
\[
 \lVert f\rVert_{L^p([0,1]^d)}
 :=\biggl(\int_{[0,1]^d}\lvert f(\bmx)\rvert^p\mathrm{d}\bmx\biggr)^{1/p},
 \qquad
 \lVert f\rVert_{L^\infty([0,1]^d)}
 :=\sup_{\bmx\in[0,1]^d}\lvert f(\bmx)\rvert.
\]
The notation $\lvert\calS\rvert$ denotes the cardinality of a finite set
$\calS$.  We use $\one$ generically for indicators: $\one_{\{E\}}$ is the
indicator of a statement $E$, whereas $\one_{\calS}$ is the indicator
function of a set $\calS$.  The notation $\med(a,b,c)$ denotes the median of
$a,b,c\in\R$.  We write $\lfloor t\rfloor$ for the floor of $t$;
$\log$ and $\ln$ denote natural logarithms, $\log_2$ denotes the logarithm
to base two, and $e$ is Euler's number.  For positive quantities $a$ and $b$,
we write $a\asymp_\Lambda b$ if there are constants
$c_\Lambda,C_\Lambda>0$, depending only on the listed parameters $\Lambda$,
such that $c_\Lambda b\le a\le C_\Lambda b$.

For a binary function class $\calC$, $\VCdim(\calC)$ denotes its VC
dimension.  For a real-valued class $\calF$, $\Pdim(\calF)$ denotes its
pseudo-dimension, while $\VCdim(\calF)$ denotes the VC dimension of the class
obtained by thresholding $\calF$ at zero.  Formal definitions are given in
Section~\ref{sec:capacity-preliminaries}.

For $f\in C([0,1]^d)$, its modulus of continuity is
\begin{equation}\label{eq:modulus}
\omega_f(r)
 :=\sup\Bigl\{
 \lvert f(\bmx)-f(\bmy)\rvert:
 \bmx,\bmy\in[0,1]^d,\ \lVert\bmx-\bmy\rVert_2\le r
 \Bigr\},\qquad r\ge0.
\end{equation}
In particular, $\omega_f$ is nondecreasing, $\omega_f(0)=0$, and
$\omega_f(r)\to0$ as $r\downarrow0$.  It also satisfies
$\omega_f(Kr)\le K\omega_f(r)$ for every $r\ge0$ and $K\in\N^+$: subdivide
the line segment joining any two admissible points into $K$ equal pieces and
use the triangle inequality.

\subsubsection*{Network architectures}
A fully connected ReLU architecture is denoted by $\Phi$.  Its input
dimension $d_\Phi$ is determined by the target function, and its output
dimension is one throughout this paper; both dimensions are therefore
suppressed in the notation for architecture classes.  Thus, when these
classes are used for targets on $[0,1]^d$, the input dimension is understood
to be $d$.  Let
$L_\Phi\in\N^+$ be the number of hidden layers and let
$n_1,\ldots,n_{L_\Phi}\in\N^+$ be their widths.  Set
$n_0=d_\Phi$ and $n_{L_\Phi+1}=1$, and define
\[
 \depth(\Phi):=L_\Phi,
 \qquad
 \width(\Phi):=\max_{1\le\ell\le L_\Phi}n_\ell.
\]
Thus every architecture admitted by our convention has at least one hidden
ReLU unit.  An affine network with no hidden layer could be added as a
degenerate case.  Doing so would not weaken the lower bound, but excluding it
keeps the layer notation uniform throughout the constructive proof.  Indeed,
after rank reduction, its subgraph predicate is one polynomial inequality of
degree one in the effective coordinates, so Proposition~\ref{prop:GJ} gives
the same capacity estimate under the parameter budget.
For $N,L\in\N^+$, let
\[
 \Arch(N,L)
 :=
 \left\{\Phi:\width(\Phi)\le N,\ \depth(\Phi)\le L\right\}.
\]

For $\ell=1,\ldots,L_\Phi+1$, let
\[
 \calA_\ell(\bmz):=\bmW_\ell\bmz+\bmb_\ell
\]
denote the affine map associated with layer $\ell$.  Vectorizing and
concatenating all weight matrices and bias vectors in their natural order
yields
\[
 \bmtheta
 :=
 \operatorname{vec}\bigl(
 \bmW_1,\bmb_1,\ldots,
 \bmW_{L_\Phi+1},\bmb_{L_\Phi+1}
 \bigr)
 \in\R^{P_\Phi}.
\]
The function realized by $\Phi$ with parameters $\bmtheta$ is
\[
 \Phi_{\bmtheta}
 :=
 \calA_{L_\Phi+1}\circ\varrho\circ\calA_{L_\Phi}
 \circ\cdots\circ\varrho\circ\calA_1
 \in C(\R^{d_\Phi}).
\]
The hidden-unit index set and the number of hidden units are
\begin{equation}\label{eq:hidden-unit-count}
 \calU_\Phi
 :=
 \left\{(\ell,j):
 1\le\ell\le L_\Phi,\ 1\le j\le n_\ell\right\},
 \qquad
 U_\Phi:=\lvert\calU_\Phi\rvert
 =\sum_{\ell=1}^{L_\Phi}n_\ell.
\end{equation}
The total number of scalar weight and bias entries is
\begin{equation}\label{eq:parameter-count}
 P_\Phi
 :=
 \sum_{\ell=1}^{L_\Phi+1}n_\ell(n_{\ell-1}+1).
\end{equation}
After fixing an ordering of these entries, a parameter assignment is
identified with $\bmtheta\in\R^{P_\Phi}$.  All dense weight and bias entries
are counted, including those assigned a fixed value, possibly zero.  For
$P\in\N^+$, define
\begin{equation*}
 \Arch_{\mathrm{par}}(P)
 :=
 \left\{\Phi:P_\Phi\le P\right\}.
\end{equation*}
No separate width or depth restriction is imposed in
$\Arch_{\mathrm{par}}(P)$.
Figure~\ref{fig:fixed-width} summarizes these architectural conventions.

\begin{figure}[ht]
 \centering
 \includegraphics[width=0.8\textwidth]{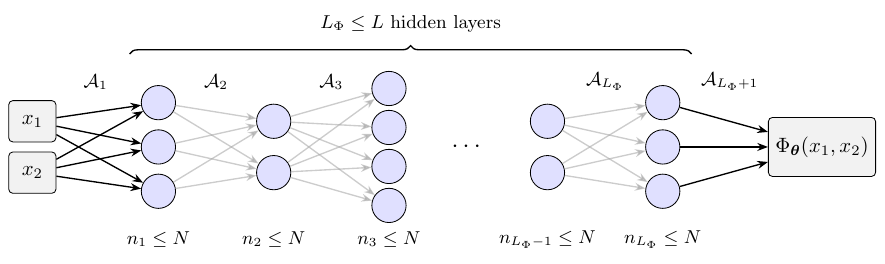}
 \caption{An architecture $\Phi\in\Arch(N,L)$.  The widths of hidden layers may be
 smaller than $N$, and all dense weight and bias entries are counted in
 $P_\Phi$.}
 \label{fig:fixed-width}
\end{figure}

\subsubsection*{Affine latent parameterizations}
For $M,Q\in\N^+$, let
\[
 \aff(M,Q)
 :=
 \left\{
 \calA:\R^M\to\R^Q:
 \calA(\bmxi)=\bmA\bmxi+\bma,\
 \bmA\in\R^{Q\times M},\
 \bma\in\R^Q
 \right\}.
\]
The convention $M\in\N^+$ excludes the degenerate zero-dimensional latent
space, whose affine image would consist of a single parameter vector.
For an architecture $\Phi$ and a map
$\calA\in\aff(M,P_\Phi)$, the network parameters are generated by
\[
 \bmtheta=\calA(\bmxi)=\bmA\bmxi+\bma,
 \qquad \bmxi\in\R^M.
\]
Here $\R^M$ is the latent parameter space, while $\bmA$ and $\bma$ are fixed
over the target class.  The corresponding realized family is
\begin{equation}\label{eq:realized-family}
 \calF_{\Phi,\calA}
 :=
 \left\{\Phi_{\calA(\bmxi)}:\bmxi\in\R^M\right\}
 =
 \left\{\Phi_{\bmtheta}:
       \bmtheta\in\calA(\R^M)\right\}
 \subseteq C(\R^{d_\Phi}).
\end{equation}
Its effective affine dimension is
\begin{equation}\label{eq:effective-rank}
 r_{\calA}:=\rank(\bmA)
 \le\min\{M,P_\Phi\}.
\end{equation}
The pair $(\Phi,\calA)$ is chosen independently of the target; only the
latent vector $\bmxi$ may depend on the target function.

\subsubsection*{Target class and minimax error}
For $0<\alpha\le1$, define the unit H\"older ball
\[
 \calH_d^\alpha
 :=
 \Bigl\{
 f\in C([0,1]^d):
 \lVert f\rVert_{L^\infty([0,1]^d)}\le1,\
 \lvert f(\bmx)-f(\bmy)\rvert
 \le\lVert\bmx-\bmy\rVert_2^\alpha
 \ \text{for all }\bmx,\bmy\in[0,1]^d
 \Bigr\}.
\]
We use the following dimension-dependent constants throughout:
\begin{align}
 N_d
 &:=3^d\bigl(\max\{20d,96\}+4\bigr),
 \label{eq:Nd}\\
 B_d
 &:=(48+2d)N_d(N_d+1)+N_d(d+2)+1,
 \label{eq:Bd}\\
 P_d^\star
 &:=2^{2d+2}B_d,
 \label{eq:Pdstar}\\
 C_d
 &:=4\cdot 8^{1/d}\sqrt{d}B_d^{2/d}.
 \label{eq:Cd}
\end{align}
All four constants depend only on $d$.  For $P\in\N^+$, we also use the depth
budget
\begin{equation*}
 L_{d,P}:=18\lfloor P/B_d\rfloor+30+2d.
\end{equation*}

With this notation in place, for an admissible pair $(\Phi,\calA)$ with
$d_\Phi=d$, define
\begin{equation}\label{eq:pair-risk}
 \calR_{\alpha,d}(\Phi,\calA)
 :=
 \sup_{f\in\calH_d^\alpha}
 \inf_{\bmxi\in\R^M}
 \left\lVert f-\Phi_{\calA(\bmxi)}
 \right\rVert_{L^\infty([0,1]^d)}.
\end{equation}
The affine latent minimax error under a parameter budget $P$ is
\begin{equation}\label{eq:minimax}
 \calE_{\alpha,d}(M,P)
 :=
 \inf_{\substack{\Phi\in\Arch_{\mathrm{par}}(P)\\ d_\Phi=d}}
 \inf_{\calA\in\aff(M,P_\Phi)}
 \calR_{\alpha,d}(\Phi,\calA).
\end{equation}
Thus the architecture and affine map are chosen before the target, whereas
the latent vector is chosen after the target is given.  In \eqref{eq:minimax},
the infimum is restricted to architectures with input dimension $d$ and
scalar output, as it is whenever an architecture class is used for targets
on $[0,1]^d$.  This convention keeps the input dimension determined by the
target domain out of the architecture notation while making every difference
$f-\Phi_{\bmtheta}$ well defined.  As usual,
$\inf\varnothing:=+\infty$.

\subsection{Main results}\label{sec:main}

With the minimax problem now defined, we state the three main results.
Theorem~\ref{thm:upper} gives a modulus-of-continuity bound for every continuous target,
Corollary~\ref{cor:holder-upper} specializes it to the H\"older class, and
Theorem~\ref{thm:lower} proves the matching lower bound.  Consequently,
\[
 \calE_{\alpha,d}(M,P)
 \asymp_{\alpha,d}
 \begin{cases}
  (PM)^{-\alpha/d},&M\le P,\\
  P^{-2\alpha/d},&M\ge P,
 \end{cases}
\]
whenever $P\ge P_d^\star$ and $\min\{M,P\}\ge4$.  The regime $M\le P$ is
particularly relevant to parameter-efficient approximation: $M$ may remain
fixed while the error still decreases algebraically as $P$ grows.

Recall the named constants $N_d,B_d,P_d^\star,C_d$ in
\eqref{eq:Nd}, \eqref{eq:Bd}, \eqref{eq:Pdstar}, and
\eqref{eq:Cd}, together with the depth budget $L_{d,P}$ defined in
Section~\ref{sec:notation}.  They are conservative by design so that every
dense parameter slot and every estimate involving a floor can be checked
without asymptotic notation.  We do not optimize their numerical values.

\begin{theorem}[Upper bound for continuous functions]
\label{thm:upper}
Let $d,M,P\in\N^+$ satisfy $\min\{M,P\}\ge4$ and
$P\ge P_d^\star$.  There exist an architecture
$\Phi\in\Arch_{\mathrm{par}}(P)\cap\Arch(N_d,L_{d,P})$ with input dimension
$d$ and an affine map $\calA\in\aff(M,P_\Phi)$, both depending only on
$(d,M,P)$ and not on the target, with the following property.  For every
$f\in C([0,1]^d)$, there is a target-dependent latent vector
$\bmxi_f\in\R^M$ such that, with
$\bmtheta_f:=\calA(\bmxi_f)$,
\begin{equation}
\label{eq:main-upper}
 \left\lVert f-\Phi_{\bmtheta_f}\right\rVert_{L^\infty([0,1]^d)}
 \le
 (d+2)
 \omega_f\bigl(
 C_d[P\min\{M,P\}]^{-1/d}
 \bigr).
\end{equation}
Equivalently, every entry of every layer matrix $\bmW_\ell$ and bias vector
$\bmb_\ell$ is an affine function of $\bmxi_f$.  Every dense weight and bias
entry is counted in $P_\Phi$, including entries assigned the value zero.
\end{theorem}

Theorem~\ref{thm:upper} is proved in Section~\ref{sec:upper-proof}.  That
section develops the fixed width modules, assembles the construction under
integer resources, and then selects these resources in terms of $(M,P)$.
Applying the theorem to H\"older functions gives the following consequence.

\begin{corollary}[Upper bound for H\"older functions]
\label{cor:holder-upper}
Let $d,M,P\in\N^+$ satisfy the assumptions of
Theorem~\ref{thm:upper}, and let $(\Phi,\calA)$ be the single pair supplied
by that theorem.  Let $0<\alpha\le1$ and $\lambda\ge0$.  If
$f\in C([0,1]^d)$ satisfies
$\lvert f(\bmx)-f(\bmy)\rvert
\le\lambda\lVert\bmx-\bmy\rVert_2^\alpha$
for all $\bmx,\bmy\in[0,1]^d$, then
\begin{equation*}
 \inf_{\bmtheta\in\calA(\R^M)}
 \left\lVert f-\Phi_{\bmtheta}\right\rVert_{L^\infty([0,1]^d)}
 \le
 \overline{C}_{\alpha,d}\lambda
 [P\min\{M,P\}]^{-\alpha/d},
\end{equation*}
where $\overline{C}_{\alpha,d}:=(d+2)C_d^\alpha$.  In particular,
\begin{equation}
\label{eq:minimax-upper}
 \calE_{\alpha,d}(M,P)
 \le
 \overline{C}_{\alpha,d}
 [P\min\{M,P\}]^{-\alpha/d}.
\end{equation}
The same approximants also satisfy this bound in $L^p([0,1]^d)$ for every
$1\le p<\infty$.
\end{corollary}

\begin{proof}
The H\"older condition gives $\omega_f(r)\le\lambda r^\alpha$.
Substitution into \eqref{eq:main-upper} yields
\[
 \left\lVert f-\Phi_{\bmtheta_f}\right\rVert_{L^\infty([0,1]^d)}
 \le
 \overline{C}_{\alpha,d}\lambda
 [P\min\{M,P\}]^{-\alpha/d}.
\]
Since $\bmtheta_f=\calA(\bmxi_f)$, it is admissible for the infimum over
$\calA(\R^M)$.  For $f\in\calH_d^\alpha$, the estimate applies with
$\lambda=1$.  Taking the supremum over $\calH_d^\alpha$ and then the two
outer infima in the definition of $\calE_{\alpha,d}(M,P)$ proves
\eqref{eq:minimax-upper}.  Finally, since $[0,1]^d$ has measure
one, $\lVert g\rVert_{L^p([0,1]^d)}
\le\lVert g\rVert_{L^\infty([0,1]^d)}$ for every bounded measurable $g$.
\end{proof}

The upper bound is formulated for every continuous target because the
construction naturally adapts to the modulus of continuity of the individual
function.  To determine whether its dependence on $(M,P)$ can be improved,
we consider the normalized H\"older class $\calH_d^\alpha$.  The following
lower bound applies uniformly over all admissible architectures and affine
parameterizations.

\begin{theorem}[Lower bound for H\"older functions]
\label{thm:lower}
Let $d\in\N^+$ and $0<\alpha\le1$.  There exists a constant
$\underline{c}_{\alpha,d}>0$ such that, for every $M,P\in\N^+$,
\begin{equation}
\label{eq:main-lower}
 \calE_{\alpha,d}(M,P)
 \ge
 \underline{c}_{\alpha,d}
 [P\min\{M,P\}]^{-\alpha/d}.
\end{equation}
The constant is independent of $M$, $P$, the network depth, the architecture,
and the affine map.
\end{theorem}

Theorem~\ref{thm:lower} is proved in Section~\ref{sec:lower-proof}.
The proof first controls the pseudo-dimension of a network family generated
by an affine map in terms of its affine rank and the number of ReLU units.
It then combines this capacity estimate with a packing of localized H\"older
functions.

Combining \eqref{eq:minimax-upper}
and \eqref{eq:main-lower} gives, whenever $P\ge P_d^\star$ and
$\min\{M,P\}\ge4$,
\[
 \calE_{\alpha,d}(M,P)
 \asymp_{\alpha,d}
 [P\min\{M,P\}]^{-\alpha/d}.
\]
Thus the upper and lower bounds match up to constants depending only on
$(\alpha,d)$.

\begin{remark}[A constant number of latent coordinates]
\label{rem:constant-M}
Fix an integer $M_0\ge4$ independently of $P$.  For
$P\ge\max\{M_0,P_d^\star\}$, the sharp estimate reduces to
\[
 \calE_{\alpha,d}(M_0,P)
 \asymp_{\alpha,d}
 (M_0P)^{-\alpha/d}
 =
 M_0^{-\alpha/d}P^{-\alpha/d}.
\]
Thus the number of target-dependent latent coordinates need not increase
with the network budget for the worst-case error to converge to zero.
A fixed collection of exact real coordinates, together with an increasingly
large decoder shared across the target class, already achieves the algebraic
rate $P^{-\alpha/d}$.
By contrast, when $M\ge P$, the rate saturates at $P^{-2\alpha/d}$.
The regime with fixed $M$ is therefore genuinely different from the regime in
which all $P$ parameter slots can vary independently with the target.  When
only the dependence on $P$ is displayed, the implied constants may depend
on the fixed value $M_0$.
\end{remark}

\section{Constructive proof of the upper bound}
\label{sec:upper-proof}

This section gives a constructive proof of Theorem~\ref{thm:upper}.  The proof
is assembled from elementary fixed-width modules.  Details are included
because two bookkeeping questions are central: which coefficients may depend
on the latent vector, and how signed quantities are transported through a
standard fully connected ReLU network without skip connections.  Every sparse
module below is embedded into a fully connected layer by assigning zero to
unused entries.  These entries are still included in the dense parameter count
$P_\Phi$.

References below to a ``vector-valued module'' are shorthand for several
channels retained within a single network with scalar output; they do not
introduce a separate architecture class.  Intermediate affine readouts are
either omitted, leaving the relevant coordinates in the hidden state, or
merged with an adjacent affine layer only when one side is fixed, as in
Lemma~\ref{lem:safe-composition}.  Every width and depth ledger refers to this
single final architecture, and every zero used to embed a sparse state
transition is counted as a dense parameter slot.

The seven subsections follow the data flow of the construction.
Section~\ref{sec:serialized-hinges} develops serialized hinge sums and safe
composition; Sections~\ref{sec:spline-loading} and
\ref{sec:binary-prefix-decoder} construct the affine loader and fixed decoder;
and Section~\ref{sec:point-fitting} combines them into a discrete fitting
module.  Section~\ref{sec:spatial-repair} then supplies spatial addressing,
bridge sequences, and boundary repair.  Section~\ref{sec:integer-construction}
assembles all modules under integer resource budgets, and
Section~\ref{sec:budget-selection} converts those resources into the prescribed
budgets $(M,P)$.

\subsection{Serialized hinge sums and safe composition}
\label{sec:serialized-hinges}

A continuous piecewise linear scalar function can be written as a sum of
hinges.  The next lemma serializes those hinges so that width is independent of
their number.

\begin{lemma}[Serialized hinges]\label{lem:hinge-serialization}
Let $J\in\N^+$, let
$\kappa_1<\cdots<\kappa_J$, and let
$a,b,c_1,\ldots,c_J\in\R$.  The function
\begin{equation*}
 \varphi(x):=a+bx+\sum_{j=1}^{J}c_j\varrho(x-\kappa_j)
\end{equation*}
is realized by an architecture $\Phi\in\Arch(5,2J+2)$.  More generally, if
$a,b,c_1,\ldots,c_J$ are affine functions of a common vector
$\bmxi\in\R^{M_{\mathrm{code}}}$ for some $M_{\mathrm{code}}\in\N^+$, then
$\Phi$ can be chosen independently of
$\bmxi$, and all slot assignments jointly define one affine map
$\calA\in\aff(M_{\mathrm{code}},P_\Phi)$.  A target-dependent slot uses only
the coordinates that occur in one of the displayed coefficients.
\end{lemma}

\begin{proof}
We carry both the input and a running signed sum by their positive and negative
parts.  For each hinge, one layer computes the hinge and a second layer adds
its weighted contribution to the running sum.  This organization keeps the
width fixed while the depth grows linearly with the number of hinges.

The elementary identity
\[
 \varrho(v)-\varrho(-v)=v,
 \qquad v\in\R,
\]
will be used repeatedly.  Write
\[
 x_+:=\varrho(x),\qquad x_-:=\varrho(-x),
 \qquad x=x_+-x_-.
\]
The first hidden layer produces $x_+,x_-$ and
$a_+:=\varrho(a),a_-:=\varrho(-a)$.  Here $a_+$ is obtained using
zero input weight and bias $a$, while $a_-$ uses zero input weight and bias
$-a$.  The next hidden layer forms
\[
 H_0^+
 :=\varrho\bigl(a_+-a_-+b(x_+-x_-)\bigr),\qquad
 H_0^-
 :=\varrho\bigl(-a_++a_--b(x_+-x_-)\bigr).
\]
Hence $H_0^+-H_0^-=a+bx$.  Suppose that after the $(j-1)$st update the
state contains
\[
 x_+,\ x_-,\ H_{j-1}^+,\ H_{j-1}^-,
 \qquad
 H_{j-1}^+-H_{j-1}^-
 =a+bx+\sum_{i=1}^{j-1}c_i\varrho(x-\kappa_i).
\]
One additional hidden layer computes
\[
 h_j:=\varrho(x_+-x_--\kappa_j)
      =\varrho(x-\kappa_j)
\]
while copying the four nonnegative state coordinates.  The next layer computes
\begin{align*}
 H_j^+
 &:=
 \varrho\bigl(H_{j-1}^+-H_{j-1}^-+c_jh_j\bigr),\\
 H_j^-
 &:=
 \varrho\bigl(-H_{j-1}^++H_{j-1}^--c_jh_j\bigr),
\end{align*}
again copying $x_+$ and $x_-$.  Therefore
\[
 H_j^+-H_j^-
 =H_{j-1}^+-H_{j-1}^-+c_jh_j.
\]
This proves the induction.  The affine output
$H_J^+-H_J^-$ equals
$\varphi(x)$.  At most five channels are present:
$x_+,x_-,H_{j-1}^+,H_{j-1}^-,h_j$.  Every copied coordinate is
nonnegative, so an identity weight and zero bias pass it unchanged through
ReLU.  There are two initialization
layers and two layers for each of the $J$ hinges, so the hidden depth
is at most $2J+2$.

Only the biases $\pm a$ and the weights $\pm b,\pm c_j$ vary with the
coefficients.  They enter individual network slots without being multiplied
by another varying parameter slot.  A varying weight such as $c_j$ is allowed
to multiply the activation $h_j$ during the forward pass; the requirement is
only that the value assigned to each parameter slot be affine in the latent
vector.  Consequently, affine dependence of the coefficients on that vector
implies affine dependence of every parameter slot.  Listing the fixed and
varying assignments in the global slot order gives the single affine map
asserted in the statement.
\end{proof}

The preceding lemma applies to every continuous piecewise linear scalar
function with finitely many breakpoints, because its slope changes are exactly
the hinge coefficients.  We record this elementary fact to ensure that the
later staircase constructions do not rely on an unstated representation
theorem.

\begin{lemma}[Slope-jump representation]\label{lem:slope-jump}
Let $\varphi:\R\to\R$ be continuous and piecewise linear with breakpoints
$\kappa_1<\cdots<\kappa_J$.  Suppose its slope is $m_0$ to the left of
$\kappa_1$, is $m_j$ between $\kappa_j$ and $\kappa_{j+1}$ for
$j=1,\ldots,J-1$, and is $m_J$ to the right of $\kappa_J$.  Then
\begin{equation}\label{eq:slope-jump}
 \varphi(x)
 =\varphi(\kappa_1)+m_0(x-\kappa_1)
 +\sum_{j=1}^{J}(m_j-m_{j-1})\varrho(x-\kappa_j).
\end{equation}
Consequently, $\varphi$ is realized by an architecture in
$\Arch(5,2J+2)$.  If $\varphi(\kappa_1)$ and all its slopes are affine
functions of a latent vector, the realization can be chosen so that every
parameter slot is affine in that vector.
\end{lemma}

\begin{proof}
Call the right-hand side of \eqref{eq:slope-jump} $\psi(x)$.  When
$x<\kappa_1$, every hinge vanishes, so $\psi$ has slope $m_0$.  When
$\kappa_k<x<\kappa_{k+1}$, exactly the first $k$ hinges are active, and the slope
telescopes:
\[
 m_0+\sum_{j=1}^{k}(m_j-m_{j-1})=m_k.
\]
The same calculation gives slope $m_J$ for $x>\kappa_J$.  Moreover,
$\psi(\kappa_1)=\varphi(\kappa_1)$.  Starting with the interval containing
$\kappa_1$ and moving across the breakpoints, continuity and equality of the
slopes show successively that $\psi=\varphi$ on every linearity interval.
This proves the formula.  Lemma~\ref{lem:hinge-serialization} gives the
network realization.  Its coefficients are fixed linear combinations of
$\varphi(\kappa_1),m_0,\ldots,m_J$, which proves the final affine-dependence
statement.
\end{proof}

The preceding lemmas provide the scalar building blocks.  To assemble them
into a single network while preserving affine dependence on the code, we next
record the required transport and composition rules.

\begin{lemma}[Signed transport and one-side-fixed composition]
\label{lem:safe-composition}
The following operations preserve affine dependence of network parameters on a
common code vector.
\begin{enumerate}[label={\textup{(\roman*)}}]
 \item A signed scalar $u$ may be transported through any number of hidden
 layers as the nonnegative pair
 $(u_+,u_-)=(\varrho(u),\varrho(-u))$, with the fixed identity update
 $(u_+,u_-)\mapsto(u_+,u_-)$ and recovery $u=u_+-u_-$.
 \item Finitely many networks may be copied or run in parallel.  Their widths
 add.  A $q$-dimensional state is transported coordinatewise by $2q$
 nonnegative channels, so shorter branches can be extended to a common
 depth by fixed signed transport.  If a branch already terminates in a
 nonnegative state or a signed pair, its existing depth gap is filled by fixed
 identity layers.  If only a signed scalar affine readout
 $u=\bmw(\bmxi)^\ts\bmh+b(\bmxi)$ is available, compose that readout with the
 fixed split map $u\mapsto(u,-u)$, whose matrix is $[1,-1]^\ts$, and merge the
 two adjacent affine maps.  The resulting hidden layer produces
 $(\varrho(u),\varrho(-u))$; its two affine
 rows are sign copies of $(\bmw(\bmxi),b(\bmxi))$, so affine code dependence
 is preserved.  This split layer must be included in the interface depth
 ledger.
 \item Two adjacent affine layers with no intervening activation may be
 merged without losing affine code dependence whenever at least one of the
 two layers is fixed.
 \item Repeating an affine slot assignment in copied branches and padding
 unused slots by fixed zeros preserve affine code dependence.  Zero padding
 changes the dense parameter count but not the realized function.
\end{enumerate}
By contrast, merging two code-dependent affine layers need not preserve affine
dependence.
\end{lemma}

\begin{proof}
The proof is a coefficientwise calculation.  The transport and routing
coefficients in parts (i) and (ii) are fixed at $0$, $1$, or $-1$.  Splitting
a code-dependent readout merely copies its coefficient row with both signs,
so affine dependence is preserved.  For (iii), we expand the merged
coefficients and identify exactly where a product of two code-dependent
entries could occur.

For (i), define
\[
 \bm{u}^\pm:=\bigl(\varrho(u),\varrho(-u)\bigr).
\]
Every coordinate of $\bm{u}^\pm$ is nonnegative, and hence each additional
hidden layer may copy it by
\[
 \bm{u}^\pm\longmapsto
 \varrho(\bmI_2\bm{u}^\pm)=\bm{u}^\pm.
\]
The fixed output row $[1,-1]$ gives
$[1,-1]\bm{u}^\pm=u$.
For a vector $\bmu=(u_1,\ldots,u_q)$, apply this construction separately to
each coordinate.  The transported state is
\[
 (u_{1,+},u_{1,-},\ldots,u_{q,+},u_{q,-})\in\R^{2q},
\]
and a fixed block row recovers $\bmu$.  This is the vector-valued padding
used in part~(ii).

For (ii), suppose $J$ modules at a common layer have matrices
$\bmW^{(1)},\ldots,\bmW^{(J)}$ and biases
$\bmb^{(1)},\ldots,\bmb^{(J)}$.  If they receive different state vectors,
their parallel layer is
\[
 \bmW^{\parallel}
 :=\operatorname{diag}\bigl(\bmW^{(1)},\ldots,\bmW^{(J)}\bigr),
 \qquad
 \bmb^{\parallel}
 :=\bigl(\bmb^{(1)},\ldots,\bmb^{(J)}\bigr).
\]
Every off-diagonal routing block is fixed at zero; the diagonal blocks may be
code dependent.  If the matrices and biases of the original
modules are affine in $\bmxi$, then so are all entries of
$(\bmW^{\parallel},\bmb^{\parallel})$.  A shorter module whose retained state
is already nonnegative is padded by the fixed signed-transport update from
(i), so unequal depths introduce no new code-dependent coefficient.  If its
terminal value is available only as a signed affine readout, write
$u=\bmw(\bmxi)^\ts\bmh+b(\bmxi)$.  Merging this readout with the fixed split
gives the two affine rows $(\bmw(\bmxi),b(\bmxi))$ and
$(-\bmw(\bmxi),-b(\bmxi))$.  After ReLU they produce
$(\varrho(u),\varrho(-u))$; subsequent layers copy this pair by fixed
identities.  The split costs one hidden layer and may duplicate
 code-dependent slots, but every duplicated slot remains affine and no product
 of varying coefficients is formed.  The final recovery row is fixed.  If all
branches receive the same input, their
first-layer matrices are stacked vertically rather than placed diagonally.
In both cases, the widths add.  The padded depth is the common target depth
when the signed pair is already present; otherwise the one split layer is
counted explicitly at the relevant interface.

For (iii), let adjacent affine maps be
$\bmu\mapsto\bmW_1(\bmxi)\bmu+\bmb_1(\bmxi)$ and
$\bmv\mapsto\bmW_2(\bmxi)\bmv+\bmb_2(\bmxi)$.  With no activation between
them, their composition is
\[
 \bmW_2(\bmxi)\bmW_1(\bmxi)\bmu
 +\bmW_2(\bmxi)\bmb_1(\bmxi)+\bmb_2(\bmxi).
\]
If $(\bmW_1,\bmb_1)=(\bmW_1^0,\bmb_1^0)$ is fixed and
\[
 \bmW_2(\bmxi)=\bmW_2^0+\sum_{m=1}^M\xi_m\bmW_2^m,
 \qquad
 \bmb_2(\bmxi)=\bmb_2^0+\sum_{m=1}^M\xi_m\bmb_2^m,
\]
then the merged matrix and bias are
\[
 \bmW_2^0\bmW_1^0+
 \sum_{m=1}^M\xi_m\bmW_2^m\bmW_1^0,
 \qquad
 \bmW_2^0\bmb_1^0+\bmb_2^0+
 \sum_{m=1}^M\xi_m(\bmW_2^m\bmb_1^0+\bmb_2^m),
\]
which are affine in $\bmxi$.  If instead
$(\bmW_2,\bmb_2)=(\bmW_2^0,\bmb_2^0)$ is fixed, the merged matrix and bias
are
\[
 \bmW_2^0\bmW_1(\bmxi),
 \qquad
 \bmW_2^0\bmb_1(\bmxi)+\bmb_2^0,
\]
which are again affine in $\bmxi$.  If
both layers vary, the product can contain $\xi_m\xi_{m'}$; for example, the
scalar composition of the weights $w_1(\xi)=w_2(\xi)=\xi$ has weight
$\xi^2$.

For (iv), copying a module repeats the same affine coordinate function
in additional entries of the complete parameter vector.  Assigning zero to a
new slot gives a constant, hence affine, coordinate function.  Neither
operation changes an already constructed branch value.  Every interface used
below has a fixed side:
fixed routing precedes an affine loader, and a fixed decoder or median module
follows an affine loader.
\end{proof}

The preceding operations will often be described at the level of retained
hidden states rather than by writing a new dense matrix at every interface.
The following closure statement makes that shorthand precise and will be used
throughout the remainder of the upper construction.

\begin{corollary}[Assembly through retained states and fixed interfaces]
\label{cor:state-interface}
Consider finitely many ReLU modules whose parameter slots are affine functions
of one common code vector.  Form a computation by repeatedly applying the
following operations: run modules in parallel; retain a terminal hidden state
instead of applying its scalar affine readout; route or split retained states
through fixed affine/ReLU layers; merge consecutive affine maps when at least
one is fixed; and align depths by fixed signed transport.  Then the resulting
computation is realized by one fully connected scalar-output ReLU architecture,
and every slot of that architecture is affine in the same code vector.

If a hidden layer of the resulting architecture is enlarged by adding channels
that are identically zero, the enlarged architecture realizes the same
function.  More precisely, if the original pair is $(\Psi,\calA)$ with
$\calA\in\aff(M,P_\Psi)$ and the enlarged architecture is
$\Psi^{\mathrm{pad}}$, inserting fixed zeros into the new slots defines
\[
 \calA^{\mathrm{pad}}
 \in\aff(M,P_{\Psi^{\mathrm{pad}}})
\]
such that
$\Psi^{\mathrm{pad}}_{\calA^{\mathrm{pad}}(\bmxi)}
=\Psi_{\calA(\bmxi)}$ for every $\bmxi\in\R^M$.
\end{corollary}

\begin{proof}
Parallelization stacks first-layer matrices when the branches share an input
and uses block-diagonal matrices at later layers.  Retaining a terminal state
simply omits a scalar affine row.  Every subsequent substitution of an omitted
readout is a composition of two adjacent affine maps, with at least one side
fixed by assumption, and is therefore affine in the code by
Lemma~\ref{lem:safe-composition}(iii).  Fixed routing, splitting, and signed
transport use only constant coefficients, while copied branches repeat
existing affine coordinate functions.  Induction over the finite sequence of
operations produces a single dense architecture after all unused connections
are assigned zero.  This proves the first assertion.

For padding, keep every old slot assignment in its corresponding position and
assign zero to each new incoming weight, outgoing weight, and bias associated
with an added channel.  This is a fixed affine injection from
$\R^{P_\Psi}$ into $\R^{P_{\Psi^{\mathrm{pad}}}}$.  Its composition with
$\calA$ is the displayed map $\calA^{\mathrm{pad}}$, and all added channels
remain zero after ReLU.  Hence the realized function is unchanged and the
codomain agrees with the enlarged architecture.
\end{proof}

This closure rule separates architectural bookkeeping from the numerical
purpose of each module.  From now on, saying that a module ``returns'' several
quantities means that they occur in its terminal hidden state as nonnegative
channels or signed pairs.  Intermediate scalar readouts are attached only
conceptually and are eliminated at the next fixed-side interface.  We may
therefore construct the loader, decoder, and routing modules separately and
then combine them without leaving the class of ordinary scalar-output fully
connected networks.

\subsection{Affine spline loading with serialized hinges}
\label{sec:spline-loading}

We first use the serialized hinge construction to turn $S$ latent coordinates
into interpolation data for a one-dimensional spline while keeping the width
independent of $S$.

\begin{lemma}[Affine spline loader]\label{lem:spline-loader}
Let $S\in\N^+$ and
$\bmy=(y_0,\ldots,y_{S-1})\in\R^S$.  There exist an architecture
$\Phi^{\mathrm{spl}}\in\Arch(5,2S+2)$ and an affine map
$\calA^{\mathrm{spl}}\in\aff(S,P_{\Phi^{\mathrm{spl}}})$, both independent of
$\bmy$, such that, after setting
\[
 \bmtheta:=\calA^{\mathrm{spl}}(\bmy),
\]
one has
\[
 \Phi^{\mathrm{spl}}_{\bmtheta}(m)=y_m,
 \qquad m=0,\ldots,S-1.
\]
Every target-dependent slot depends on at most three consecutive entries of
$\bmy$.
\end{lemma}

\begin{proof}
We express the linear interpolant through the points $(m,y_m)$ in the hinge
basis $1,x,\varrho(x-m)$.  Its coefficients are first and second differences
of the data and are therefore linear functions of $\bmy$.
Lemma~\ref{lem:hinge-serialization} then adds the hinges one at a time while
carrying $x$ and a signed running sum through a constant number of channels.

For $S=1$, let the first hidden layer use zero input weights and the two
code-dependent biases $y_0,-y_0$ to produce
$(\varrho(y_0),\varrho(-y_0))$.  Three fixed identity layers transport this
pair, and the fixed scalar readout $[1,-1]$ returns $y_0$.  If that readout is
omitted, the terminal hidden state still contains the same signed pair,
exactly as required by the later loader interface.  Fixed zero channels give
width at most $5$ and hidden depth $4=2S+2$.  The only varying slots are
linear in $\bmy$.

We now assume $S\ge2$.  The continuous piecewise linear interpolant, extended
constantly outside $[0,S-1]$, is
\begin{equation}\label{eq:hinge-loader}
 \begin{split}
 \varphi_{\bmy}(x)
 ={}&y_0+(y_1-y_0)\varrho(x)\\
 &+\sum_{m=1}^{S-2}(y_{m+1}-2y_m+y_{m-1})\varrho(x-m)\\
 &-(y_{S-1}-y_{S-2})\varrho(x-(S-1)).
 \end{split}
\end{equation}
To verify the interpolation, first note that
$\varphi_{\bmy}(0)=y_0$.  On the open interval $(j,j+1)$,
$j=0,\ldots,S-2$, the slope is
\[
 (y_1-y_0)
 +\sum_{m=1}^{j}(y_{m+1}-2y_m+y_{m-1})
 =y_{j+1}-y_j;
\]
the sum telescopes.  Hence
$\varphi_{\bmy}(j+1)-\varphi_{\bmy}(j)=y_{j+1}-y_j$, and induction gives
$\varphi_{\bmy}(j)=y_j$ for every required integer.  The final hinge at $S-1$
cancels the last slope, so the right extension is constant.  The formula is
also valid when $S=2$, in which case the middle sum is empty.

The representation \eqref{eq:hinge-loader} contains exactly $S$ hinges, at
$0,1,\ldots,S-1$.  Lemma~\ref{lem:hinge-serialization} therefore gives
width at most $5$ and hidden depth at most $2S+2$.  Its initial value and every
hinge coefficient are linear combinations of at most three consecutive
entries of $\bmy$.  The vector of all such coefficients is therefore a fixed
linear transformation of $\bmy$.  Placing those coefficients into the
corresponding network
slots defines the fixed affine map $\calA^{\mathrm{spl}}$ and proves the claimed
support bound.
\end{proof}

We will later use three copies of this loader in parallel.  Their combined
width is at most $15$, their depth is unchanged, and each target-dependent
slot still involves at most three consecutive coordinates from its data block.

\subsection{Fixed binary prefix decoder}
\label{sec:binary-prefix-decoder}

The complementary decoding module is entirely fixed.  It reads a prefix of a
finite binary string encoded in one real input.  The integer $D$ determines
the sequential reading time, while only the binary fraction supplied as input
changes with the encoded string.

For a bit vector
$\bmnu=(\nu_1,\ldots,\nu_D)\in\{0,1\}^D$, use the finite
binary-fraction notation
\[
 \operatorname{bin}0.\nu_1\cdots\nu_D
 :=\sum_{i=1}^D\nu_i2^{-i}.
\]
For fixed $D$, this encoding is injective on $\{0,1\}^D$; the familiar
ambiguity of infinite binary expansions does not arise.

\begin{lemma}[Binary prefix extraction]\label{lem:binary-prefix}
For every $D\in\N^+$, there exist a two-input architecture
$\Psi_D\in\Arch(20,4D+2)$ and a fixed parameter vector
$\bmtheta_D\in\R^{P_{\Psi_D}}$.  Denote the resulting function by
\[
 E_D:=(\Psi_D)_{\bmtheta_D}:\R^2\to\R.
\]
Then, for every $\bmnu\in\{0,1\}^D$,
\[
 E_D\bigl(\operatorname{bin}0.\nu_1\cdots\nu_D,k\bigr)
 =\sum_{i=1}^k\nu_i,
 \qquad k=0,1,\ldots,D,
\]
where the sum is understood to be zero when $k=0$.
The vector $\bmtheta_D$ depends on $D$, but none of its entries depends on
$\bmnu$.
\end{lemma}

\begin{proof}
At stage $i$, a fixed piecewise linear gate extracts the leading bit of the
current binary remainder.  We double the remainder, subtract the extracted
bit, and add that bit to the accumulator exactly when $i\le k$.  Repeating
this fixed-width block $D$ times produces the required prefix sum.  The
details below also verify that no decoder parameter depends on the bit string.

The identities proved below are required only on the following finite
collection of legal inputs:
\[
 \left\{(\operatorname{bin}0.\nu_1\cdots\nu_D,k):
 \bmnu\in\{0,1\}^D,\ k\in\{0,\ldots,D\}\right\}.
\]
The constructed ReLU network is defined on all of $\R^2$, but the lemma makes
no assertion away from these inputs.  This distinction is useful because some
intermediate affine expressions are known to be nonnegative only at the legal
inputs.

Set $\delta_D:=2^{-(D+1)}$ and define the fixed gate
\begin{equation}\label{eq:bit-gate}
 g_D(t):=
 \frac{\varrho(t-1+\delta_D)-\varrho(t-1)}{\delta_D}.
\end{equation}
The three ranges can be checked separately.  If $t\le1-\delta_D$, both
ReLU terms vanish.  If $1-\delta_D\le t\le1$, only the first is active and
\[
 g_D(t)=\frac{t-1+\delta_D}{\delta_D}.
\]
If $t\ge1$, their difference is
$(t-1+\delta_D)-(t-1)=\delta_D$.  Hence
\[
 g_D(t)=0\quad(t\le1-\delta_D),\qquad
 g_D(t)=1\quad(t\ge1),
\]
with the displayed linear interpolation between the two plateaus.

Let $r_0=\operatorname{bin}0.\nu_1\cdots\nu_D$.  Recursively define
\begin{equation}\label{eq:bit-recurrence}
 e_i=g_D(2r_{i-1}),\qquad r_i=2r_{i-1}-e_i,\qquad i=1,\ldots,D.
\end{equation}
We first prove the complete remainder invariant
\begin{equation}\label{eq:bit-remainder-invariant}
 r_i=\sum_{j=i+1}^D\nu_j2^{-(j-i)},
 \qquad i=0,\ldots,D,
\end{equation}
where the sum is empty and therefore equal to zero when $i=D$.  For $i=0$,
\eqref{eq:bit-remainder-invariant} is exactly the definition of
$r_0$.  Assume it holds at index $i-1$.  Then
\[
 r_{i-1}
 =\sum_{j=i}^D\nu_j2^{-(j-i+1)}.
\]
If $\nu_i=0$, then
\[
 2r_{i-1}
 =\sum_{j=i+1}^D\nu_j2^{-(j-i)}
 \le1-2^{-(D-i)}
 \le1-2^{-(D-1)}
 =1-4\delta_D
 <1-\delta_D,
\]
where the empty sum for $i=D$ is zero.  The second inequality uses
$D-i\le D-1$, and the equality uses
$\delta_D=2^{-(D+1)}$.  Hence $e_i=0=\nu_i$.
If $\nu_i=1$, then $2r_{i-1}\ge1$, so $e_i=1=\nu_i$.
Substitution in \eqref{eq:bit-recurrence} gives
\[
 r_i
 =\sum_{j=i+1}^D\nu_j2^{-(j-i)}.
\]
This is the invariant at index $i$.  Induction therefore proves
both
\begin{equation}\label{eq:bit-extraction-invariant}
 e_i=\nu_i,
 \qquad
 r_i=\sum_{j=i+1}^D\nu_j2^{-(j-i)},
 \qquad i=1,\ldots,D.
\end{equation}
In particular, every $r_i$ lies in $[0,1)$ at a legal input.

For the integer input $k$, the quantity
\[
 \tau_i(k):=\varrho(k-i+1)-\varrho(k-i)
\]
 equals $1$ if $k\ge i$ and $0$ if $k<i$.  Since
 $e_i,\tau_i(k)\in\{0,1\}$ at the required inputs,
\[
 e_i\tau_i(k)=\varrho(e_i+\tau_i(k)-1).
\]
Thus a running accumulator adds $e_i$ exactly when $i\le k$.  More
formally, define
\[
 u_i:=\varrho(e_i+\tau_i(k)-1),
 \qquad
 A_0:=0,
 \qquad
 A_i:=A_{i-1}+u_i.
\]
The second induction invariant is
\begin{equation}\label{eq:bit-accumulator-invariant}
 A_i
 =\sum_{j=1}^i\nu_j\one_{\{j\le k\}}
 =\sum_{j=1}^{\min\{i,k\}}\nu_j,
 \qquad i=0,\ldots,D.
\end{equation}
Here the last sum is understood to be zero when $\min\{i,k\}=0$.
It holds for $i=0$ because both sides vanish.  If it holds for $i-1$, then
\eqref{eq:bit-extraction-invariant} and the definition of $\tau_i$ give
\[
 u_i=e_i\tau_i(k)=\nu_i\one_{\{i\le k\}}.
\]
Adding this term to $A_{i-1}$ proves
\eqref{eq:bit-accumulator-invariant} at index $i$.  At the end of the
$D$ stages,
\[
 A_D=\sum_{j=1}^k\nu_j,
\]
which is the desired output.

It remains to realize these recurrences by a ReLU network and verify the
resource bounds.  One four-layer block can be organized as follows.  Its
input is the nonnegative state $(r_{i-1},k,A_{i-1})$, whose values are
described by \eqref{eq:bit-remainder-invariant} and
\eqref{eq:bit-accumulator-invariant}:
\begin{enumerate}[label={\textup{(\arabic*)}}]
 \item compute the two hinges in \eqref{eq:bit-gate} and the two
 hinges defining $\tau_i(k)$, while copying $r_{i-1},k,A_{i-1}$;
 \item form the nonnegative values $e_i$, $\tau_i(k)$, and
 $r_i=2r_{i-1}-e_i$, while copying $k,A_{i-1}$;
 \item form $u_i=\varrho(e_i+\tau_i(k)-1)$ and copy
 $r_i,k,A_{i-1}$;
 \item set $A_i=A_{i-1}+u_i$ and copy $r_i,k$.
\end{enumerate}
The four operations above yield the following state and channel ledger.  The
middle column lists the nonnegative coordinates retained after each layer.
\begingroup
\renewcommand{\arraystretch}{1.214}
\setlength{\arraycolsep}{7pt}
\[
\begin{array}{@{}c l c@{}}
\toprule
\text{layer of block }i
&\text{state retained after the layer}
&\text{channels}\\
\midrule
1&
\begin{gathered}
 \text{copied state: }(r_{i-1},k,A_{i-1}),\\
 \text{hinges: }\varrho(2r_{i-1}-1+\delta_D),\
 \varrho(2r_{i-1}-1),\\
 \varrho(k-i+1),\ \varrho(k-i)
\end{gathered}
&7\\[3pt]
2&(r_i,k,A_{i-1},e_i,\tau_i(k))&5\\
3&(r_i,k,A_{i-1},u_i)&4\\
4&(r_i,k,A_i)&3\\
\bottomrule
\end{array}
\]
\endgroup
The only point requiring care is layer~2.  It computes $e_i$, $\tau_i(k)$,
and $r_i$ from the hinge channels produced by layer~1, together with the
copied value $r_{i-1}$.  In particular, the computation of $r_i$ does not use
the layer~2 postactivation $e_i$.  All coordinates listed in the state column
are nonnegative at the required inputs, so copied coordinates pass unchanged
through ReLU.  One initial hidden layer creates
$(r_0,k,A_0)$ with $A_0=0$; after the $D$ blocks, the scalar output layer
reads $A_D$.  Thus the hidden depth is at most
\[
 1+4D\le4D+2,
\]
and the maximum displayed width is $7\le20$.  Assigning zero to unused entries
of the actual dense layers gives a fully connected realization of width at
most $20$.  All constants depend
on $D$ through $\delta_D$ but not on the particular bit string, so the
parameter vector is fixed.
\end{proof}

This construction is the one-dimensional, single-stream specialization of
the bit-extraction schemes underlying
\cite[Theorem~2]{Bartlett98almostlinear} and
\cite[Lemma~4.5]{shijun:optimal:rate:in:width:and:depth}.  Since the full
recurrence is verified above, neither cited result is invoked as a black box.

\subsection{Point fitting with an affine code}
\label{sec:point-fitting}

The two preceding modules play complementary roles.  The spline loader uses
$S$ latent entries to select data for each block, and its parameter slots
depend affinely on those entries.  The fixed prefix decoder uses depth
proportional to $D$ to read the data stored in the selected block.  We now
combine the two modules into a one-dimensional construction that fits up to
$SD$ consecutive samples.  Here and below, ``affine'' refers to the parameter
map $\bmxi\mapsto\bmtheta$; the target-dependent selection of $\bmxi$ from the
samples need not be affine.

\begin{lemma}[Point fitting with an affine code]\label{lem:point-fitting}
Let $1\le S\le D$, $1\le J\le SD$, and $\eta\ge0$.  Suppose
$y_0,\ldots,y_{J-1}\in\R$ satisfy
\begin{equation*}
 \lvert y_j-y_{j-1}\rvert\le\eta,\qquad j=1,\ldots,J-1.
\end{equation*}
Write $\bmy:=(y_0,\ldots,y_{J-1})$.  There exist a one-input architecture
$\Phi^{\mathrm{fit}}\in\Arch(96,10D+24)$ and an affine map
$\calA^{\mathrm{fit}}\in\aff(3S+1,P_{\Phi^{\mathrm{fit}}})$, both depending
only on $(S,D)$, such that one can choose $\bmxi\in\R^{3S+1}$ and set
\[
 \bmtheta:=\calA^{\mathrm{fit}}(\bmxi)
\]
so that
\begin{equation}\label{eq:point-fitting-error}
 \left\lvert
 \Phi^{\mathrm{fit}}_{\bmtheta}(j)-y_j
 \right\rvert
 \le\eta,\qquad j=0,\ldots,J-1.
\end{equation}
Neither the architecture nor the affine map depends on the samples or on
$\eta$.
\end{lemma}

\begin{proof}
The construction has four stages.  First, we pad the sequence and write each
index as $j=mD+k$.  Second, within each block, quantization at scale $\eta$
turns every successive change into $-1$, $0$, or $1$; we store the block
baseline directly and encode the positive and negative changes in two binary
fractions.  Third, three affine spline loaders select these blockwise
quantities, and two fixed prefix decoders reconstruct the cumulative changes.
Finally, the scale $\eta$ enters through the generated output-layer weights
$\pm\eta$, so no network parameter is a product of two target-dependent
quantities.  Figure~\ref{fig:point-fitting} summarizes this data flow.

For $\eta>0$, the central reconstruction identity, with notation introduced
below, is $a_m+\eta(U_{m,k}^+-U_{m,k}^-)=\eta z_{m,k}$.  The loaders select the
baseline and the two bit streams, while the fixed decoders supply only their
prefix counts.  Multiplication by $\eta$ occurs only during the forward pass
through output weights whose slot values are $\pm\eta$; no parameter slot
contains a product of two latent-dependent coordinates.

Define a padded sequence $\bar{y}_0,\ldots,\bar{y}_{SD-1}$ by
\[
 \bar{y}_j
 :=\begin{cases}
 y_j,&0\le j\le J-1,\\
 y_{J-1},&J\le j\le SD-1.
 \end{cases}
\]
The original samples are unchanged, and the padding preserves the increment
bound because all appended differences vanish:
\begin{equation}\label{eq:padded-slow-sequence}
 \lvert\bar{y}_j-\bar{y}_{j-1}\rvert\le\eta,
 \qquad j=1,\ldots,SD-1.
\end{equation}
Every $j\in\{0,\ldots,SD-1\}$ has a unique representation
\[
 j=mD+k,\qquad 0\le m\le S-1,\quad 0\le k\le D-1.
\]

Assume first that $\eta>0$ and put
\[
 z_{m,k}:=\bigl\lfloor \bar{y}_{mD+k}/\eta\bigr\rfloor.
\]
For $m=0,\ldots,S-1$ and $k=1,\ldots,D-1$,
\eqref{eq:padded-slow-sequence} implies
$z_{m,k}-z_{m,k-1}\in\{-1,0,1\}$.  Indeed, put
$u=\bar{y}_{mD+k}/\eta$ and $v=\bar{y}_{mD+k-1}/\eta$.  Then
$\lvert u-v\rvert\le1$.
If $\lfloor u\rfloor\ge\lfloor v\rfloor+2$, then
$u\ge\lfloor v\rfloor+2>v+1$, a contradiction; the opposite inequality is
excluded in the same way.  Define the two increment bits
\[
 u_{m,k}^+:=\max\{z_{m,k}-z_{m,k-1},0\},\qquad
 u_{m,k}^-:=\max\{z_{m,k-1}-z_{m,k},0\}
\]
for $k=1,\ldots,D-1$, and append $u_{m,D}^+=u_{m,D}^-=0$.  These are bits,
and telescoping gives
\begin{equation}\label{eq:quantized-telescope}
 z_{m,k}
 =z_{m,0}
 +\sum_{i=1}^k u_{m,i}^+
 -\sum_{i=1}^k u_{m,i}^-.
\end{equation}

For each block, store the three real numbers
\[
 a_m:=\eta z_{m,0},
 \qquad
 \zeta_m^\pm
 :=\operatorname{bin}0.u_{m,1}^\pm\cdots u_{m,D}^\pm.
\]
Together with the single scale $\eta$, these entries form the latent
vector (or code)
\begin{equation}\label{eq:point-code}
 \bmxi
 :=(a_0,\ldots,a_{S-1},
      \zeta_0^+,\ldots,\zeta_{S-1}^+,
      \zeta_0^-,\ldots,\zeta_{S-1}^-,\eta)
 \in\R^{3S+1}.
\end{equation}
Although $a_m=\eta z_{m,0}$ in the target-dependent choice of the code,
$a_m$ is stored as a separate latent coordinate.  Hence the fixed affine
parameter map never needs to form this product.

We next construct a fixed routing module that recovers the block index $m$
and the within-block index $k$ from the input index $j$.  Define a continuous
piecewise linear function $Q_{S,D}$ that equals $m$ on
\[
 [mD,(m+1)D-1],\qquad m=0,\ldots,S-1,
\]
and interpolates linearly from $m$ to $m+1$ on
$[(m+1)D-1,(m+1)D]$ for $m=0,\ldots,S-2$.  Extend it constantly to
$(-\infty,0]$ and $[SD-1,\infty)$.  Thus $Q_{S,D}:\R\to\R$ is globally
continuous and, at every required integer $j=mD+k$, one has
\[
 Q_{S,D}(j)=m,\qquad k=j-DQ_{S,D}(j).
\]
The function has at most $2(S-1)$ breakpoints when $S\ge2$, since each of the
$S-1$ transition intervals contributes at most its left and right endpoint.
Its only slopes are
$0$ on the constant pieces and $1$ on the transition pieces, and all
breakpoints and slopes depend only on $(S,D)$.  Lemma~\ref{lem:slope-jump}
therefore writes $Q_{S,D}$ as a fixed hinge sum and realizes it with width at
most $5$ and depth at most
\[
 4(S-1)+2=4S-2.
\]
For $S=1$, set $Q_{1,D}=0$ and use one fixed hidden layer with terminal state
\[
 (j_+,j_-,Q_+,Q_-)
 :=(\varrho(j),\varrho(-j),0,0).
\]
The fixed scalar readout $Q_+-Q_-$ is zero.  After omitting this readout, the
same routing layer used below gives
\[
 m=\varrho(Q_+-Q_-)=0,
 \qquad
 k=\varrho\bigl(j_+-j_--D(Q_+-Q_-)\bigr)=j
\]
at every required integer $j=0,\ldots,D-1$.  Thus the exceptional case has
exactly the same terminal-state interface as the case $S\ge2$, and its two
hidden layers, including the routing layer, lie within the bound $4S+2$.

For $S\ge2$, apply the construction from
Lemma~\ref{lem:hinge-serialization} to $Q_{S,D}$ and omit only its final
scalar readout.  Its terminal hidden state before that readout contains the
transported input pair and the running-sum pair.  Together with the explicit
construction above for $S=1$, we may therefore in every case retain a
terminal state containing the signed pairs
\[
 (j_+,j_-)
 \quad\text{and}\quad
 (Q_+,Q_-)
 :=\bigl(\varrho(Q_{S,D}(j)),\varrho(-Q_{S,D}(j))\bigr).
\]
One fixed routing layer forms
\[
 m=\varrho(Q_+-Q_-),
 \qquad
 k=\varrho\bigl(j_+-j_--D(Q_+-Q_-)\bigr).
\]
At a required integer $j=mD+k$, both displayed preactivations are
nonnegative and equal the desired quotient and remainder.  Keeping the signed
input is necessary because the architecture has no skip connections.  The
serialized state together with this routing layer uses at most twelve
channels and hidden depth at most $4S+2$.

With the quotient and remainder available, we now load the three blockwise
code values.  To make their common affine domain explicit, let
\[
 \pi_a,\pi_+,\pi_-:\R^{3S+1}\to\R^S,
 \qquad
 \pi_\eta:\R^{3S+1}\to\R
\]
be the fixed coordinate projections for which
\[
 \pi_a(\bmxi)=(a_m)_{m=0}^{S-1},\qquad
 \pi_+(\bmxi)=(\zeta_m^+)_{m=0}^{S-1},\qquad
 \pi_-(\bmxi)=(\zeta_m^-)_{m=0}^{S-1},\qquad
 \pi_\eta(\bmxi)=\eta.
\]
Apply Lemma~\ref{lem:spline-loader} in parallel with parameter maps
$\calA^{\mathrm{spl}}\circ\pi_a$,
$\calA^{\mathrm{spl}}\circ\pi_+$, and
$\calA^{\mathrm{spl}}\circ\pi_-$.  At the integer input $m$, the outputs are
$a_m,\zeta_m^+,\zeta_m^-$.  Every generated spline coefficient is therefore
a fixed linear combination of at most three coordinates of
\eqref{eq:point-code}.  The nonnegative remainder
$k\in\{0,\ldots,D-1\}$ is copied by fixed identity channels through all
loader layers.  The three five-channel loaders, together with the transported
remainder, use at most sixteen channels.

At the interface to the decoders, omit the three scalar output layers and
merge their affine readouts with a following ReLU interface layer whose
pre-merge coefficients are fixed.  That
interface converts the possibly signed baseline to
\[
 a_{m,+}:=\varrho(a_m),
 \qquad
 a_{m,-}:=\varrho(-a_m),
\]
while the two nonnegative binary fractions and $k$ are copied.  Thus the state
entering the next module can be taken to be
\[
 (a_{m,+},a_{m,-},\zeta_m^+,\zeta_m^-,k).
\]
The omitted loader readouts are also fixed when $S=1$, because of the
signed-pair implementation in Lemma~\ref{lem:spline-loader}.  Hence their
merger with the fixed interface never multiplies two target-dependent
coefficients.

Next apply two parallel copies of the fixed decoder from
Lemma~\ref{lem:binary-prefix}, one to $(\zeta_m^+,k)$ and the other to
$(\zeta_m^-,k)$.  They return
\[
 U_{m,k}^+:=\sum_{i=1}^k u_{m,i}^+,\qquad
 U_{m,k}^-:=\sum_{i=1}^k u_{m,i}^-.
\]
While the fixed decoders run, $(a_{m,+},a_{m,-})$ is transported through
identity channels.  Two width-$20$ decoders together with these two baseline
channels use width at most $42$; we retain the looser bound $44$ in the ledger
below.  At the terminal decoder layer, retain the two accumulator coordinates
$U_{m,k}^+$ and
$U_{m,k}^-$ and omit the decoders' fixed scalar readouts.  The scalar output
layer of the complete network is
\begin{equation}\label{eq:point-decoder-output}
 (a_{m,+}-a_{m,-})+\eta U_{m,k}^+-\eta U_{m,k}^-.
\end{equation}
The multiplication by $\eta$ is not part of a nonlinear parameter map:
the two varying output-layer weights are $\pi_\eta(\bmxi)$ and
$-\pi_\eta(\bmxi)$.  They are therefore affine in the latent vector.  By
\eqref{eq:quantized-telescope}, the readout in
\eqref{eq:point-decoder-output} equals $\eta z_{m,k}$.  The defining
property of the floor function gives
\[
 0\le \bar{y}_{mD+k}-\eta z_{m,k}<\eta.
\]
For every original index $j<J$, one has $\bar{y}_j=y_j$, so this proves
\eqref{eq:point-fitting-error}.

If $\eta=0$, \eqref{eq:padded-slow-sequence} makes the padded sequence
constant.  Set every baseline $a_m$ equal to $\bar{y}_0$, set both binary
streams to zero, and set the scale coordinate to zero.  The same architecture
then returns $\bar{y}_0$ at every required integer and hence reproduces the
original data exactly.  The baseline coordinates are stored directly in the
latent vector, so the affine map never forms a product of $\eta$ with
another latent coordinate.

It remains to verify the architecture budget and the affine nature of the
complete assignment.  At the required integer inputs, the stagewise
quantities evolve as
\[
\begin{aligned}
 (m,k)
 &\longrightarrow (a_m,\zeta_m^+,\zeta_m^-,k)\\
 &\longrightarrow (a_{m,+},a_{m,-},\zeta_m^+,\zeta_m^-,k)\\
 &\longrightarrow (a_{m,+},a_{m,-},U_{m,k}^+,U_{m,k}^-)\\
 &\longrightarrow
 (a_{m,+}-a_{m,-})+\eta U_{m,k}^+-\eta U_{m,k}^-.
\end{aligned}
\]
The first arrow is implemented by the three parallel loaders while $k$ is
copied.  Their scalar affine readouts are merged into the fixed
split/copy interface layer implementing the second arrow.  The third arrow is
implemented by two fixed prefix decoders while the baseline pair is copied;
their fixed scalar readouts are omitted and their terminal nonnegative
accumulators are retained.  The last arrow is the single scalar output layer
with weights $1,-1,\eta,-\eta$ and zero bias.

\Needspace{18\baselineskip}
The corresponding resource ledger is given below.  The width budget is the
maximum of the modulewise width bounds, whereas the hidden-depth
contributions are added sequentially.  The target-dependent slots are
identified immediately after the table.
\begin{center}
\begingroup
\small
\renewcommand{\arraystretch}{1.16}
\setlength{\tabcolsep}{4pt}
\begin{tabularx}{0.95\linewidth}{
 @{}
 >{\raggedright\arraybackslash}p{.29\linewidth}
 >{\centering\arraybackslash}X
 >{\centering\arraybackslash}p{.12\linewidth}
 >{\centering\arraybackslash}p{.18\linewidth}
 @{}
}
\toprule
Module
& State at the module output
& Width
& Hidden depth\\
\midrule
Quotient and remainder routing
& $(m,k)$
& $\le12$
& $\le4S+2$\\
\addlinespace[2pt]
Three spline loaders and transported $k$
& terminal states of the three loaders and $k$
& $\le16$
& $\le2S+2$\\
\addlinespace[2pt]
Loader readout and split/copy interface
& $\begin{gathered}
   (a_{m,+},a_{m,-}),\\
   (\zeta_m^+,\zeta_m^-,k)
  \end{gathered}$
& $\le5$
& $1$\\
\addlinespace[2pt]
Two prefix decoders and transported baseline pair
& $\begin{gathered}
   (a_{m,+},a_{m,-}),\\
   (U_{m,k}^+,U_{m,k}^-)
  \end{gathered}$
& $\le44$
& $\le4D+2$\\
\addlinespace[2pt]
Final affine readout
& one scalar
& $1$
& $0$\\
\bottomrule
\end{tabularx}
\endgroup
\end{center}
The first module is completely fixed.  In the loader module, the three maps
$\calA^{\mathrm{spl}}\circ\pi_a$,
$\calA^{\mathrm{spl}}\circ\pi_+$, and
$\calA^{\mathrm{spl}}\circ\pi_-$ place every varying hinge coefficient into
its designated block of network slots; fixed zeros occupy the off-block
connections.  Hence each affected slot is affine in the common vector
$\bmxi$.  The two prefix decoders and all transport connections are fixed.
In the last row, the two varying output weights are
$\pi_\eta(\bmxi)$ and $-\pi_\eta(\bmxi)$, while the coefficients of
$a_{m,+}-a_{m,-}$ are fixed.  Concatenating these blockwise assignments in the
fixed global slot order defines a single affine map from $\R^{3S+1}$ into the
complete parameter space.  Thus every varying slot has been identified
explicitly.

The width bounds include every signed transport channel.  The quotient and
remainder routing layer is already included in the first row; the loader
readout and split/copy interface is displayed separately; alignment inside a
parallel block uses fixed identity/zero padding up to the common block depth
and contributes no additional depth beyond the module bounds already listed;
and the final affine readout is the network's scalar output layer, not a
hidden layer.  The sequential hidden depths therefore add to at most
\[
 (4S+2)+(2S+2)+1+(4D+2)
 =6S+4D+7
 \le10D+7
 \le10D+24.
\]
The maximum used width is $44\le96$, and all unused dense slots are fixed at
zero.  The slot-by-slot description following the table, together with
Lemma~\ref{lem:safe-composition}, shows that every parameter assignment is
affine in the code and that no product of two target-dependent coefficients is
created.  Every loader row uses at most three latent coordinates.  This
defines $\calA^{\mathrm{fit}}$ and completes the proof.
\end{proof}

\begin{figure}[ht]
\centering
\includegraphics[width=0.874796\textwidth]{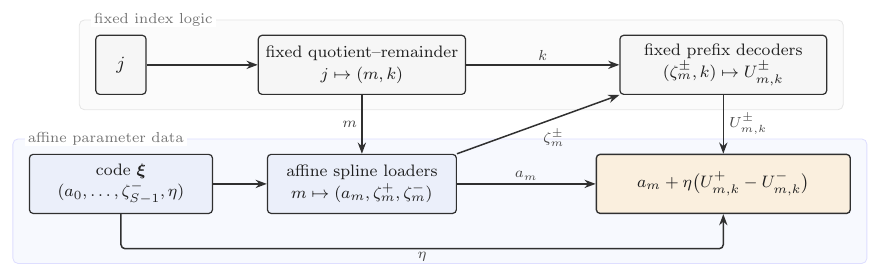}
\caption{Data flow in affine-coded point fitting.  Orange nodes contain the
latent vector or values loaded affinely from it; quotient computation,
prefix extraction, and routing are fixed.}
\label{fig:point-fitting}
\end{figure}

\subsection{Spatial addressing, bridge sequences, and boundary repair}
\label{sec:spatial-repair}

Lemma~\ref{lem:point-fitting} fits a slowly varying table indexed by integers.
To transfer this construction to a function on $[0,1]^d$, we must map each
spatial point to a table address, arrange the stored values so that adjacent
addresses vary slowly, and repair the thin regions in which a continuous
address map transitions between integer values.  The next three lemmas carry
out these tasks in order.  The first is a fixed-width specialization of the
step construction in
\cite[Proposition~3.1]{shijun:optimal:rate:in:width:and:depth}.

\begin{lemma}[One-dimensional address map]\label{lem:address-1d}
Let $q\in\N^+$, $K=q^2$, and $0<\delta\le1/(3K)$.  There exist a
one-input architecture $\Psi\in\Arch(20,8q+4)$ and a fixed parameter vector
$\bmtheta\in\R^{P_\Psi}$ such that
\begin{equation}\label{eq:address-map}
 \Psi_{\bmtheta}(x)=k
 \quad\text{for }x\in
 \bigl[\tfrac{k}{K},\tfrac{k+1}{K}
       -\delta\one_{\{k\le K-2\}}\bigr],
 \quad k=0,\ldots,K-1.
\end{equation}
All parameters depend only on $(K,\delta)$.
Moreover, in the realization constructed below, the terminal hidden state
contains nonnegative channels from which the address is obtained by a fixed
affine readout.  Consequently, this readout may be omitted and merged into a
following fixed affine layer.
\end{lemma}

\begin{proof}
Write the desired address as $k=mq+\ell$.  The first continuous staircase
recovers the coarse index $m$.  After subtracting $m/q$ from the input, a
second staircase recovers the fine index $\ell$.  Both staircases interpolate
continuously across gaps of width $\delta$ and are serialized at fixed width;
the final output is $qm+\ell$.

If $q=1$, then $K=1$.  Use one fixed hidden layer whose state contains
\[
 \bigl(\varrho(x),\varrho(-x),0,0\bigr),
\]
and use the fixed scalar readout that subtracts the fourth coordinate from
the third.  This realizes the required
constant address, belongs to $\Arch(20,8q+4)$, and, after the readout is
omitted, leaves an explicit terminal address pair for the later composition.
Hence assume $q\ge2$ for the staircase construction.

Construct a continuous piecewise linear function $\varphi_1$ that equals $m$ on
\[
 \bigl[\tfrac{m}{q},\tfrac{m+1}{q}-\delta\bigr],
 \qquad m=0,\ldots,q-2,
\]
and equals $q-1$ on the last coarse interval $[(q-1)/q,1]$.  On each
omitted gap interpolate linearly, and extend $\varphi_1$ constantly to
$(-\infty,0]$ and $[1,\infty)$.  Construct $\varphi_2$ on $[0,1/q]$ in the same
way so that it equals $\ell$ on
\[
 \bigl[\tfrac{\ell}{q^2},
       \tfrac{\ell+1}{q^2}-\delta\one_{\{\ell\le q-2\}}\bigr],
 \qquad \ell=0,\ldots,q-1.
\]
Thus the last plateau extends to the endpoint $1/q$.  Extend $\varphi_2$
constantly to $(-\infty,0]$ and $[1/q,\infty)$.  Both functions are now
continuous piecewise linear functions on all of $\R$.  Their breakpoints,
values, and slopes depend only on $(q,\delta)$, so
Lemma~\ref{lem:slope-jump} gives fixed hinge-sum representations for both
staircases.
The resulting address function is
\begin{equation*}
 \iota_{K,\delta}(x)
 =q\varphi_1(x)+\varphi_2\bigl(x-\tfrac{\varphi_1(x)}{q}\bigr).
\end{equation*}
Let $x$ belong to the interval in \eqref{eq:address-map} and write
$k=mq+\ell$, with $0\le m,\ell\le q-1$.  We first verify explicitly that
$x$ lies on the $m$th coarse plateau.

If $m\le q-2$ and $\ell\le q-2$, then
\[
 \frac{k+1}{K}-\delta
 =\frac{m}{q}+\frac{\ell+1}{q^2}-\delta
 \le\frac{m+1}{q}-\delta.
\]
If $m\le q-2$ and $\ell=q-1$, then the right endpoint of the fine interval is
exactly $(m+1)/q-\delta$.  Thus in both subcases the full fine interval lies
inside
\[
 \bigl[\tfrac{m}{q},\tfrac{m+1}{q}-\delta\bigr].
\]
If $m=q-1$, every fine interval lies inside the last coarse interval
$[(q-1)/q,1]$.  We have therefore proved in all cases that
$\varphi_1(x)=m$.

Subtracting $m/q$ gives
\[
 x-\tfrac{m}{q}
 \in\bigl[
 \tfrac{\ell}{q^2},
 \tfrac{\ell+1}{q^2}-\delta\one_{\{k\le K-2\}}
 \bigr].
\]
If $\ell\le q-2$, then necessarily $k\le K-2$, and this is exactly the
$\ell$th plateau interval of $\varphi_2$.  If $\ell=q-1$ and $m\le q-2$,
the right endpoint is $1/q-\delta$, which lies in the last plateau.  Finally,
if $m=\ell=q-1$, then $k=K-1$ and the right endpoint is $1/q$; the last
plateau was defined to include that endpoint.  Hence
\[
 \varphi_2\bigl(x-\tfrac{m}{q}\bigr)=\ell,
\]
and consequently $\iota_{K,\delta}(x)=qm+\ell=k$.

Each staircase has $2(q-1)$ breakpoints: the two endpoints of each of its
$q-1$ transition gaps.  By Lemma~\ref{lem:slope-jump}, each staircase is a
hinge sum with $2(q-1)$ hinges, so Lemma~\ref{lem:hinge-serialization}
realizes it with width $5$ and depth at most
\[
 4(q-1)+2=4q-2.
\]

We next make the carried states explicit.  At the end of the first serialized
staircase, retain
\[
 (x_+,x_-)
 \quad\text{and}\quad
 (H_{1,+},H_{1,-})
 :=\bigl(\varrho(\varphi_1(x)),
         \varrho(-\varphi_1(x))\bigr).
\]
One fixed interface layer forms the signed remainder
\[
 \begin{aligned}
 r_+
 &:=\varrho\bigl(x_+-x_- -\tfrac{H_{1,+}-H_{1,-}}{q}\bigr),\\
 r_-
 &:=\varrho\bigl(-x_++x_- +\tfrac{H_{1,+}-H_{1,-}}{q}\bigr).
 \end{aligned}
\]
Therefore
\[
 r_+-r_-=x-\frac{\varphi_1(x)}{q}.
\]
The pair $(H_{1,+},H_{1,-})$ is copied by fixed identity channels while the
second serialized staircase is evaluated at the signed input $(r_+,r_-)$.  A
serialized hinge sum uses at most five channels; retaining the two channels
for $\varphi_1$ raises this to at most seven.  The interface state itself has
only four channels.  Hence width $7$ already suffices for the described
states, and the stated width $20$ provides conservative uniform padding.

To justify the reuse of a signed input explicitly, every hinge
$\varrho(r-\kappa)$ in the scalar serialization is replaced by
\[
 \varrho(r_+-r_- -\kappa).
\]
The channels $(r_+,r_-)$ are copied by the same fixed identity updates used
in Lemma~\ref{lem:hinge-serialization}; its running signed sum is unchanged.
Thus the second serializer implements the same hinge formula at
$r=r_+-r_-$, without any multiplication or additional target-dependent
parameter slot.

The displayed interface is one fixed hidden layer.  The first initialization
layer of the second serializer accepts $(r_+,r_-)$ directly (rather than
recomputing them from a scalar input), forms its own fixed initial value
channels, and copies $(H_{1,+},H_{1,-})$.  All subsequent layers copy the
retained pair produced by the first staircase using fixed identities.
Consequently, no additional alignment layer is needed, and the total hidden
depth is at most
\[
 (4q-2)+1+(4q-2)=8q-3\le8q+4.
\]
At the terminal hidden state, write the second staircase value as the signed
pair $(H_{2,+},H_{2,-})$.  The fixed affine output row computes
\[
 q(H_{1,+}-H_{1,-})+(H_{2,+}-H_{2,-})
 =q\varphi_1+\varphi_2.
\]
All coefficients are fixed functions of $(q,\delta)$, and the final affine
row is fixed as well.  This produces $\Psi_{\bmtheta}$ and proves the
lemma without invoking an
external network construction for step functions.
\end{proof}

To combine the coordinate addresses into a single table index while reserving
room for bridge values, define, for
$\bmbeta=(\beta_1,\ldots,\beta_d)\in\{0,\ldots,K-1\}^d$,
\begin{equation}\label{eq:lex-index}
 \operatorname{row}(\bmbeta)
 :=\sum_{\ell=1}^{d-1}\beta_\ell K^{d-1-\ell},
 \qquad
 \operatorname{addr}(\bmbeta)
 :=2K\operatorname{row}(\bmbeta)+\beta_d.
\end{equation}
The empty sum is zero when $d=1$.  The factor $2K$ reserves $K$ unused
addresses after each row for bridge values.

\begin{lemma}[Slow bridge sequence]\label{lem:bridge}
Let $f\in C([0,1]^d)$ and $K\in\N^+$.  There exist $2K^d$ samples
$y_0,\ldots,y_{2K^d-1}$ such that
\begin{equation}\label{eq:bridge-target}
 y_{\operatorname{addr}(\bmbeta)}=f(\bmbeta/K),\qquad
 \bmbeta\in\{0,\ldots,K-1\}^d,
\end{equation}
and
\begin{equation}\label{eq:bridge-increment}
 \lvert y_j-y_{j-1}\rvert
 \le\omega_f\bigl(\tfrac{\sqrt d}{K}\bigr),
 \qquad j=1,\ldots,2K^d-1.
\end{equation}
\end{lemma}

\begin{proof}
In each row, the first $K$ addresses store grid values, while the next
$K$ addresses form a linear bridge from the last value in the current row to
the first value in the next row.  Each bridge increment is the endpoint
difference divided by $K$, and the integer scaling property of the modulus
bounds it at the desired scale $\sqrt d/K$.  We now define the sequence
explicitly.

For every $i\in\{0,\ldots,K^{d-1}-1\}$, let
$\bmbeta'(i)\in\{0,\ldots,K-1\}^{d-1}$ be the unique prefix
multi-index with lexicographic number $i$.  Concretely, if
$\bmbeta'(i)=(\beta_1,\ldots,\beta_{d-1})$, then its base-$K$ digits satisfy
\[
 i=\sum_{\ell=1}^{d-1}\beta_\ell K^{d-1-\ell}.
\]
Consequently, adjoining the last digit $k$ gives
$\operatorname{addr}((\bmbeta'(i),k))=2Ki+k$.  When $d=1$,
$\bmbeta'(i)$ is the empty vector and there is only one row.  Define
\[
 v_{i,k}
 :=f\bigl(\tfrac{1}{K}(\bmbeta'(i),k)\bigr),
 \qquad k=0,\ldots,K-1,
\]
and assign the actual grid values by
\[
 y_{2Ki+k}:=v_{i,k},\qquad k=0,\ldots,K-1.
\]
This is exactly \eqref{eq:bridge-target}.  Two consecutive actual
values $v_{i,k-1}$ and $v_{i,k}$ are sampled at points whose Euclidean
distance is $1/K$, and hence their difference is at most
$\omega_f(1/K)\le\omega_f\bigl(\tfrac{\sqrt d}{K}\bigr)$.

For $i<K^{d-1}-1$, set, for $u=1,\ldots,K$,
\begin{equation*}
 y_{2Ki+K-1+u}
 :=\bigl(1-\tfrac{u}{K}\bigr)v_{i,K-1}
   +\tfrac{u}{K}v_{i+1,0}.
\end{equation*}
These are the $K$ bridge addresses
$2Ki+K,\ldots,2K(i+1)-1$.  Every consecutive bridge increment, including
the first increment from $v_{i,K-1}$, has magnitude
\[
 \frac{1}{K}\lvert v_{i+1,0}-v_{i,K-1}\rvert
 \le\frac{1}{K}\omega_f(\sqrt{d})
 \le\omega_f\bigl(\tfrac{\sqrt d}{K}\bigr).
\]
The first inequality uses the diameter $\sqrt{d}$ of the cube; the second
is exactly the integer scaling property recorded after
\eqref{eq:modulus}, namely
\[
 \omega_f(\sqrt{d})
 =\omega_f\bigl(K\tfrac{\sqrt d}{K}\bigr)
 \le K\omega_f\bigl(\tfrac{\sqrt d}{K}\bigr).
\]
At $u=K$, the bridge
value equals $v_{i+1,0}$, so the following transition to the first actual
value of row $i+1$ has size zero.

Let $i_\ast:=K^{d-1}-1$.  For the last row, set
\[
 y_{2Ki_\ast+K-1+u}:=v_{i_\ast,K-1},
 \qquad u=1,\ldots,K.
\]
These are precisely the remaining indices
$2Ki_\ast+K,\ldots,2K^d-1$, and their increments are zero.  The preceding
cases exhaust all pairs of adjacent addresses: consecutive actual samples
within one row, the last actual sample to the first bridge sample, consecutive
bridge samples, the final bridge sample to the first actual sample of the next row,
and the constant tail after the final row.  When $K=1$, the within-row and
nonconstant-bridge cases are empty, and the same assignments still give the
two required samples.  Hence these cases prove
\eqref{eq:bridge-increment}.
\end{proof}

The address map is constant and correct away from thin transition strips.  The
boundary-repair argument below adapts the horizontal shift and median
construction in
\cite[Theorem~2.1 and Lemmas~3.2--3.4]{shijun:smooth:functions}.  We include
the proof to verify that, in the present affine latent setting, the repair
preserves affine dependence on the code while retaining explicit width and
depth bounds.  To keep the affine map universal, $\delta$ must be fixed
independently of $f$.

For positive integers $d,K$ and $0<\delta\le 1/(3K)$, define the transition
set
\begin{equation*}
 \Omega_{K,\delta}
 :=\bigcup_{j=1}^d\bigcup_{k=1}^{K-1}
 \left\{\bmx\in[0,1]^d:x_j\in(k/K-\delta,k/K)\right\}.
\end{equation*}
Thus $\Omega_{K,\delta}$ is the union of the one-sided strips on
which at least one coordinate address may be changing.

Figure~\ref{fig:boundary-shift} illustrates the one-dimensional geometry
behind the repair.

\begin{figure}[ht]
\centering
\includegraphics[width=0.7258\textwidth]{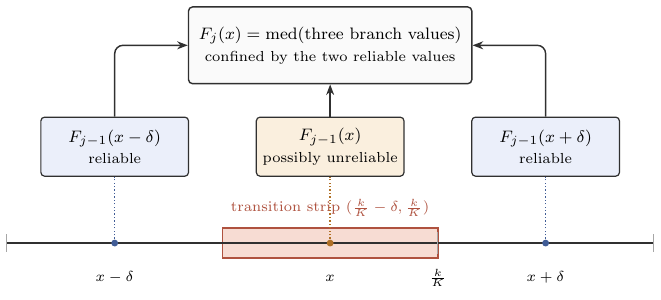}
\caption{A representative one-dimensional transition.  If $x$ lies in a
transition strip, the shifts $x-\delta$ and $x+\delta$ lie on its two sides.
More generally, among the three translated branches at least two are reliable,
so the median remains between the two reliable values, regardless of the
remaining branch.}
\label{fig:boundary-shift}
\end{figure}

\begin{lemma}[Boundary repair]\label{lem:boundary-repair}
Let $d,K,M,N,L\in\N^+$ and $0<\delta\le1/(3K)$.
Let $\widetilde\Phi$ be a $d$-input architecture in $\Arch(N,L)$ and
$\widetilde\calA\in\aff(M,P_{\widetilde\Phi})$ be given.  There are
$\Phi\in\Arch(3^d(N+4),L+2d)$ and
$\calA\in\aff(M,P_\Phi)$, constructed solely from
$(\widetilde\Phi,\widetilde\calA,K,\delta)$, with the following property.
For every $f\in C([0,1]^d)$, $\varepsilon\ge0$, and
$\bmxi_f\in\R^M$, set
$\widetilde{\bmtheta}_f:=\widetilde\calA(\bmxi_f)$.  If
\[
 \lvert f(\bmx)
 -\widetilde\Phi_{\widetilde{\bmtheta}_f}(\bmx)\rvert
 \le\varepsilon,\qquad
 \bmx\in[0,1]^d\setminus\Omega_{K,\delta},
\]
then, with $\bmtheta_f:=\calA(\bmxi_f)$,
\begin{equation*}
 \lvert f(\bmx)-\Phi_{\bmtheta_f}(\bmx)\rvert
 \le\varepsilon+d\omega_f(\delta),\qquad \bmx\in[0,1]^d.
\end{equation*}
\end{lemma}

\begin{proof}
We repair one coordinate at a time.  For the coordinate currently under
consideration, at least two of the three shifted points corresponding to
$-\delta$, $0$, and $+\delta$ remain in the cube and avoid the transition
strips in that coordinate.  The later coordinates already satisfy the
induction-domain restrictions, so these two points belong to the preceding
reliable set.  Their network values are therefore close to the target at the
unshifted point.  Taking the median prevents the third, possibly unreliable,
value from leaving the interval determined by the two reliable ones.  Repeating
this argument over all $d$ coordinates adds one modulus term per coordinate.

\proofstep{1}{Realize the median by two explicit ReLU layers.}

For real $a,b,c$, set
\[
 h_{a,b}:=\varrho(a-b),\qquad
 u:=\varrho(b+h_{a,b}-c),\qquad
 v:=\varrho(c-a+h_{a,b}),
\]
and let $t:=a+b-c$.  Thus
$t_+=\varrho(a+b-c)$ and $t_-=\varrho(-a-b+c)$ under the
positive and negative part convention fixed in Section~\ref{sec:notation}.
Since $b+h_{a,b}=\max\{a,b\}$ and $a-h_{a,b}=\min\{a,b\}$,
\begin{equation*}
 \med(a,b,c)=t_+-t_--u+v.
\end{equation*}
Indeed, let $a_{\max}:=\max\{a,b\}$ and
$a_{\min}:=\min\{a,b\}$.  Then
\begin{align*}
 t_+-t_--u+v
 &=a_{\max}+a_{\min}-c-(a_{\max}-c)_++(c-a_{\min})_+\\
 &=\min\{a_{\max},c\}+(a_{\min}-c)_+\\
 &=\med(a,b,c).
\end{align*}
The last identity covers simultaneously the three alternatives
$c\le a_{\min}$, $a_{\min}\le c\le a_{\max}$, and $c\ge a_{\max}$.
We shall also use the following elementary interval consequence.  If
$I\subseteq\R$ is an interval, $r\le s$, and $r,s\in I$, while $w$ is
arbitrary, then
$\med(r,s,w)\in[r,s]\subseteq I$: the median is $r$, $w$, or $s$ according as
$w\le r$, $r\le w\le s$, or $w\ge s$.
The first hidden layer computes $a_+:=\varrho(a)$,
$a_-:=\varrho(-a)$, and analogously $b_\pm,c_\pm$, together with $h_{a,b}$,
using seven channels.  The second computes $(t_+,t_-,u,v)$ using four
channels, and the final affine row is $[1,-1,-1,1]$ with respect to the
displayed ordering.  Thus the median block adds two hidden layers and has
width at most seven.

\proofstep{2}{Set up the induction domains.}

For $j=0,\ldots,d$, let
\[
 \calD_j:=\left\{\bmx\in[0,1]^d:
 x_{j+1},\ldots,x_d\text{ avoid all transition strips}\right\}.
\]
There is no restriction when $j=d$, so $\calD_d=[0,1]^d$, whereas
$\calD_0=[0,1]^d\setminus\Omega_{K,\delta}$.

For $d=2$, Figure~\ref{fig:boundary-repair-domains} shows how these reliable
domains expand as the two coordinates are repaired.

\begin{figure}[ht]
\centering
\includegraphics[width=0.902\textwidth]{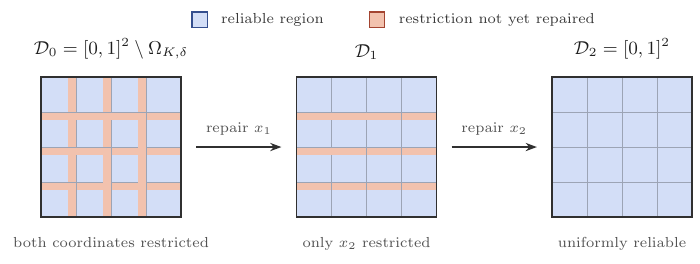}
\caption{Coordinatewise boundary repair in two dimensions.  Initially the
approximation is reliable outside the vertical and horizontal transition
strips.  Repairing the first coordinate removes the vertical restrictions,
and repairing the second coordinate yields a uniform estimate on the whole
square.}
\label{fig:boundary-repair-domains}
\end{figure}

We first construct the architecture and affine map independently of the code
selected for $f$.  Set
$(\Phi_0,\calA_0):=(\widetilde\Phi,\widetilde\calA)$.  Recursively, after
$(\Phi_{j-1},\calA_{j-1})$ has been constructed, form $\Phi_j$ from three
translated copies of $\Phi_{j-1}$ followed by the fixed two-hidden-layer
median block from Step~1.  For an arbitrary code $\bmxi$, define
$\calA_j(\bmxi)$ by copying the three affine slot assignments of
$\calA_{j-1}(\bmxi)$.  In the branches evaluated at
$\bmx-\delta\bm{e}_j$, $\bmx$, and $\bmx+\delta\bm{e}_j$, respectively,
replace each first-layer bias $\bmb(\bmxi)$ by
\[
 \bmb(\bmxi)-\delta\bmW(\bmxi)\bm{e}_j,
 \qquad \bmb(\bmxi),
 \qquad \bmb(\bmxi)+\delta\bmW(\bmxi)\bm{e}_j.
\]
The three scalar branch readouts are merged into the first fixed median layer,
and the remaining median parameters are fixed.  Every resulting slot is
affine in $\bmxi$: the translated biases are affine because $\delta$ is
fixed, and the readout merge has a fixed side.  Thus $(\Phi_j,\calA_j)$
depends only on $(\widetilde\Phi,\widetilde\calA,K,\delta)$ and not on $f$.

Having fixed these target-independent pairs, for the selected code
$\bmxi_f$, put
\[
 F_j(\bmx):=(\Phi_j)_{\calA_j(\bmxi_f)}(\bmx),
 \qquad j=0,\ldots,d.
\]
Then $F_0=\widetilde\Phi_{\widetilde{\bmtheta}_f}$, and the preceding
recursion yields
\[
 F_j(\bmx)
 =\med\left(
 F_{j-1}(\bmx-\delta\bm{e}_j),
 F_{j-1}(\bmx),
 F_{j-1}(\bmx+\delta\bm{e}_j)
 \right).
\]
Assume inductively that
\begin{equation}\label{eq:repair-induction}
 \lvert F_{j-1}(\bmx)-f(\bmx)\rvert
 \le\varepsilon+(j-1)\omega_f(\delta),
 \qquad \bmx\in\calD_{j-1}.
\end{equation}
Here $\bm{e}_j$ is the $j$th coordinate vector.  A neural network is defined
on all of $\R^d$, so a branch shifted outside the cube still has a numerical
value; that branch will be treated as unreliable.

\proofstep{3}{Show that two of the three branches are reliable.}

Fix $\bmx\in\calD_j$.  If $K=1$, there are no transition strips and
$\delta\le1/3$.  If $x_j<\delta$, the shifts $x_j$ and $x_j+\delta$ lie in
$[0,1]$; if $x_j>1-\delta$, the shifts $x_j-\delta$ and $x_j$ lie in
$[0,1]$; and if $\delta\le x_j\le1-\delta$, all three shifts do.  Since
there are no strips, at least two shifted points belong to
$\calD_{j-1}=[0,1]^d$, which proves the required reliability in this case.
Assume henceforth that $K\ge2$.  We first isolate the strip
geometry.  At least two of the three shifted coordinates lie in
$[0,1]$ and outside the $j$th transition strips.  To see this, if
$x_j<\delta$, then $x_j$ and $x_j+\delta$ are both admissible and lie before
the first strip; indeed,
$x_j+\delta<2\delta\le1/K-\delta$ because $3\delta\le1/K$.  If
$x_j>1-\delta$, the two admissible values are $x_j-\delta$ and $x_j$; both lie
beyond the last strip because $x_j-\delta>1-2\delta>(K-1)/K$.  In the remaining case
$\delta\le x_j\le1-\delta$, all three shifts lie in $[0,1]$.
One open strip of width $\delta$ cannot contain two of the
$\delta$-separated shifts.  Moreover, points in two distinct open strips are
separated by more than $1/K-\delta\ge2\delta$, while the three shifts span
only $2\delta$.  Hence two different strips cannot contain two of the shifts
either.
The fact that the strips are open matters only at equality: a shifted
coordinate lying exactly on a strip endpoint is reliable.  Hence at most one
shift is unreliable in every case.

For each reliable shift $\bmx'$, its $j$th coordinate avoids every strip,
and its coordinates $j+1,\ldots,d$ equal those of $\bmx\in\calD_j$.  Hence
$\bmx'\in\calD_{j-1}$, so the induction hypothesis applies.  Together with
\eqref{eq:modulus}, the induction hypothesis gives
\[
 \lvert F_{j-1}(\bmx')-f(\bmx)\rvert
 \le\varepsilon+(j-1)\omega_f(\delta)+
 \lvert f(\bmx')-f(\bmx)\rvert
 \le\varepsilon+j\omega_f(\delta).
\]
Both reliable values lie in the interval centered at $f(\bmx)$ with radius
$\varepsilon+j\omega_f(\delta)$.  The median of three real numbers lies
between the two reliable values by the interval fact proved in Step~1,
regardless of the third branch.  Thus \eqref{eq:repair-induction}
holds with $j$ in place of $j-1$.  Induction from $\calD_0$ to
$\calD_d=[0,1]^d$ proves the error estimate.

\proofstep{4}{Verify architecture size and affine dependence.}

Define width bounds recursively by
\[
 w_0:=N,\qquad
 w_j:=\max\{3w_{j-1},7\},\quad j=1,\ldots,d.
\]
The three copies of the previous network run in parallel, while the following
median block uses width at most seven.  Hence
$\width(\Phi_j)\le w_j$ for every $j$.  Since
$w_j\le3w_{j-1}+7$, solving the recurrence gives
\[
 w_d\le3^dN+\frac{7}{2}(3^d-1)<3^d(N+4).
\]
Each repair adds two hidden layers, so the final depth is at most $L+2d$.

The recursion for $(\Phi_j,\calA_j)$ in Step~2 already specifies every affine
slot assignment: translated first-layer biases are affine, copied branches
repeat the same coordinate functions, each merge of a branch readout into the
median block has one fixed side, and all remaining median parameters are
fixed.  Lemma~\ref{lem:safe-composition} therefore applies at every
interface.  Taking
$(\Phi,\calA):=(\Phi_d,\calA_d)$ gives the asserted target-independent pair,
and $F_d=\Phi_{\calA(\bmxi_f)}$ has the error proved in Step~3.
\end{proof}

\begin{remark}[Exact real arithmetic]\label{rem:exact-real}
The construction is formulated over exact real arithmetic.  The latent vector
$\bmxi_f$ contains binary fractions whose number of relevant digits increases
with $P$.  Moreover, the fixed decoder uses coefficients of size
$\delta_D^{-1}=2^{D+1}$ in its bit gate.  Accordingly,
Theorem~\ref{thm:upper} neither bounds the bit length of the latent scalars nor
controls the dynamic range of the fixed decoder weights; it also provides no
numerical stability guarantee.  If every latent scalar is restricted to $b$
bits, the model is additionally constrained by the total latent information
$Mb$.  As discussed in Section~\ref{sec:exact-real-scope}, we keep these
finite-precision effects separate from the main theorem over exact real
parameters.
\end{remark}

\subsection{Construction under integer budgets}
\label{sec:integer-construction}

All analytic and network-theoretic ingredients are now available.  It remains
to assemble them and coordinate their integer resources: $S$ counts latent
code blocks, $D$ is the number of decoder stages, and $q$ determines the
spatial mesh.  We first introduce the three abbreviations used below.
For $d,S,D\in\N^+$, set
\begin{equation}\label{eq:qK}
 q:=\bigl\lfloor(SD/2)^{1/(2d)}\bigr\rfloor,
 \qquad K:=q^2,
 \qquad \widehat{L}_{d,D}:=18D+30+2d.
\end{equation}
When $SD=1$, this formula gives $q=K=0$ and therefore does not yield a usable
spatial mesh.  The proposition below assumes $q\ge1$, while
Theorem~\ref{thm:upper} invokes it only with $q\ge2$.  It packages the complete
construction.  For reference, its resource ledger is
\[
 \bmxi_f\in\R^{3S+1},
 \qquad
 P_\Phi\le B_dD,
 \qquad
 K=\bigl\lfloor(SD/2)^{1/(2d)}\bigr\rfloor^2.
\]
Thus the product of the number of code blocks and the decoder length determines
the approximation scale.  The final subsection converts this exact integer
form into the prescribed budgets $(M,P)$.

\begin{proposition}[Universal construction in the regime $q\ge1$]\label{prop:integer-upper}
Let $d,S,D\in\N^+$ with $S\le D$, and let $q,K,\widehat{L}_{d,D}$ be the
quantities in \eqref{eq:qK}.  Assume $q\ge1$.  There exist an architecture
$\Phi\in\Arch(N_d,\widehat{L}_{d,D})$ with input dimension $d$ and an affine map
\[
 \calA_0\in\aff(3S+1,P_\Phi),
\]
both independent of $f$.  For every $f\in C([0,1]^d)$, there exists
$\bmxi_f\in\R^{3S+1}$ such that, with $\bmtheta_f:=\calA_0(\bmxi_f)$, one has
\begin{equation}\label{eq:integer-upper}
 \left\lVert f-\Phi_{\bmtheta_f}\right\rVert_{L^\infty([0,1]^d)}
 \le(d+2)\omega_f\bigl(\tfrac{\sqrt d}{K}\bigr).
\end{equation}
Moreover,
\begin{equation}\label{eq:integer-P}
 P_\Phi\le B_dD.
\end{equation}
\end{proposition}

\begin{proof}
We now combine the preceding modules.  Put
$\delta:=1/(3K)$ and
$\eta:=\omega_f\bigl(\tfrac{\sqrt d}{K}\bigr)$.  The grid has $K$ cells per
coordinate, and the definition of $q$ ensures that its $2K^d$ target and
bridge samples fit into $S$ blocks of length $D$.  A fixed address network maps
points outside the transition strips to grid indices, and
Lemma~\ref{lem:point-fitting} decodes the corresponding samples with error
$\omega_f\bigl(\tfrac{\sqrt d}{K}\bigr)$.  Spatial discretization contributes
one more copy of this modulus, while Lemma~\ref{lem:boundary-repair}
contributes at most $d$ additional copies at the smaller scale $\delta$.  The
point-fitting code uses exactly $3S+1$ coordinates.  After composing the
modules, we use uniform upper bounds on their actual widths and depths to
count all dense slots, including unused connections assigned zero.

Figure~\ref{fig:upper-roadmap} summarizes the construction.  The five steps
below implement its arrows and track the associated constants.

\begin{figure}[ht]
\centering
\includegraphics[width=0.8276\textwidth]{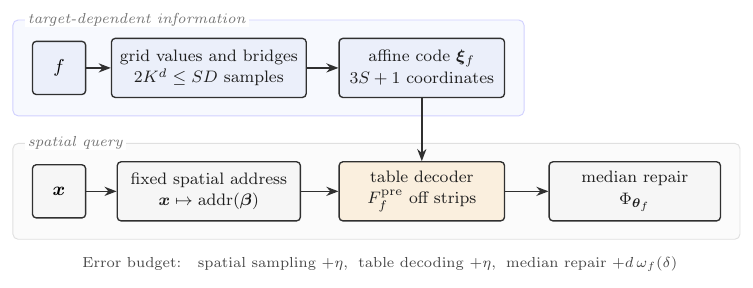}
\caption{Proof architecture for the constructive upper bound.  The target
determines only the sampled values and latent vector.  The spatial address
network, decoder architecture, and affine loading rule are fixed, while the
generated decoder parameters vary affinely with the latent vector.  The
capacity condition
$2K^d\le SD$ couples the number $S$ of code blocks with the decoder length
$D$.}
\label{fig:upper-roadmap}
\end{figure}

\proofstep{1}{Choose the grid and a transition width fixed across targets.}

The choice $\delta=1/(3K)$ depends only on $(d,S,D)$ through $K$ and is
therefore independent of $f$.  Since $K=q^2$ and $q$ is defined by
\eqref{eq:qK},
\begin{equation}\label{eq:capacity-check}
 2K^d=2q^{2d}\le SD.
\end{equation}
If $q=1$, then $K=1$; the constant address module in
Lemma~\ref{lem:address-1d}, the bridge construction, and all later estimates
remain valid.  Thus this edge case requires no separate architecture.

\proofstep{2}{Construct the spatial address.}

Apply Lemma~\ref{lem:address-1d} to each coordinate in parallel.  Outside
$\Omega_{K,\delta}$, the retained address channels represent
\[
 \bmbeta=(\beta_1,\ldots,\beta_d),
 \quad\text{where}\quad
 x_j\in\bigl[\tfrac{\beta_j}{K},
 \tfrac{\beta_j+1}{K}-\delta\one_{\{\beta_j\le K-2\}}\bigr].
\]
The $d$ copies run in parallel: their rows in the first layer are stacked
vertically, and connections between distinct copies are fixed at zero.  Their
total width is at most $20d$ and their depth is at most $8q+4$.  The scalar
address $\operatorname{addr}(\bmbeta)$ in \eqref{eq:lex-index} is a
fixed affine function of $\bmbeta$, so forming it requires no hidden layer.
This parallel construction can still be incorporated into a scalar-output
architecture.  We omit the $d$ fixed scalar readout rows of the coordinate
copies, retain their terminal hidden states, and substitute those readouts
into the fixed affine formula for $\operatorname{addr}(\bmbeta)$.  That fixed
scalar readout is in turn merged into the first fixed affine layer of the
quotient and remainder block of $\Phi^{\mathrm{fit}}$.  Both sides of this
interface are fixed, so the merger adds neither a hidden layer nor a
code-dependent slot.  Outside $\Omega_{K,\delta}$, the scalar supplied to the
table module is exactly $\operatorname{addr}(\bmbeta)$.

\proofstep{3}{Load and decode the function values.}

Use Lemma~\ref{lem:bridge} to construct $2K^d$ samples.  The quantity
$\eta=\omega_f\bigl(\tfrac{\sqrt d}{K}\bigr)$ bounds every adjacent
increment.  Since \eqref{eq:capacity-check} holds,
Lemma~\ref{lem:point-fitting} applies with $J=2K^d$ and gives a latent vector
$\bmxi_f\in\R^{3S+1}$.  Set
\[
 \bmtheta_f^{\mathrm{fit}}
 :=\calA^{\mathrm{fit}}(\bmxi_f),
 \qquad
 F_f^{\mathrm{tab}}
 :=\Phi^{\mathrm{fit}}_{\bmtheta_f^{\mathrm{fit}}}.
\]
Thus $F_f^{\mathrm{tab}}$ is the function produced by the point fitting
module, and
\begin{equation}\label{eq:table-approx}
 \lvert F_f^{\mathrm{tab}}(j)-y_j\rvert\le\eta,
 \qquad j=0,\ldots,2K^d-1.
\end{equation}
The degenerate case $\eta=0$ requires no change in the universal pair.  Indeed,
$f$ is then constant on the connected cube: any two points can be joined by
finitely many segments of length at most $\sqrt d/K$.  We keep the same pair
$(\Phi^{\mathrm{fit}},\calA^{\mathrm{fit}})$ and use the $\eta=0$ code from
Lemma~\ref{lem:point-fitting}, namely constant baselines, zero binary
fractions, and zero scale.  Thus no target-dependent replacement of the
architecture or affine map is made.

Compose the fixed address module with the universal table pair
$(\Phi^{\mathrm{fit}},\calA^{\mathrm{fit}})$ using the address-to-table merger
described in Step~2.  The address readout and the quotient part of the
point fitting network are both fixed, so
Lemma~\ref{lem:safe-composition} preserves affine dependence on the latent
vector without an additional interface layer.  Denote the resulting
target-independent pair by $(\widetilde\Phi,\widetilde\calA)$.  For the
selected code, set
\[
 \bmtheta_f^{\mathrm{pre}}
 :=\widetilde\calA(\bmxi_f),
 \qquad
 F_f^{\mathrm{pre}}
 :=\widetilde\Phi_{\bmtheta_f^{\mathrm{pre}}}.
\]
If $\bmx\notin\Omega_{K,\delta}$ and its cell index is $\bmbeta$,
\eqref{eq:bridge-target} and \eqref{eq:table-approx} imply
\begin{align*}
 \lvert f(\bmx)-F_f^{\mathrm{pre}}(\bmx)\rvert
 &\le \lvert f(\bmx)-f(\bmbeta/K)\rvert
      +\lvert f(\bmbeta/K)-F_f^{\mathrm{tab}}(\operatorname{addr}(\bmbeta))\rvert\\
 &\le\omega_f\bigl(\tfrac{\sqrt d}{K}\bigr)+\eta
 =2\eta.
\end{align*}

\proofstep{4}{Repair all transition strips.}

Apply Lemma~\ref{lem:boundary-repair} to the preliminary pair
$(\widetilde\Phi,\widetilde\calA)$ with latent-vector dimension $3S+1$.
It produces
the architecture $\Phi$ and affine map $\calA_0$ in the proposition.  Since
\[
 \delta=\frac{1}{3K}\le\frac{\sqrt{d}}{K},
\]
monotonicity of the modulus gives $\omega_f(\delta)\le\eta$.  The repaired
realized function, with $\bmtheta_f:=\calA_0(\bmxi_f)$, therefore satisfies
\[
 \left\lVert f-\Phi_{\bmtheta_f}\right\rVert_{L^\infty([0,1]^d)}
 \le2\eta+d\omega_f(\delta)
 \le(d+2)\eta,
\]
which is \eqref{eq:integer-upper}.

\proofstep{5}{Verify width, depth, affine dependence, and total parameter count.}

The $d$ coordinate address copies run in parallel, so their widths add to at
most $20d$ while their common depth remains at most $8q+4$.  The resulting
address module is then composed sequentially with the point fitting module of
width $96$.  Sequential modules take the maximum of their widths
and add their depths.  Therefore, before boundary repair, the combined module
has width at most $\max\{20d,96\}$ and depth at most
\[
 (8q+4)+(10D+24).
\]
Because $S\le D$, one has
\[
 q\le(D^2/2)^{1/(2d)}
 \le D.
\]
Hence the preliminary depth is at most $18D+28$.  Boundary repair increases
the width to at most
\[
 3^d\bigl(\max\{20d,96\}+4\bigr)=N_d
\]
and the depth to at most
$18D+28+2d\le\widehat{L}_{d,D}$.  Thus the constructed architecture is fully
connected and belongs to $\Arch(N_d,\widehat{L}_{d,D})$.  Because no further
enlargement to the upper bounds $N_d$ and $\widehat{L}_{d,D}$ is needed, the
architecture and the codomain of its affine parameter map remain unchanged.
All fixed-zero embeddings within the dense layers are included in the
parameter count.

For affine dependence, Lemma~\ref{lem:point-fitting} uses $3S+1$
target-dependent coordinates, while the address, spatial routing, and
transition width are fixed.  Every interface has
at least one fixed side, as checked in Step~3 and formalized in
Corollary~\ref{cor:state-interface}.  Thus no product of two target-dependent
coefficients is introduced.  Under the ordering in
\eqref{eq:parameter-count}, these affine assignments of all dense weights and
biases define a single universal map
$\calA_0\in\aff(3S+1,P_\Phi)$.  In particular, its codomain is the parameter
space of the same architecture $\Phi$ appearing in the proposition.

Finally, write $N:=N_d$ and use \eqref{eq:parameter-count}.  Replacing
each actual hidden width and the actual depth by their upper bounds can only
increase the count, and
$\widehat{L}_{d,D}=18D+30+2d$.  Since $D\ge1$,
the first hidden layer contributes at most $N(d+1)$ slots, the remaining
$\widehat{L}_{d,D}-1=18D+29+2d$ hidden-to-hidden layers contribute at most
$(18D+29+2d)N(N+1)$ slots, and the scalar output layer contributes $N+1$.
To bound this expression by a constant multiple of $D$, use $D\ge1$ to obtain
\[
 18D+29+2d\le(47+2d)D,
\]
and
\[
 N(d+1)+(N+1)=N(d+2)+1
 \le\bigl[N(d+2)+1\bigr]D.
\]
Consequently,
\begin{align*}
 P_\Phi
 &\le N(d+1)+(18D+29+2d)N(N+1)+(N+1)\\
 &\le\bigl[(47+2d)N(N+1)+N(d+2)+1\bigr]D\\
 &\le B_dD.
\end{align*}
This proves \eqref{eq:integer-P} and completes the proof.
\end{proof}

\subsection{Matching the prescribed budgets}
\label{sec:budget-selection}

Proposition~\ref{prop:integer-upper} is stated in terms of the internal
integers $(S,D)$.  No further network construction is needed.  We now choose
these integers from $(M,P)$, verify the latent-coordinate and parameter-slot
budgets separately, and translate the product $SD$ into the final mesh scale.
The choices and the resulting capacity relation are
\[
 D=\lfloor P/B_d\rfloor,
 \qquad
 3S+1\le\min\{M,P\},
 \qquad
 SD\ge\frac{P\min\{M,P\}}{4B_d^2}.
\]
Thus the choice of $D$ is governed by the parameter-slot budget, the choice of
$S$ by the latent-coordinate budget, and their product determines the spatial
resolution.

\begin{proof}[Proof of Theorem~\ref{thm:upper}]
Let $\widetilde{M}:=\min\{M,P\}$ be the effective latent coordinate budget.
We allocate
$3S+1\le\widetilde{M}$ coordinates
to $S$ code blocks and choose the decoder length $D$ proportional to $P$.
The explicit threshold $P_d^\star$ ensures that the grid integer $q$ is at
least two.  The floor estimates below yield
$SD\ge P\widetilde{M}/(4B_d^2)$ and therefore bound the mesh size by
$C_d(P\widetilde{M})^{-1/d}$.  Proposition~\ref{prop:integer-upper} then gives
the required architecture, while precomposition with a coordinate projection
extends its affine map from $\R^{3S+1}$ to the prescribed domain $\R^M$.

Set
\begin{equation*}
 D:=\lfloor P/B_d\rfloor,
 \qquad
 S:=\min\bigl\{\lfloor(\widetilde{M}-1)/3\rfloor,D\bigr\}.
\end{equation*}
The assumptions give $\widetilde{M}\ge4$ and
$P\ge P_d^\star=2^{2d+2}B_d$.  We will repeatedly use the elementary estimate
\[
 z\ge2\quad\Longrightarrow\quad
 \lfloor z\rfloor\ge z-1\ge z/2.
\]
Applying this estimate to $P/B_d\ge2^{2d+2}\ge2$ gives
\begin{equation}\label{eq:D-lower}
 D\ge\frac{P}{2B_d}\ge2^{2d+1}.
\end{equation}
Also $1\le S\le D$ and $3S+1\le\widetilde{M}\le M$.

We next derive a lower bound for $S$.  Since $\widetilde{M}\ge4$,
\[
 \lfloor(\widetilde{M}-1)/3\rfloor
 \ge\frac{\widetilde{M}}{6}.
\]
For $\widetilde{M}=4,5,6,7$, this inequality is checked directly; for
$\widetilde{M}\ge8$, it follows from
$\lfloor(\widetilde{M}-1)/3\rfloor\ge(\widetilde{M}-4)/3
\ge\widetilde{M}/6$.  Moreover,
$D\ge P/(2B_d)\ge\widetilde{M}/(2B_d)$.  Since $N_d\ge1$,
\eqref{eq:Bd} yields
$B_d\ge(48+2d)N_d(N_d+1)\ge100>3$.  Thus the two arguments in the minimum
defining $S$ satisfy, respectively,
\[
 \lfloor(\widetilde{M}-1)/3\rfloor
 \ge\frac{\widetilde{M}}{6}
 \ge\frac{\widetilde{M}}{2B_d},
\]
and
\[
 D\ge\frac{\widetilde{M}}{2B_d}.
\]
Therefore
\begin{equation}\label{eq:S-lower}
 S\ge\frac{\widetilde{M}}{2B_d}.
\end{equation}
Combining \eqref{eq:D-lower} and \eqref{eq:S-lower} gives
\begin{equation*}
 SD\ge\frac{P\widetilde{M}}{4B_d^2}.
\end{equation*}
Moreover, since $S\ge1$, \eqref{eq:D-lower} gives
$SD\ge D\ge2^{2d+1}$.  Therefore
\[
 x:=(SD/2)^{1/(2d)}\ge2,
\]
so $q=\lfloor x\rfloor\ge2$ and, by the same floor estimate,
$q\ge x/2$.  With $K=q^2$,
\begin{align}
 \frac{\sqrt{d}}{K}
 &\le4\sqrt{d}\bigl(\tfrac{2}{SD}\bigr)^{1/d}\notag\\
 &\le4\cdot 8^{1/d}\sqrt{d}B_d^{2/d}
      (P\widetilde{M})^{-1/d}\notag\\
 &=C_d(P\widetilde{M})^{-1/d}.
 \label{eq:mesh-budget}
\end{align}

We may now apply Proposition~\ref{prop:integer-upper}, which supplies
$\Phi\in\Arch(N_d,\widehat{L}_{d,D})$ with
\[
 P_\Phi\le B_dD\le P
\]
and error at most $(d+2)\omega_f\bigl(\tfrac{\sqrt d}{K}\bigr)$.  Since
$D=\lfloor P/B_d\rfloor$, the two depth bounds coincide:
\[
 \widehat{L}_{d,D}=L_{d,P}.
\]

To obtain an affine map with domain exactly $\R^M$, define the fixed
coordinate projection
\[
 \pi_{S,M}:\R^M\to\R^{3S+1},\qquad
 \pi_{S,M}(\bmxi):=(\xi_1,\ldots,\xi_{3S+1}),
\]
and set
\[
 \calA:=\calA_0\circ\pi_{S,M}
 \in\aff(M,P_\Phi).
\]
If the latent vector supplied by Proposition~\ref{prop:integer-upper} is
$\bmxi_f^{\mathrm{core}}\in\R^{3S+1}$, take
$\bmxi_f=(\bmxi_f^{\mathrm{core}},\bmzero)\in\R^M$, where the appended
zero vector has length $M-(3S+1)$ and may therefore have length zero.  Then
$\pi_{S,M}(\bmxi_f)=\bmxi_f^{\mathrm{core}}$.  Set
$\bmtheta_f:=\calA(\bmxi_f)$.  Finally, Proposition~\ref{prop:integer-upper},
monotonicity of $\omega_f$, and \eqref{eq:mesh-budget} give
\[
 \lVert f-\Phi_{\bmtheta_f}\rVert_{L^\infty([0,1]^d)}
 \le(d+2)\omega_f\bigl(\tfrac{\sqrt d}{K}\bigr)
 \le(d+2)\omega_f\bigl(C_d(P\widetilde{M})^{-1/d}\bigr).
\]
Since $\widetilde{M}=\min\{M,P\}$, this is
\eqref{eq:main-upper}.  The affine assertion follows from
Proposition~\ref{prop:integer-upper} and the fixed projection.
\end{proof}

\section{Proof of the H\"older lower bound}\label{sec:lower-proof}

The lower-bound proof combines a capacity estimate with a packing construction.
For every admissible pair $(\Phi,\calA)$ with input dimension $d$ and
$P_\Phi\le P$, recall
$U_\Phi$ from \eqref{eq:hidden-unit-count} and put
\[
 \widetilde{c}_\alpha:=4^{-(\alpha+1)}2^{-\alpha}.
\]
The first step removes all redundant affine latent coordinate directions:
\[
 r_{\calA}=\rank(\bmA)
 \le\min\{M,P_\Phi\}.
\]
Counting binary polynomial patterns then gives the capacity estimate
\[
 \Pdim(\calF_{\Phi,\calA})
 \le8r_{\calA}(U_\Phi+1)
 \le8P\min\{M,P\}.
\]
Finally, binary choices of positive and negative H\"older bumps convert this
capacity upper bound into approximation error.  With
$V_0:=8P\min\{M,P\}\ge1$, the packing lemma gives
\[
 \sup_{f\in\calH_d^\alpha}\inf_{g\in\calF_{\Phi,\calA}}
 \lVert f-g\rVert_{L^\infty([0,1]^d)}
 \ge \widetilde{c}_\alpha V_0^{-\alpha/d}.
\]
The remainder of the section proves these steps in order.
Section~\ref{sec:capacity-preliminaries} develops the required VC-dimension,
pseudo-dimension, and polynomial-pattern tools.
Section~\ref{sec:effective-coordinates} removes redundant latent directions,
and Section~\ref{sec:affine-pdim} applies the resulting coordinates to the
network family.  Section~\ref{sec:bump-packing} converts the capacity estimate
into a uniform approximation lower bound, and
Section~\ref{sec:lower-completion} completes the minimax argument.  Throughout
the capacity argument, $r$ is the number of effective parameter
variables, $s$ bounds the number of Boolean-atom occurrences,
$d_{\mathrm{pol}}$ bounds their polynomial degree, and $m$ is the number of
lifted input and threshold pairs under consideration.  The symbol $p_j$ is
reserved for an atom polynomial, while $p_{i,j}$ denotes that same atom after
the $i$th lifted instance has been fixed.

\subsection{VC dimension, pseudo-dimension, and polynomial predicates}
\label{sec:capacity-preliminaries}
To prove the capacity estimate announced above, we first isolate the general
combinatorial tools from the network-specific argument.  We begin with VC
dimension and pseudo-dimension, then develop the Boolean-formula and polynomial
sign-pattern bounds used later to analyze $\calF_{\Phi,\calA}$.

\subsubsection*{VC dimension and pseudo-dimension}

The capacity argument concerns binary labelings, although the original
function class is real-valued.  The required passage has two steps.  First,
each real-valued function is converted into a binary function on a lifted
input $(x,t)$.  Second, ordinary VC dimension is applied to the resulting
binary class.  Pseudo-dimension is exactly the quantity obtained in this way.

Let $X$ be a nonempty set, and let
$\calC\subseteq\{0,1\}^X$ be a binary function class.  For a finite set
$T\subseteq X$, its trace is
\[
 \calC|_T:=\{c|_T:c\in\calC\}.
\]
Here $c|_T$ is the restriction of $c$ to $T$, while $\{0,1\}^T$ denotes
the set of all functions from $T$ to $\{0,1\}$.  In particular, if
$\lvert T\rvert=m$, then $\{0,1\}^T$ contains exactly $2^m$ labelings.
The set $T$ is shattered by $\calC$ when every binary labeling of $T$ is
realized, that is,
\[
 \calC|_T=\{0,1\}^T.
\]
Throughout this section, we use $\bmv$ for binary label vectors.  Thus,
equivalently, if $T=\{x_1,\ldots,x_m\}$ is enumerated without repetition,
then for every
label vector
$\bmv=(v_1,\ldots,v_m)\in\{0,1\}^m$ there exists
$c_{\bmv}\in\calC$ such that
\[
 c_{\bmv}(x_i)=v_i,
 \qquad i=1,\ldots,m.
\]
Let $\operatorname{Shat}(\calC)$ denote the family of finite subsets of $X$
that are shattered by $\calC$.  The VC dimension is therefore the following
quantity:
\[
 \VCdim(\calC)
 :=\sup\left(\{0\}\cup
 \{\lvert T\rvert:T\in\operatorname{Shat}(\calC)\}\right).
\]
Thus $\VCdim(\calC)$ is the largest number of points that
can be shattered, or $+\infty$ if arbitrarily large finite sets can be
shattered.  The extra set $\{0\}$ makes the value equal to zero when no
nonempty set is shattered.

Now let $\calF\subseteq\R^X$ be a real-valued class.  Its subgraph class is
the binary class
\begin{equation*}
 \operatorname{Sub}(\calF)
 :=\left\{
 (x,t)\mapsto\one_{\{g(x)>t\}}:
 g\in\calF
 \right\}
\end{equation*}
on $X\times\R$; in symbols,
\[
 \operatorname{Sub}(\calF)\subseteq\{0,1\}^{X\times\R}.
\]
Thus the class to which VC dimension is applied is
$\operatorname{Sub}(\calF)$, not the original real-valued class $\calF$.
Geometrically, for a fixed $g$, the lifted point $(x,t)$ receives label one
exactly when its height $t$ lies strictly below the graph height $g(x)$.
Define
\begin{equation}\label{eq:pdim-definition}
 \Pdim(\calF):=\VCdim\left(\operatorname{Sub}(\calF)\right).
\end{equation}
More explicitly, for $m\in\N^+$, the inequality $m\le\Pdim(\calF)$ holds
precisely when there are $m$ distinct lifted points
\[
 (x_1,t_1),\ldots,(x_m,t_m)\in X\times\R
\]
such that, for every label vector
$\bmv=(v_1,\ldots,v_m)\in\{0,1\}^m$,
there is a function $g_{\bmv}\in\calF$ satisfying
\begin{equation}\label{eq:pdim-labeling}
 \one_{\{g_{\bmv}(x_i)>t_i\}}
 =v_i,
 \qquad i=1,\ldots,m.
\end{equation}
The base inputs $x_i$ must also be pairwise distinct.  Otherwise, if
$x_i=x_j$ and $t_i<t_j$, the requested label pair $(v_i,v_j)=(0,1)$ would
require both $g(x_i)\le t_i$ and $g(x_i)>t_j$, which is impossible.  The case
$t_j<t_i$ is symmetric, while equality would contradict the distinctness of
the lifted points.
The quantifier order is essential.  One first fixes the pairwise distinct
lifted instances $(x_i,t_i)$; then, for every prescribed label vector
$\bmv\in\{0,1\}^m$, one may choose a function $g_{\bmv}\in\calF$ satisfying
\eqref{eq:pdim-labeling}.  In particular, the thresholds $t_i$ cannot
depend on the desired label vector.

For comparison, the zero-threshold class is
\[
 \operatorname{Thr}_0(\calF)
 :=\left\{x\mapsto\one_{\{g(x)>0\}}:g\in\calF\right\}
\]
and is obtained by restricting the subgraph class to the slice
$X\times\{0\}$.  For a real-valued class, we use the harmless overload
\begin{equation}\label{eq:real-vcdim}
 \VCdim(\calF)
 :=\VCdim\left(\operatorname{Thr}_0(\calF)\right),
\end{equation}
where the right-hand side is the ordinary VC dimension of the displayed
binary class.  Consequently,
\begin{equation}\label{eq:vc-below-pdim}
 \VCdim(\calF)
 \le\Pdim(\calF).
\end{equation}
Indeed, if $\operatorname{Thr}_0(\calF)$ shatters
$x_1,\ldots,x_m$, then the lifted points
$(x_1,0),\ldots,(x_m,0)$ are shattered by
$\operatorname{Sub}(\calF)$.  Restricting the allowed lifted points therefore
cannot increase VC dimension.

The next definition combines two ingredients.  A \emph{polynomial atom} asks
one binary question of the form $p>0$ and produces one bit; a fixed Boolean
rule combines finitely many such bits.  We formalize the resulting description
in two stages.  First, a fixed formula tree is identified with its Boolean
map.  Definition~\ref{def:boolean-polynomial} then combines that map with
polynomial atoms whose parameter degree is controlled after $(\bmx,t)$ is
fixed.  The Boolean formula itself need not be a polynomial.

\subsubsection*{Boolean formulas and polynomial predicates}
Let $s_0\in\N^+$.  A Boolean formula tree
$\mathfrak{B}(Z_1,\ldots,Z_{s_0})$ is a finite rooted ordered syntax tree with
exactly $s_0$ leaf occurrences.  Choose once and for all a
bijection between those occurrences and $\{1,\ldots,s_0\}$, and label the
occurrence indexed by $j$ with $Z_j$.  Each internal node is one of the fixed operations
``not'' $(\neg)$, ``and'' $(\wedge)$, or ``or'' $(\vee)$.  If the same
polynomial question is used twice in the written formula, its two uses are
assigned to two different leaf occurrences and are counted twice.  Thus
$s_0$ counts positions in the syntax tree, not algebraically distinct
polynomials.
Whenever a finite $\bigwedge$ or $\bigvee$ is used below, we fix an ordering
of its index set and a binary bracketing.  This expands the displayed notation
into a formula tree of the preceding kind without changing its truth value or
its number of leaf occurrences.

We now define the map computed by this tree, rather than leaving the phrase
``evaluate the formula'' implicit.  Let
$\bmu=(u_1,\ldots,u_{s_0})\in\{0,1\}^{s_0}$.  For every subtree $\psi$, define
its value $\mathsf{B}_\psi(\bmu)$ recursively by
\[
 \mathsf{B}_{Z_j}(\bmu):=u_j,
 \qquad
 \mathsf{B}_{\neg\psi}(\bmu):=1-\mathsf{B}_\psi(\bmu).
\]
For two subtrees $\psi_1$ and $\psi_2$, define
\[
 \mathsf{B}_{\psi_1\wedge\psi_2}(\bmu)
 :=\mathsf{B}_{\psi_1}(\bmu)\mathsf{B}_{\psi_2}(\bmu).
\]
Also define
\[
 \mathsf{B}_{\psi_1\vee\psi_2}(\bmu)
 :=\max\{\mathsf{B}_{\psi_1}(\bmu),\mathsf{B}_{\psi_2}(\bmu)\}.
\]
Because all subtree values are bits, these equations are exactly the usual
truth tables.  Applying the recursion at the root defines one map
\[
 \mathsf{B}:=\mathsf{B}_{\mathfrak{B}}:\{0,1\}^{s_0}\to\{0,1\}.
\]

For example, consider the fixed four-leaf formula
\begin{equation}\label{eq:boolean-example}
 \mathfrak{B}_{\mathrm{ex}}
 :=(Z_1\wedge Z_2)\vee(Z_3\wedge Z_4).
\end{equation}
Its associated map is
\[
 \mathsf{B}_{\mathrm{ex}}(u_1,u_2,u_3,u_4)
 :=(u_1\wedge u_2)\vee(u_3\wedge u_4).
\]
At $\bmu=(1,0,1,1)$, the recursion gives
\[
 \mathsf{B}_{\mathrm{ex}}(1,0,1,1)
 =(1\wedge0)\vee(1\wedge1)=1.
\]
Figure~\ref{fig:boolean-tree} shows the same calculation.  Only the four leaf
values vary.  The connectives, their wiring, and the assignment of the
$j$th input bit to the leaf $Z_j$ remain fixed.

\begin{figure}[ht]
\centering
\includegraphics[width=0.68\textwidth]{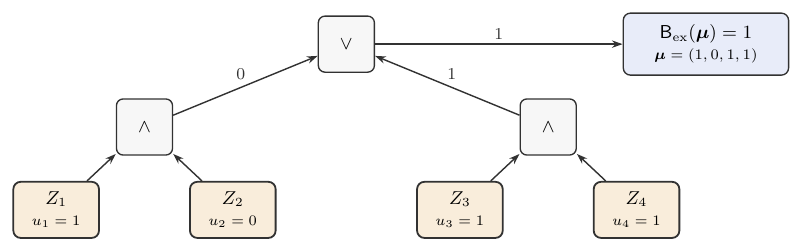}
\caption{Evaluation of the fixed formula in
\eqref{eq:boolean-example}.  The orange boxes supply the changing
input bits.  The two $\wedge$ gates, the $\vee$ gate, and every edge of the
tree are fixed before the bits are known.}
\label{fig:boolean-tree}
\end{figure}

Thus, saying that the Boolean map is fixed means that the formula tree, its
connectives, wiring, and leaf labels are all fixed in advance.  Once these
choices have been made, the recursion above determines $\mathsf{B}$ uniquely,
and only the $s_0$ leaf bits may vary with $(\bmx,t,\bmmu)$.

\begin{definition}[Boolean-polynomial description]
\label{def:boolean-polynomial}
Let $d,r,s,d_{\mathrm{pol}}\in\N^+$, let $X\subseteq\R^{d}$, and let
\[
 \calF:=\{g_{\bmmu}:X\to\R:\bmmu\in\R^r\}.
\]
Its joint truth set is
\begin{equation}\label{eq:truth-set}
 \calS
 :=\left\{(\bmx,t,\bmmu)\in X\times\R\times\R^r:
 g_{\bmmu}(\bmx)>t\right\}.
\end{equation}

Choose a number of leaf occurrences $s_0$ with $1\le s_0\le s$.
For each occurrence $j=1,\ldots,s_0$, choose a polynomial
\[
 p_j\in\R[x_1,\ldots,x_{d},t,\mu_1,\ldots,\mu_r]
\].
We use the standard convention that the zero polynomial has degree
$-\infty$; thus identically zero atoms are covered by every nonnegative
degree bound below.
The $j$th atom produces the bit
\begin{equation}\label{eq:atom-truth-value}
 A_j(\bmx,t,\bmmu)
 :=\one_{\{p_j(\bmx,t,\bmmu)>0\}}.
\end{equation}
Finally, choose one Boolean formula tree
$\mathfrak{B}(Z_1,\ldots,Z_{s_0})$ as described above, and let
$\mathsf{B}:\{0,1\}^{s_0}\to\{0,1\}$ be the map obtained from its recursive
evaluation.  We say that $\calS$
has an $(s,d_{\mathrm{pol}})$ Boolean-polynomial description in the parameter
$\bmmu$ when, for every $(\bmx,t,\bmmu)$,
\begin{equation}\label{eq:boolean-description}
 \one_{\calS}(\bmx,t,\bmmu)
 =\mathsf{B}\bigl(
 A_1(\bmx,t,\bmmu),\ldots,
 A_{s_0}(\bmx,t,\bmmu)\bigr),
\end{equation}
and, for each fixed $(\bmx,t)$,
\begin{equation*}
 \deg_{\bmmu}p_j(\bmx,t,\cdot)\le d_{\mathrm{pol}},
 \qquad j=1,\ldots,s_0.
\end{equation*}
The integer $s$ is only an upper bound; $s_0$ is the number of leaf
occurrences actually used after all comparisons have been written in the
strict-positivity form of \eqref{eq:atom-truth-value}.
\end{definition}

Every object defining the rule ($s_0$, the polynomials $p_j$, the
formula tree $\mathfrak{B}$, its map $\mathsf{B}$, and the assignment of atoms to
leaves) is chosen before $(\bmx,t,\bmmu)$ is given.  Only the atom bits in
\eqref{eq:atom-truth-value}, and hence the final output bit, vary
with the arguments.  If the same polynomial comparison is written twice,
its two appearances occupy two different leaf slots and count as two atom
occurrences.  They may produce identical bits, but the formula still has two
syntactic leaves.

Restricting atoms to strict positivity loses no Boolean expressive power.
For any real polynomial $p$, the other standard comparisons are represented
by
\[
 \one_{\{p<0\}}=\one_{\{-p>0\}},
 \qquad
 \one_{\{p\le0\}}=1-\one_{\{p>0\}}.
\]
Similarly,
\[
 \one_{\{p\ge0\}}=1-\one_{\{-p>0\}},
 \qquad
 \one_{\{p=0\}}
 =\bigl(1-\one_{\{p>0\}}\bigr)
  \bigl(1-\one_{\{-p>0\}}\bigr).
\]
These are identities for every real value of $p$, including $p=0$.
Accordingly, an equality may create two strict-positivity leaf occurrences;
the atom count is always taken after this normalization.

For example, the predicate
\[
 \bigl[\neg(-p>0)\wedge(q>0)\bigr]\vee(-p>0)
\]
has three leaf occurrences after normalization, even though only the two
polynomials $-p$ and $q$ occur.  The strict-positivity question for $-p$ is
used twice and is therefore counted twice.  This distinction is why the
definition counts syntax-tree occurrences rather than distinct polynomial
expressions.

Here is a concrete instance that also previews the network application.
Let
\[
 g_{\bmmu}(x)
 :=a(\bmmu)\varrho\bigl(w(\bmmu)x+b(\bmmu)\bigr)+c(\bmmu),
 \qquad
 z(x,\bmmu):=w(\bmmu)x+b(\bmmu),
\]
where $a,w,b,c$ are affine functions of $\bmmu$.  The inequality
$g_{\bmmu}(x)>t$ is equivalent to
\begin{equation}\label{eq:one-unit-boolean-preview}
 \bigl[\neg(-z>0)\wedge(az+c-t>0)\bigr]
 \vee
 \bigl[(-z>0)\wedge(c-t>0)\bigr],
\end{equation}
where the arguments have been suppressed inside the formula.  Thus its four
leaf occurrences may be substituted for $Z_1,\ldots,Z_4$: the question
$-z>0$ occupies both $Z_1$ and $Z_3$, while the two output comparisons
occupy $Z_2$ and $Z_4$.  The Boolean map is
\[
 \mathsf{B}(u_1,u_2,u_3,u_4)
 :=\bigl[(\neg u_1)\wedge u_2\bigr]\vee(u_3\wedge u_4).
\]
It is fixed for every $(x,t,\bmmu)$; only the four supplied bits change.
Because $\neg(-z>0)$ is exactly $z\ge0$, the formula assigns the boundary
$z=0$ to the active branch, where $\varrho(0)=0$.  For fixed $(x,t)$, each of
the four polynomial-atom occurrences in
\eqref{eq:one-unit-boolean-preview} has degree at most two in
$\bmmu$, so this
is a $(4,2)$ Boolean-polynomial description.  The network proof below will
use exactly this active-branch/inactive-branch idea for every hidden unit.

\subsubsection*{Polynomial sign-pattern bounds}

We now state the only external algebraic result used in the capacity part of
the lower-bound proof.  We use a convenient weakened form of Warren's standard
strict-sign theorem \cite[Theorem~3]{warren1968signpatterns}; see also
\cite[Theorem~2.1]{goldberg1995bounding}.  The classical estimate
$2(2ed_{\mathrm{pol}}n/r)^r$ implies the slightly looser bound below because
$r\ge1$.

\begin{theorem}[Warren's strict-sign theorem]\label{thm:warren}
Let $r,n,d_{\mathrm{pol}}\in\N^+$ with $n\ge r$, and let
$q_1,\ldots,q_n$ be $n$ distinct nonzero real polynomials in $r$ variables,
each of total degree at most $d_{\mathrm{pol}}$.  A vector
$\bmsigma=(\sigma_1,\ldots,\sigma_n)\in\{-1,1\}^n$ is a consistent nonzero
sign assignment if there exists $\bmmu\in\R^r$ such that
$\sigma_j q_j(\bmmu)>0$ for every $j$.  The number of such assignments
satisfies
\begin{equation}\label{eq:warren}
 \Bigl\lvert
 \bigl\{
  \bmsigma\in\{-1,1\}^n:
  \text{there exists }\bmmu\in\R^r
  \text{ such that }\sigma_j q_j(\bmmu)>0
  \text{ for every }j
 \bigr\}
 \Bigr\rvert
 \le
 \left(\frac{4ed_{\mathrm{pol}}n}{r}\right)^r.
\end{equation}
\end{theorem}

The strict inequalities in \eqref{eq:warren} exclude all points at
which one of the polynomial values vanishes.  The form needed later is the
following indexed consequence, which also permits repeated polynomials.

\begin{corollary}[Indexed Warren bound for strict-positivity patterns]
\label{cor:indexed-warren}
Let $r,n,d_{\mathrm{pol}}\in\N^+$ with $n\ge r$, and let
$q_1,\ldots,q_n$ be nonzero real polynomials in $r$ variables, each of total
degree at most $d_{\mathrm{pol}}$.  The number of binary patterns realized at
points where none of the evaluated polynomials vanishes satisfies
\begin{equation}\label{eq:indexed-warren}
 \Bigl\lvert
 \bigl\{
  \left(
   \one_{\{q_1(\bmmu)>0\}},\ldots,
   \one_{\{q_n(\bmmu)>0\}}
  \right):
  \bmmu\in\R^r,\
  q_j(\bmmu)\ne0\text{ for every }j
 \bigr\}
 \Bigr\rvert
 \le
 \left(\frac{4ed_{\mathrm{pol}}n}{r}\right)^r.
\end{equation}
\end{corollary}

\begin{proof}
At the parameter points counted in \eqref{eq:indexed-warren}, every
polynomial value is either positive or negative.  The indicator records the
positive case by $1$ and the negative case by $0$.  Thus these are precisely
the positive/negative sign patterns counted by Warren's theorem.  A polynomial
may vanish at other points, and the indexed list
need not consist of algebraically distinct polynomials.

The displayed statement is formulated for an indexed family and therefore
allows algebraic repetitions.  To recover it from a version stated for a set
of distinct polynomials, retain one representative of each distinct
polynomial and let $N\le n$ be the number retained.  The full indexed vector
is determined by the resulting $N$-vector.  If $N\ge r$,
Theorem~\ref{thm:warren} for the distinct family gives
\[
 \left(\frac{4ed_{\mathrm{pol}}N}{r}\right)^r
 \le
 \left(\frac{4ed_{\mathrm{pol}}n}{r}\right)^r.
\]
If $N<r$, the trivial bound gives
\[
 2^N\le2^r
 \le\left(\frac{4ed_{\mathrm{pol}}n}{r}\right)^r,
\]
because $d_{\mathrm{pol}}\ge1$ and $n\ge r$.  Thus the indexed version
follows from Theorem~\ref{thm:warren}; the argument also covers repeated
polynomials and the endpoint $n=r$.
\end{proof}

Closely related binary-threshold formulations appear in
\cite[Lemma~1]{Bartlett98almostlinear},
\cite[Theorem~8.3]{anthony_bartlett_1999}, and
\cite[Lemma~17]{JMLR:v20:17-612}.  The version needed here permits zero values
and groups them with the nonpositive outcome.

For geometric intuition, each equation $p_j(\bmmu)=0$ describes an
algebraic zero set in parameter space.  On every connected component of the
complement of their union, the binary positivity vector is constant, but
different components may carry the same vector.
\eqref{eq:indexed-warren} bounds the number of distinct realized
positivity vectors, not the number of connected regions.  This geometric
picture is not used in the proof.  The argument below uses
Corollary~\ref{cor:indexed-warren}, which follows from
Theorem~\ref{thm:warren}, and does not rely on a count of connected regions.

The following lemma establishes this version by choosing a single perturbation
size that works simultaneously for every realized binary pattern.

\begin{lemma}[Binary polynomial positivity patterns]
\label{lem:binary-patterns}
Let $r,n,d_{\mathrm{pol}}\in\N^+$ with $n\ge r$.  Let
$p_1,\ldots,p_{n}$ be arbitrary real polynomials in $r$ variables,
each of total degree at most $d_{\mathrm{pol}}$.  Then
\begin{equation}\label{eq:binary-pattern-bound}
 \Bigl\lvert
 \bigl\{
 \left(
  \one_{\{p_1(\bmmu)>0\}},\ldots,
  \one_{\{p_n(\bmmu)>0\}}
 \right):
 \bmmu\in\R^r
 \bigr\}
 \Bigr\rvert
 \le
 \left(\frac{4ed_{\mathrm{pol}}n}{r}\right)^r.
\end{equation}
The list may contain repeated polynomials or the zero polynomial.
\end{lemma}

\begin{proof}
The idea is to move every polynomial downward by the same small number.
At one carefully selected witness for each realized bit vector, positive
values remain positive, while zero and negative values become strictly
negative.  Corollary~\ref{cor:indexed-warren}, and hence
Theorem~\ref{thm:warren}, can then count these strict binary patterns.

\proofstep{1}{Choose one witness for each realized binary pattern.}

Let $\calB$ denote the set on the left side of
\eqref{eq:binary-pattern-bound}.  It is finite because
$\calB\subseteq\{0,1\}^{n}$.  For every
$\bmv=(v_1,\ldots,v_n)\in\calB$, choose one representative
$\bmmu^{\bmv}\in\R^r$ such that
\[
 \one_{\{p_j(\bmmu^{\bmv})>0\}}=v_j,
 \qquad j=1,\ldots,n.
\]

\proofstep{2}{Choose one perturbation size that works for every witness.}

Consider the finite collection of positive witness values
\[
 \calV_+
 :=\left\{p_j(\bmmu^{\bmv}):
 \bmv\in\calB,\ 1\le j\le n,\
 p_j(\bmmu^{\bmv})>0\right\}.
\]
If $\calV_+$ is empty, every realized binary pattern is the all-zero vector.
Hence $\lvert\calB\rvert=1$, which is no larger than the right side of
\eqref{eq:binary-pattern-bound}; the proof is complete in this
case.

Assume that $\calV_+$ is nonempty.  Being a finite set of positive real
numbers, it has a positive minimum.  Set
\[
 a_*:=\min\calV_+>0,
 \qquad
 \delta:=\frac{a_*}{2}.
\]

\proofstep{3}{Shift every polynomial downward and preserve all witness bits.}

For $j=1,\ldots,n$, define
\[
 q_j:=p_j-\delta.
\]
Each $q_j$ has degree at most $d_{\mathrm{pol}}$.  We first check that it is
not the zero polynomial.  If $q_j$ were identically zero, then
$p_j\equiv\delta>0$.  The value $\delta$ would then belong to $\calV_+$, so
$a_*\le\delta=a_*/2$, a contradiction.

Fix $\bmv\in\calB$ and $j\in\{1,\ldots,n\}$.  If $v_j=1$, then
$p_j(\bmmu^{\bmv})\in\calV_+$ and hence
\[
 p_j(\bmmu^{\bmv})\ge a_*=2\delta.
\]
Therefore $q_j(\bmmu^{\bmv})\ge\delta>0$.  If $v_j=0$, then
$p_j(\bmmu^{\bmv})\le0$, so
$q_j(\bmmu^{\bmv})\le-\delta<0$.  In both cases,
\[
 \one_{\{q_j(\bmmu^{\bmv})>0\}}=v_j,
 \qquad
 q_j(\bmmu^{\bmv})\ne0.
\]
Thus the original bit vector $\bmv$ is reproduced exactly by the shifted
polynomials at its selected witness.

\proofstep{4}{Apply the indexed consequence of Warren's theorem.}

The preceding identity shows that distinct vectors in $\calB$ give distinct
strict-positivity patterns of $q_1,\ldots,q_n$ at parameter points where all
$q_j$ are nonzero.  Consequently, $\calB$ injects into the set counted in
\eqref{eq:indexed-warren}.  The hypotheses of
Corollary~\ref{cor:indexed-warren} hold: $n\ge r$, every $q_j$ is nonzero by
Step~3, and its degree is at most $d_{\mathrm{pol}}$.  Therefore,
\[
 \lvert\calB\rvert
 \le
 \left(\frac{4ed_{\mathrm{pol}}n}{r}\right)^r.
\]
This is \eqref{eq:binary-pattern-bound}.  Repetitions in the
original list and identically zero original polynomials cause no problem:
the argument treats the list by indexed occurrences, and the common
downward shift makes every $q_j$ a nonzero polynomial.
\end{proof}

\subsubsection*{From polynomial predicates to pseudo-dimension}

The following proposition adapts the counting argument underlying
\cite[Theorem~2.2]{goldberg1995bounding} to subgraph classes instead of
invoking that theorem directly.  Goldberg and Jerrum formulate an
ordinary VC-dimension result, count distinct atomic predicates, and use
three-way sign assignments, leading to the constant $8e$.  Here atoms are
normalized to binary strict-positivity tests, syntactic leaf occurrences are
counted, and Lemma~\ref{lem:binary-patterns} groups zero with the nonpositive
outcome.  This gives the constant $4e$ below.  The network result in
\cite[Theorem~7 and Lemma~17]{JMLR:v20:17-612} is not invoked.

\begin{proposition}[Polynomial-predicate pseudo-dimension bound]\label{prop:GJ}
Let $\calF=\{g_{\bmmu}:X\to\R:\bmmu\in\R^r\}$, where
$X\subseteq\R^{d}$ and $d,r\in\N^+$.  If its joint truth set admits an
$(s,d_{\mathrm{pol}})$ Boolean-polynomial description in the sense of
Definition~\ref{def:boolean-polynomial}, then
\begin{equation}\label{eq:GJ-bound}
 \Pdim(\calF)
 \le2r\log_2\left(4ed_{\mathrm{pol}} s\right).
\end{equation}
\end{proposition}

\begin{proof}
Fix one $(s,d_{\mathrm{pol}})$ Boolean-polynomial description of the joint truth set.
Retain the notation $s_0$, $p_j$, $\mathfrak{B}$, and $\mathsf{B}$
from Definition~\ref{def:boolean-polynomial}.  In
particular, $s_0$ is the actual number of leaf occurrences and $s_0\le s$.

Fix $m\in\N^+$ for which the explicit labeling property in
\eqref{eq:pdim-labeling} holds.  We prove that $m$ satisfies
\eqref{eq:GJ-bound}.  The functions $g_{\bmmu}$ themselves
need not be global polynomials in $\bmmu$.  Only the binary predicate
$g_{\bmmu}(\bmx)>t$ must have the stated Boolean-polynomial description.
Once the lifted instances $(\bmx_i,t_i)$ have been fixed, every relevant
polynomial depends only on the parameter variable $\bmmu$.

When $ms_0\ge r$, the counting argument reduces to the chain
\begin{equation}\label{eq:GJ-proof-chain}
 \underbrace{2^m}_{\text{all requested labels}}
 \le
 \underbrace{\left\lvert\text{realized polynomial bit patterns}\right\rvert}_
             {\text{information available from the atoms}}
 \le
 \underbrace{\left(\frac{4ed_{\mathrm{pol}}ms}{r}\right)^r}_
             {\text{binary Warren bound}}.
\end{equation}
Steps~1 through 4 justify this chain, and Step~5 solves the remaining scalar
inequality for $m$.  The easier case $ms_0<r$ is handled at the beginning of
Step~4.

\proofstep{1}{Fix the lifted instances and take parameter-space sections.}

The purpose of this step is to turn
\eqref{eq:pdim-labeling} into a numerical requirement: varying the
parameter must produce all $2^m$ binary membership vectors.  The joint set
$\calS$ records every triple for which $g_{\bmmu}(\bmx)>t$ is true, while a
section fixes $(\bmx,t)$ and keeps only the parameters giving that answer.

By the defining property of the selected $m$, there are fixed lifted
instances
\[
 (\bmx_1,t_1),\ldots,(\bmx_m,t_m)\in X\times\R
\]
on which all $2^m$ label vectors are realized.  For an arbitrary lifted
instance $(\bmx,t)$, define the corresponding parameter section of the joint
truth set by
\[
 \calS_{\bmx,t}
 :=\{\bmmu\in\R^r:(\bmx,t,\bmmu)\in\calS\}.
\]
By \eqref{eq:truth-set}, this section is
\[
 \calS_{\bmx,t}
 =\{\bmmu\in\R^r:g_{\bmmu}(\bmx)>t\}.
\]
It is a semialgebraic subset of $\R^r$, because
\eqref{eq:boolean-description} describes it using finitely many
polynomial comparisons in $\bmmu$.  Here ``semialgebraic'' means a finite
Boolean combination of sets defined by polynomial equalities and
inequalities.  For the fixed lifted instances, write
\[
 \calS_i:=\calS_{\bmx_i,t_i},
 \qquad i=1,\ldots,m.
\]

Define the label map
\[
 \operatorname{Lab}:\R^r\to\{0,1\}^m
\]
by
\[
 \operatorname{Lab}(\bmmu)
 :=\left(\one_{\calS_1}(\bmmu),\ldots,
 \one_{\calS_m}(\bmmu)\right).
\]
Equivalently, its $i$th coordinate is
\[
 \operatorname{Lab}_i(\bmmu)
 =\one_{\{g_{\bmmu}(\bmx_i)>t_i\}}.
\]
For every $\bmv\in\{0,1\}^m$,
\eqref{eq:pdim-labeling} supplies a
parameter $\bmmu_{\bmv}\in\R^r$ for which
\[
 \operatorname{Lab}(\bmmu_{\bmv})=\bmv.
\]
Consequently, the set of realized label vectors is
\begin{equation}\label{eq:all-labels}
 \calY:=\operatorname{Lab}(\R^r)=\{0,1\}^m,
 \qquad \lvert\calY\rvert=2^m.
\end{equation}

\proofstep{2}{Collect the parameter polynomials.}

We now freeze the lifted instances.  This removes the input and threshold
from the list of variables, leaving a finite family of ordinary polynomials
in $\bmmu$ to which a binary pattern theorem can later be applied.

For each lifted instance and each atom occurrence, define
\[
 p_{i,j}(\bmmu)
 :=p_j(\bmx_i,t_i,\bmmu),
 \qquad i=1,\ldots,m,
 \quad j=1,\ldots,s_0.
\]
Every $p_{i,j}$ is a polynomial in the $r$ coordinates of $\bmmu$, and
\[
 \deg_{\bmmu}p_{i,j}\le d_{\mathrm{pol}}.
\]
The natural occurrence index set is
\[
 \calI:=\{1,\ldots,m\}\times\{1,\ldots,s_0\}.
\]
Its cardinality satisfies
\[
 \lvert\calI\rvert=ms_0\le ms.
\]
There are $s_0$ atom occurrences for each of the $m$ fixed lifted instances;
hence the indexed list contains $ms_0$ polynomials.  Two members of this
list may coincide, and some may be identically zero.  This causes no problem:
Lemma~\ref{lem:binary-patterns} was stated for an indexed list and explicitly
allows both possibilities.

\proofstep{3}{Show that polynomial bits determine every label.}

This step proves that the labels contain no more combinatorial information
than the simultaneous strict-positivity bits of the collected polynomials.

Define the simultaneous bit map
\[
 \operatorname{Bit}_{\calI}:\R^r\to\{0,1\}^{\calI}
\]
by
\[
 \operatorname{Bit}_{\calI}(\bmmu)
 :=\left(
 \one_{\{p_{i,j}(\bmmu)>0\}}
 \right)_{(i,j)\in\calI}.
\]
Let
\[
 \calB:=\operatorname{Bit}_{\calI}(\R^r)
\]
be the set of realizable bit patterns.  It contains only patterns attained
by some parameter and need not equal the full cube $\{0,1\}^{\calI}$.
A fiber of $\operatorname{Bit}_{\calI}$ may be disconnected; the proof counts
bit patterns, not connected components.

We now define, rather than merely assert, the map from atom bits to labels.
For
$\bmchi=(\chi_{i,j})_{(i,j)\in\calI}\in\{0,1\}^{\calI}$, set
\[
 \mathsf{H}_i(\bmchi)
 :=\mathsf{B}\left(
 \chi_{i,1},\ldots,\chi_{i,s_0}
 \right),
 \qquad i=1,\ldots,m,
\]
and then set
\[
 \mathsf{H}(\bmchi)
 :=\left(\mathsf{H}_1(\bmchi),\ldots,\mathsf{H}_m(\bmchi)\right).
\]
For every parameter $\bmmu$,
\eqref{eq:boolean-description} gives the exact factorization
\[
 \operatorname{Lab}
 =\mathsf{H}\circ\operatorname{Bit}_{\calI}.
\]
In words, the information flows as
\[
 \bmmu
 \longmapsto
 \left(\one_{\{p_{i,j}(\bmmu)>0\}}\right)_{(i,j)\in\calI}
 \longmapsto
 \operatorname{Lab}(\bmmu).
\]
Thus each realized bit pattern determines exactly one label vector, although
different bit patterns may determine the same label vector.  Therefore
\begin{equation}\label{eq:labels-below-bits}
 \lvert\calY\rvert\le\lvert\calB\rvert.
\end{equation}
Combining \eqref{eq:all-labels} and
\eqref{eq:labels-below-bits} gives
\begin{equation}\label{eq:two-m-below-bits}
 2^m\le\lvert\calB\rvert.
\end{equation}

\proofstep{4}{Apply the binary polynomial-pattern estimate.}

We first remove the only small-count case.  If $ms_0<r$, then $s_0\ge1$
implies $m<r$.  Since $d_{\mathrm{pol}},s\ge1$, we have
\[
 m<r<2r\log_2(4ed_{\mathrm{pol}} s),
\]
so the desired estimate already holds.

Assume from now on that $ms_0\ge r$.  Apply
Lemma~\ref{lem:binary-patterns} directly to the indexed list
$(p_{i,j})_{(i,j)\in\calI}$.  The list has $ms_0$ members, each of degree at
most $d_{\mathrm{pol}}$, and the lemma allows repeated polynomials as well as
polynomials that are identically zero.  We obtain
\begin{equation}\label{eq:bit-pattern-count}
 \lvert\calB\rvert
 \le\left(\frac{4ed_{\mathrm{pol}}ms_0}{r}\right)^r
 \le\left(\frac{4ed_{\mathrm{pol}}ms}{r}\right)^r.
\end{equation}
The first inequality is precisely where polynomial geometry enters.
Lemma~\ref{lem:binary-patterns} preserves Warren's constant $4$ because a
single downward shift groups zero and negative values into the same binary
outcome.  \eqref{eq:two-m-below-bits} and
\eqref{eq:bit-pattern-count} imply
\begin{equation}\label{eq:GJ-counting-inequality}
 2^m
 \le\left(\frac{4ed_{\mathrm{pol}}ms}{r}\right)^r.
\end{equation}
Thus \eqref{eq:GJ-proof-chain} has been justified.  The sole
external algebraic input was Theorem~\ref{thm:warren}; the all-points binary
lemma used here was derived from it above.

\proofstep{5}{Solve the resulting scalar inequality.}

The unknown $m$ still appears on the right side of the counting inequality.
We now remove it by an elementary one-variable estimate.

Set
\[
 u:=\frac{m}{r}.
\]
Also put
\[
 a:=4ed_{\mathrm{pol}} s.
\]
Taking base-two logarithms on both sides of
\eqref{eq:GJ-counting-inequality} and dividing by $r$ gives
\[
 m\le r\log_2\left(\frac{4ed_{\mathrm{pol}}ms}{r}\right).
\]
Substituting $m=ru$ and $a=4ed_{\mathrm{pol}} s$ turns this into
\begin{equation}\label{eq:GJ-algebra}
 u\le\log_2a+\log_2u.
\end{equation}
Notice that $a\ge4e>4$.

If $u\le a$, then $\log_2u\le\log_2a$, and
\eqref{eq:GJ-algebra} implies
\[
 u\le2\log_2a.
\]
If $u>a$, then $\log_2a<\log_2u$, so
\eqref{eq:GJ-algebra} would imply
\[
 u<2\log_2u.
\]
However, $u>a>4$ and the elementary inequality
$2\log_2u\le u$ holds for every $u\ge4$.  For completeness, define
\[
 \psi(u):=u-2\log_2u.
\]
Then $\psi(4)=0$, and for $u\ge4$,
\[
 \psi'(u)=1-\frac{2}{u\ln 2}
 \ge1-\frac{1}{2\ln 2}>0.
\]
Hence $\psi(u)\ge0$ on $[4,\infty)$, proving the claimed elementary
inequality and contradicting $u<2\log_2u$.  Thus in all cases
\[
 m\le2r\log_2\left(4ed_{\mathrm{pol}} s\right).
\]

Together with the small-count case handled at the start of Step~4, this
bounds every integer $m$ satisfying the explicit labeling property in
\eqref{eq:pdim-labeling}.  Taking the supremum in
\eqref{eq:pdim-definition} proves
\eqref{eq:GJ-bound}, without assuming in advance that the
pseudo-dimension is finite.
\end{proof}

\subsection{Effective latent coordinates}
\label{sec:effective-coordinates}

Proposition~\ref{prop:GJ} measures complexity in terms of the variables that
genuinely affect the predicate.  In the present model, using the nominal
latent dimension $M$ would fail to capture the saturation at $M=P$, since
directions in the kernel of $\bmA$ do not affect the generated parameters.
Before analyzing the network computation, we therefore parameterize
$\calA(\R^M)=\operatorname{Range}(\bmA)+\bma$ using
$r_{\calA}=\rank(\bmA)$ effective coordinates.

\begin{lemma}[Rank reduction for an affine latent parameterization]
\label{lem:rank-reduction}
Let $\Phi$ be an architecture and let
$\calA\in\aff(M,P_\Phi)$ be given by
\[
 \calA(\bmxi)=\bmA\bmxi+\bma.
\]
Put $r_{\calA}:=\rank(\bmA)$.  If $r_{\calA}\ge1$, there is a
full-column-rank matrix
$\bmR\in\R^{P_\Phi\times r_{\calA}}$ such that
\begin{equation}\label{eq:affine-image-coordinates}
 \calA(\R^M)
 =\left\{\bmR\bmmu+\bma:
          \bmmu\in\R^{r_{\calA}}\right\}.
\end{equation}
Define the complete parameter vector in effective coordinates by
\[
 \bmtheta(\bmmu):=\bmR\bmmu+\bma.
\]
Consequently,
\[
 \calF_{\Phi,\calA}
 =\left\{\Phi_{\bmtheta(\bmmu)}:
          \bmmu\in\R^{r_{\calA}}\right\}.
\]
If $r_{\calA}=0$, the family contains exactly one realized function.
\end{lemma}

\begin{proof}
The purpose of the reduction is to replace the nominal latent vector by
coordinates on the range of $\bmA$.  We use only an elementary rank
factorization and verify
surjectivity explicitly, so no matrix-decomposition theorem is needed.

Assume first that $r_{\calA}\ge1$.  Choose
$r_{\calA}$ linearly independent columns of $\bmA$ that form a basis of
$\operatorname{Range}(\bmA)$, and place them in
\[
 \bmR\in\R^{P_\Phi\times r_{\calA}}.
\]
This matrix has full column rank.  Every column of $\bmA$ has a unique
coordinate vector in the chosen basis.  Placing these coordinate vectors side
by side produces a matrix
\[
 \bmC\in\R^{r_{\calA}\times M}
\]
such that
\[
 \bmA=\bmR\bmC.
\]

The matrix $\bmC$ has rank $r_{\calA}$.  Indeed,
$r_{\calA}=\rank(\bmA)\le\rank(\bmC)\le r_{\calA}$.
Consequently, $\bmC$ contains an invertible
$r_{\calA}\times r_{\calA}$ submatrix: Gaussian elimination produces
$r_{\calA}$ pivot columns because $\bmC$ has full row rank.  Given any
$\bmmu\in\R^{r_{\calA}}$, solve the corresponding square system for those
$r_{\calA}$ coordinates of $\bmxi$ and set all remaining coordinates to zero.
This gives $\bmC\bmxi=\bmmu$.  Hence the map
\[
 \bmxi\mapsto\bmC\bmxi
\]
is onto $\R^{r_{\calA}}$.

To verify the image identity, we prove the two inclusions separately.  If
$\bmxi\in\R^M$, then
$\bmA\bmxi+\bma=\bmR(\bmC\bmxi)+\bma$, so the left-hand image is contained
in the right-hand image.  Conversely, if $\bmmu\in\R^{r_{\calA}}$,
surjectivity supplies $\bmxi\in\R^M$ with $\bmC\bmxi=\bmmu$, and hence
$\bmR\bmmu+\bma=\bmA\bmxi+\bma$.  Therefore
\[
 \left\{\bmA\bmxi+\bma:\bmxi\in\R^M\right\}
 =
 \left\{\bmR\bmmu+\bma:
          \bmmu\in\R^{r_{\calA}}\right\},
\]
which is \eqref{eq:affine-image-coordinates}.  Substitution into
\eqref{eq:realized-family} gives the identity of realized function
families.

If $r_{\calA}=0$, then every column of $\bmA$ is zero, so
$\bmA=\bmzero$ and
$\calA(\bmxi)=\bma$ for every $\bmxi\in\R^M$.  Hence
$\calF_{\Phi,\calA}=\{\Phi_{\bma}\}$ is a singleton.
\end{proof}

\subsection{Pseudo-dimension under affine parameter tying}
\label{sec:affine-pdim}

The rank reduction identifies the correct parameter space.  We next bound the
binary complexity generated by the fixed ReLU architecture as the effective
coordinates vary.  For each activation pattern, the network output is
polynomial in those coordinates; we then combine all patterns into a single
Boolean formula for the subgraph predicate.  Throughout the next lemma, the
architecture $\Phi$ and the affine map $\calA$ remain fixed.  In
the positive-rank case, we also choose once and for all one basis matrix
$\bmR$ supplied by Lemma~\ref{lem:rank-reduction}.  This choice is made before
the input, threshold, effective parameter vector, or desired labeling is
specified.  Consequently, the entries of $\bmR$ and $\bma$ are fixed
coefficients; only $\bmmu\in\R^{r_{\calA}}$ varies.

\begin{lemma}[Pseudo-dimension after affine parameter tying]\label{lem:affine-pdim}
Let $\Phi$ be a fully connected $d$-input ReLU architecture with at least
one hidden unit, let $P_\Phi$ be its number of scalar weight and bias entries,
and let
$\calA\in\aff(M,P_\Phi)$.  If $U_\Phi$ is the total number of hidden
ReLU units, then
\begin{equation}\label{eq:affine-pdim}
 \VCdim(\calF_{\Phi,\calA})
 \le\Pdim(\calF_{\Phi,\calA})
 \le8r_{\calA}(U_\Phi+1)
 \le8r_{\calA}P_\Phi.
\end{equation}
In particular, if $P_\Phi\le P$, then
\begin{equation}\label{eq:affine-pdim-budget}
 \Pdim(\calF_{\Phi,\calA})
 \le8P\min\{M,P\}.
\end{equation}
\end{lemma}

\begin{proof}
The proof has five components.  We first pass from the nominal latent
vector to
$r_{\calA}$ effective variables.  For each fixed global activation pattern,
we then define a formal forward pass consisting entirely of polynomials.  A
layerwise induction controls their degree in the effective parameters after
the lifted instance $(\bmx,t)$ is fixed.  A second induction verifies an
exact, boundary-aware Boolean formula for the subgraph predicate
$\Phi_{\bmtheta}(\bmx)>t$.  Finally,
Proposition~\ref{prop:GJ} converts the number and degree of the polynomial
predicates into the claimed pseudo-dimension bound.
Figure~\ref{fig:pdim-roadmap} records this logical chain.

\begin{figure}[ht]
\centering
\includegraphics[width=0.86\textwidth]{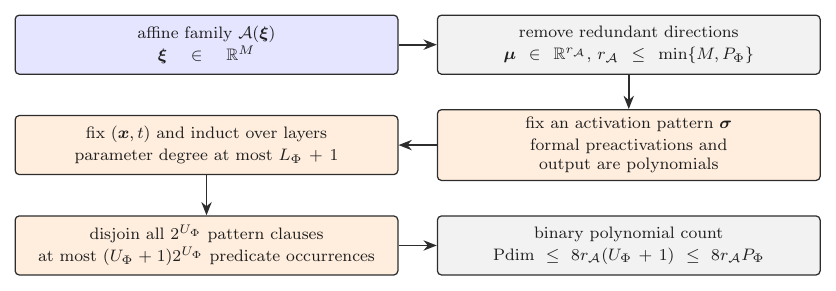}
\caption{From affine parameter tying to the pseudo-dimension bound.  Redundant
latent-coordinate directions are removed before the network is
represented, pattern by pattern,
by polynomial predicates in only $r_{\calA}$ effective variables.  The
exponential number of activation patterns appears inside a logarithm in
Proposition~\ref{prop:GJ}.}
\label{fig:pdim-roadmap}
\end{figure}

\proofstep{1}{Reduce to effective affine coordinates.}

The purpose of this step is to ensure that the number of variables entering
the polynomial-pattern theorem is the true affine rank, rather than the
possibly redundant latent-vector length $M$.

If $r_{\calA}=0$, Lemma~\ref{lem:rank-reduction} shows that
$\calF_{\Phi,\calA}$ is a singleton.  For any fixed pair $(\bmx,t)$, a
singleton realizes only one of the two labels, so the one-point condition in
\eqref{eq:pdim-labeling} fails.  Its pseudo-dimension is zero, and
the inequalities are immediate.

Assume henceforth that $r_{\calA}\ge1$.  By
Lemma~\ref{lem:rank-reduction}, every parameter vector in the affine image has
the form
\begin{equation*}
 \bmtheta(\bmmu)
 =\bmR\bmmu+\bma,
 \qquad
 \bmmu\in\R^{r_{\calA}}.
\end{equation*}
From this point onward, $\bmR$ denotes the fixed choice made before the
proof.  Although a basis of $\operatorname{Range}(\bmA)$ need not be unique,
no later choice of $\bmR$ may depend on $(\bmx,t,\bmmu)$ or on a labeling to
be realized.
Every scalar layer weight and bias is consequently an affine polynomial in
$\bmmu$.  Using the fixed ordering of the $P_\Phi$ dense parameter entries,
denote these coordinate
functions by
\[
 w_{\ell,j,k}(\bmmu)
 \quad\text{and}\quad
 b_{\ell,j}(\bmmu),
 \qquad
 \ell=1,\ldots,L_\Phi+1.
\]
Each has total degree at most one in $\bmmu$.

The one-unit calculation in
\eqref{eq:one-unit-boolean-preview} already exhibits the two ideas
needed below: fix a branch, on which the ReLU becomes polynomial, and then
join the branch clauses by one Boolean formula.  We now carry out those two
operations simultaneously for every hidden unit of the network.

\proofstep{2}{Define a formal polynomial forward pass for every activation pattern.}

The purpose of fixing a pattern is to replace every nonlinear ReLU branch by
one polynomial expression.  We will later use Boolean logic to combine the
finitely many patterns.

Let $n_1,\ldots,n_{L_\Phi}$ be the hidden widths, recall that $n_0=d$, and
recall $U_\Phi=\lvert\calU_\Phi\rvert$ from
\eqref{eq:hidden-unit-count}.
A global activation pattern is a binary array
\[
 \bmsigma=(\sigma_{\ell,j})
 \in\{0,1\}^{U_\Phi},
 \qquad
 1\le\ell\le L_\Phi,
 \quad
 1\le j\le n_\ell.
\]
The entry $\sigma_{\ell,j}=1$ will represent an active unit, and
$\sigma_{\ell,j}=0$ an inactive unit.

Fix one pattern $\bmsigma$.  We define candidate preactivations and activations
recursively without applying the nonlinear ReLU function.  At the input,
set
\begin{equation*}
 h_{0,k}^{\bmsigma}(\bmx,\bmmu):=x_k,
 \qquad k=1,\ldots,d.
\end{equation*}
For a hidden unit $(\ell,j)$, define
\[
 z_{\ell,j}^{\bmsigma}(\bmx,\bmmu)
 :=
 \sum_{k=1}^{n_{\ell-1}}
 w_{\ell,j,k}(\bmmu)
 h_{\ell-1,k}^{\bmsigma}(\bmx,\bmmu)
 +b_{\ell,j}(\bmmu).
\]
Its formal postactivation is
\[
 h_{\ell,j}^{\bmsigma}(\bmx,\bmmu)
 :=\sigma_{\ell,j}z_{\ell,j}^{\bmsigma}(\bmx,\bmmu).
\]
Thus an active formal gate transmits its candidate preactivation, while an
inactive formal gate returns the zero polynomial.  After the last hidden
layer, define the formal scalar output
\begin{equation*}
 o_{\bmsigma}(\bmx,\bmmu)
 :=
 \sum_{k=1}^{n_{L_\Phi}}
 w_{L_\Phi+1,1,k}(\bmmu)
 h_{L_\Phi,k}^{\bmsigma}(\bmx,\bmmu)
 +b_{L_\Phi+1,1}(\bmmu).
\end{equation*}
For a fixed pattern, every displayed object is now an ordinary polynomial in
$(\bmx,\bmmu)$.

\proofstep{3}{Fix the lifted instance and prove the parameter-degree bound.}

We now verify the degree condition in
Definition~\ref{def:boolean-polynomial}.  Once the lifted instance
$(\bmx,t)$ is fixed, $x_1,\ldots,x_d,t$ are constants, and only the effective
parameter vector $\bmmu$ remains variable.  Therefore
\[
 \deg_{\bmmu}h_{0,k}^{\bmsigma}\le0,
 \qquad
 \deg_{\bmmu}w_{\ell,j,k},
 \deg_{\bmmu}b_{\ell,j}\le1.
\]
It follows immediately that
\[
 \deg_{\bmmu}z_{1,j}^{\bmsigma},
 \deg_{\bmmu}h_{1,j}^{\bmsigma}\le1.
\]
Suppose for some $\ell\ge2$ that
\[
 \deg_{\bmmu}h_{\ell-1,k}^{\bmsigma}\le\ell-1
 \qquad(k=1,\ldots,n_{\ell-1}).
\]
Then every product in the layer satisfies
\[
 \deg_{\bmmu}\bigl(
 w_{\ell,j,k}(\bmmu)
 h_{\ell-1,k}^{\bmsigma}(\bmx,\bmmu)
 \bigr)
 \le1+(\ell-1)=\ell.
\]
Taking finite sums and adding the affine bias gives
\[
 \deg_{\bmmu}z_{\ell,j}^{\bmsigma},
 \deg_{\bmmu}h_{\ell,j}^{\bmsigma}
 \le\ell.
\]
The last inequality uses only that
$h_{\ell,j}^{\bmsigma}=\sigma_{\ell,j}z_{\ell,j}^{\bmsigma}$ and
$\sigma_{\ell,j}\in\{0,1\}$ is fixed.  At the output layer, one further
multiplication by an affine weight gives
\[
 \deg_{\bmmu}o_{\bmsigma}\le L_\Phi+1.
\]
Since $t$ is fixed, subtracting it does not increase the parameter degree:
\[
 \deg_{\bmmu}(o_{\bmsigma}-t)\le L_\Phi+1.
\]
Collecting the preceding bounds, for every fixed lifted instance $(\bmx,t)$
and activation pattern $\bmsigma$, the relevant polynomials satisfy
\begingroup
\renewcommand{\arraystretch}{1.14}
\setlength{\arraycolsep}{9pt}
\begin{equation*}
\begin{array}{@{}l c@{}}
\toprule
\text{polynomial in }\bmmu
&\text{degree bound}\\
\midrule
h_{0,k}^{\bmsigma}=x_k
&\le0\\
w_{\ell,j,k}(\bmmu),\ b_{\ell,j}(\bmmu)
&\le1\\
z_{\ell,j}^{\bmsigma},\ h_{\ell,j}^{\bmsigma}
&\le\ell\\
o_{\bmsigma}-t
&\le L_\Phi+1\\
\bottomrule
\end{array}
\end{equation*}
\endgroup

\proofstep{4}{Encode the exact forward-pass pattern, including zero preactivations.}

This step combines the separate polynomial pieces into one exact Boolean
description and checks the boundary case of a zero preactivation.

We use the deterministic convention
\[
 \text{active if the preactivation is nonnegative},
 \qquad
 \text{inactive if it is negative}.
\]
The nonnegative/strict-negative convention assigns every preactivation to
a unique branch.  Since $\varrho(0)=0$, treating a zero
preactivation as active does not alter the numerical network output.

For a bit $a\in\{0,1\}$ and a Boolean statement $E$, define the literal
\[
 \Lambda_a(E)
 :=
 \begin{cases}
 E,&a=0,\\
 \neg E,&a=1.
 \end{cases}
\]
For a fixed pattern $\bmsigma$, use the strict-positivity statement
\[
 -z_{\ell,j}^{\bmsigma}(\bmx,\bmmu)>0
\]
at each hidden unit.  This statement is true exactly when the formal
preactivation is negative.  Define the consistency formula
\begin{equation}\label{eq:pattern-consistency}
 \mathsf{C}_{\bmsigma}(\bmx,\bmmu)
 :=
 \bigwedge_{(\ell,j)\in\calU_\Phi}
 \Lambda_{\sigma_{\ell,j}}\bigl(
 -z_{\ell,j}^{\bmsigma}(\bmx,\bmmu)>0
 \bigr).
\end{equation}
Indeed, when $\sigma_{\ell,j}=0$, the corresponding literal says
$z_{\ell,j}^{\bmsigma}<0$.  When $\sigma_{\ell,j}=1$, its negation says
$z_{\ell,j}^{\bmsigma}\ge0$.  Thus every consistency leaf is of the required
strict-positivity form; weak nonnegativity is supplied by Boolean negation.

We verify carefully that this formula is equivalent to $\bmsigma$ being the
actual activation pattern.

First suppose that $\mathsf{C}_{\bmsigma}(\bmx,\bmmu)$ holds.  At the first hidden
layer, $z_{1,j}^{\bmsigma}$ is exactly the actual preactivation because both are
computed from the same input and generated first-layer parameters.  If
$\sigma_{1,j}=1$, consistency gives $z_{1,j}^{\bmsigma}\ge0$, and hence
\[
 \varrho(z_{1,j}^{\bmsigma})
 =z_{1,j}^{\bmsigma}
 =h_{1,j}^{\bmsigma}.
\]
If $\sigma_{1,j}=0$, consistency gives $z_{1,j}^{\bmsigma}<0$, and hence
\[
 \varrho(z_{1,j}^{\bmsigma})
 =0
 =h_{1,j}^{\bmsigma}.
\]
Thus the formal and actual first-layer states agree.  If they agree through
layer $\ell-1$, their layer-$\ell$ affine preactivations also agree.  The same
two-case argument then shows that their postactivations agree in layer
$\ell$.  Induction proves agreement through the last hidden layer, and the
formal output $o_{\bmsigma}$ is therefore the actual scalar network output.

Conversely, start from the actual forward pass for a fixed
$(\bmx,\bmmu)$ and define $\sigma_{\ell,j}=1$ exactly when its actual
preactivation is nonnegative.  The same induction shows that the associated
formal preactivations equal the actual ones.  They consequently satisfy the
appropriate weak or strict inequalities in
\eqref{eq:pattern-consistency}.  Hence the actual forward pass
produces exactly one consistent pattern.

For every pattern, combine consistency and the output comparison into the
clause
\[
 \mathsf{D}_{\bmsigma}(\bmx,t,\bmmu)
 :=\mathsf{C}_{\bmsigma}(\bmx,\bmmu)
 \wedge\bigl(o_{\bmsigma}(\bmx,\bmmu)-t>0\bigr).
\]
For an arbitrary threshold $t\in\R$, the exact subgraph predicate is
\[
 \Phi_{\bmtheta(\bmmu)}(\bmx)> t
 \quad\Longleftrightarrow\quad
 \bigvee_{\bmsigma\in\{0,1\}^{U_\Phi}}
 \mathsf{D}_{\bmsigma}(\bmx,t,\bmmu).
\]
To connect this predicate explicitly to
Definition~\ref{def:boolean-polynomial}, introduce one propositional variable
$Z_{\bmsigma,\ell,j}$ for the consistency atom associated with
$(\bmsigma,\ell,j)$, and one variable $Z_{\bmsigma,\mathrm{out}}$ for its
output atom.  For a propositional variable $Z$, use the same literal notation
\[
 \Lambda_a(Z)
 :=
 \begin{cases}
 Z,&a=0,\\
 \neg Z,&a=1.
 \end{cases}
\]
First define the formula for one pattern by
\[
 \mathfrak{c}_{\bmsigma}
 :=\biggl(
 \bigwedge_{(\ell,j)\in\calU_\Phi}
 \Lambda_{\sigma_{\ell,j}}(Z_{\bmsigma,\ell,j})
 \biggr)\wedge Z_{\bmsigma,\mathrm{out}}.
\]
The single fixed formula for the whole network is then
\begin{equation*}
 \mathfrak{B}_{\Phi}
 :=\bigvee_{\bmsigma\in\{0,1\}^{U_\Phi}}
 \mathfrak{c}_{\bmsigma}.
\end{equation*}
Fix lexicographic orderings of the pattern set and of $\calU_\Phi$, together
with the binary bracketings stipulated above.  Enumerate the leaf occurrences
of $\mathfrak{B}_\Phi$ in that fixed order as
$1,\ldots,s_\Phi$, where
\[
 s_\Phi:=(U_\Phi+1)2^{U_\Phi}.
\]
Associate with each consistency-leaf occurrence its joint polynomial
$-z_{\ell,j}^{\bmsigma}(\bmx,\bmmu)$ and with each output-leaf occurrence its
joint polynomial $o_{\bmsigma}(\bmx,\bmmu)-t$.  Because
$(\Phi,\calA,\bmR)$ and all orderings have already been fixed, this produces
one indexed list
\[
 p_1,\ldots,p_{s_\Phi}
\]
before $(\bmx,t,\bmmu)$ is given.  Algebraically identical polynomials arising
at different leaf occurrences remain separately indexed.
At a triple $(\bmx,t,\bmmu)$, substitute the truth value
\[
 \one_{\{-z_{\ell,j}^{\bmsigma}(\bmx,\bmmu)>0\}}
\]
for $Z_{\bmsigma,\ell,j}$.  Substitute
\[
 \one_{\{o_{\bmsigma}(\bmx,\bmmu)-t>0\}}
\]
for $Z_{\bmsigma,\mathrm{out}}$.  The preceding equivalence now says exactly
that the indicator of the subgraph predicate equals the evaluation of
$\mathfrak{B}_{\Phi}$ at these atom truth values.

At every fixed triple $(\bmx,t,\bmmu)$, exactly one consistency formula is
true: its pattern declares a zero preactivation active and every negative
preactivation inactive.  Thus exactly one pattern clause can survive the
consistency tests, and the disjunction compares the output polynomial from
that unique formal branch with $t$.  This also verifies exactness on
activation boundaries, not merely at points where all preactivations are
nonzero.

The Boolean map associated with this fixed formula is
\[
 \mathsf{B}_\Phi:=\mathsf{B}_{\mathfrak{B}_\Phi}:\{0,1\}^{s_\Phi}\to\{0,1\},
\]
where $s_\Phi$ is the number of leaf occurrences introduced above and counted
explicitly in Step~5.  After
the truth values of all the $Z$-variables are supplied, $\mathsf{B}_\Phi$ returns $1$
exactly when at least one complete pattern clause is true.  The formula and
its map depend only on the fixed finite index set of the architecture.  The
atom polynomials also depend on the fixed affine map, but neither the formula
nor the polynomial list is selected after seeing $(\bmx,t,\bmmu)$.

Equivalently, define the network truth set
\[
 \calS_{\Phi,\calA}
 :=\left\{(\bmx,t,\bmmu)\in\R^d\times\R\times\R^{r_{\calA}}:
 \Phi_{\bmtheta(\bmmu)}(\bmx)>t\right\}.
\]
For each fixed lifted instance $(\bmx,t)$, its parameter section is
\[
 (\calS_{\Phi,\calA})_{\bmx,t}
 =\left\{\bmmu\in\R^{r_{\calA}}:
 \Phi_{\bmtheta(\bmmu)}(\bmx)>t\right\}.
\]
The displayed disjunction gives a finite semialgebraic description of this
section.  The network output need not be one global polynomial in $\bmmu$:
each activation pattern contributes a polynomial piece, and the fixed
disjunction combines all pieces into one exact description of membership in
$\calS_{\Phi,\calA}$.  The formula is fixed once $(\Phi,\calA)$ is fixed;
it is not chosen after seeing $(\bmx,t)$ or $\bmmu$.

\proofstep{5}{Count predicates and apply Proposition~\ref{prop:GJ}.}

It remains to identify the two syntactic quantities required by
Definition~\ref{def:boolean-polynomial}: the number of atom occurrences and
their maximum degree in $\bmmu$.

The pattern set has cardinality
\[
 \left\lvert\{0,1\}^{U_\Phi}\right\rvert=2^{U_\Phi}.
\]
Each pattern contributes $U_\Phi$ consistency-atom occurrences and one output
comparison.  Hence the complete formula has exactly
\[
 s_\Phi=(U_\Phi+1)2^{U_\Phi}
\]
leaf occurrences, as introduced in Step~4.  Set
\[
 d_{\mathrm{net}}:=L_\Phi+1.
\]
By Step~3, every atom has degree at most $d_{\mathrm{net}}$ in $\bmmu$.
Since every hidden
layer is nonempty,
\[
 1\le L_\Phi\le U_\Phi.
\]
It follows that
\[
 d_{\mathrm{net}}\le U_\Phi+1.
\]
The occurrence count is syntactic: if the same preactivation polynomial
appears in two different pattern clauses, both appearances are counted, as
required by Definition~\ref{def:boolean-polynomial}.
Applying Proposition~\ref{prop:GJ} with $r=r_{\calA}$ yields
\[
 \Pdim(\calF_{\Phi,\calA})
 \le2r_{\calA}\log_2\left(4ed_{\mathrm{net}}s_\Phi\right).
\]
Substituting the two preceding bounds gives
\[
 \Pdim(\calF_{\Phi,\calA})
 \le2r_{\calA}\log_2\left(4e(U_\Phi+1)^2 2^{U_\Phi}\right).
\]
Expanding this logarithm yields
\[
 \Pdim(\calF_{\Phi,\calA})
 \le2r_{\calA}\left[U_\Phi+\log_2(4e)
 +2\log_2(U_\Phi+1)\right].
\]
Because the architecture has at least one hidden unit, $U_\Phi\ge1$.  For
$U_\Phi\ge1$,
\[
 \log_2(4e)<4,
 \qquad
 \log_2(U_\Phi+1)\le U_\Phi.
\]
The bracket is therefore bounded by
\[
 3U_\Phi+4\le4(U_\Phi+1).
\]
Consequently,
\[
 \Pdim(\calF_{\Phi,\calA})
 \le8r_{\calA}(U_\Phi+1).
\]
Together with \eqref{eq:vc-below-pdim}, this proves the first two
inequalities in \eqref{eq:affine-pdim}.

It remains to express the result using the parameter-slot budget.  By the
fully connected architecture convention, every hidden unit has its own bias
slot.  These account for $U_\Phi$ distinct slots, and the scalar output layer
has one further bias slot.  All of these slots are counted in
$P_\Phi$, even when their assigned values are fixed or zero.  Therefore
\[
 U_\Phi+1\le P_\Phi,
\]
and the last inequality in \eqref{eq:affine-pdim} follows.

Finally, \eqref{eq:effective-rank} and $P_\Phi\le P$ give
\[
 r_{\calA}\le\min\{M,P_\Phi\}
 \le\min\{M,P\},
\]
and hence
\[
 8r_{\calA}P_\Phi
 \le8P\min\{M,P\}.
\]
This is \eqref{eq:affine-pdim-budget}.
\end{proof}

The preceding proof also explains the form of the result.  The factor
$r_{\calA}$ is the number of independent real directions in which the
generated parameter vector can move.  The factor $U_\Phi+1$, and hence the
coarser factor $P_\Phi$, measures the amount of piecewise-polynomial
computation performed by the fixed deployed network.  Although there are
$2^{U_\Phi}$ possible activation patterns, this quantity occurs inside the
logarithm in Proposition~\ref{prop:GJ}, since
$\log_2(2^{U_\Phi})=U_\Phi$.  Thus the enumeration of activation patterns
does not produce an exponential pseudo-dimension bound.

\begin{remark}[Piecewise-polynomial activations]
For context, although this paper uses the ReLU activation $\varrho$
throughout, the same binary-predicate argument extends to any fixed activation
$\rho$ with finitely many polynomial pieces.  More precisely, suppose that
$J_\rho\in\N^+$ intervals partition
$\R$, with every breakpoint assigned to exactly one adjacent interval, and
that on each interval $\rho$ agrees with a polynomial of degree at most an
integer $q_\rho\ge0$.

Consider an architecture with $L$ hidden layers, $U$ hidden units, and
positive effective affine parameter dimension $r\ge1$.  Define
\[
 D_{\rho,0}:=0,
 \qquad
 D_{\rho,\ell}
 :=q_\rho\bigl(D_{\rho,\ell-1}+1\bigr),
 \quad \ell=1,\ldots,L,
 \qquad
 D_{\rho,*}:=D_{\rho,L}+1.
\]
For a fixed global piece pattern, induction over the layers shows that a
formal postactivation in layer $\ell$ has parameter degree at most
$D_{\rho,\ell}$: forming its affine preactivation raises the degree bound
from $D_{\rho,\ell-1}$ to at most $D_{\rho,\ell-1}+1$, and substitution into
a polynomial piece of degree at most $q_\rho$ gives the displayed recursion.
The interval-consistency atoms use the preactivations, while the final output
atom has degree at most $D_{\rho,L}+1$.  Hence all atom degrees are bounded
by $D_{\rho,*}$.  When $q_\rho\ge1$, the recursion gives
$D_{\rho,\ell}\ge D_{\rho,\ell-1}+1$, so $D_{\rho,*}$ dominates every
preactivation degree from every layer.  When $q_\rho=0$, every formal
postactivation is constant
on its selected piece and every preactivation atom has degree at most one,
which is again covered by $D_{\rho,*}=1$.

There are at most $J_\rho^U$ global piece patterns.  For each unit, membership
of its preactivation in the selected interval requires at most two
strict-positivity atom occurrences, and the output comparison requires one
additional occurrence.  Consequently,
the same argument gives
\[
 \Pdim(\calF_{\Phi,\calA})
 \le
 2r\log_2\bigl(
 4eD_{\rho,*}(2U+1)J_\rho^U
 \bigr).
\]
The fixed breakpoint convention keeps the Boolean branch description exact
even when a preactivation equals a breakpoint.  A formal proof would repeat
Steps~2 through 5 above and is omitted because no theorem in this paper uses a
non-ReLU activation.  For ReLU, the specialized argument above gives the
simpler and sharper estimate.  If the effective affine dimension is zero,
the realized family is a singleton and has pseudo-dimension zero, as in
Step~1.
\end{remark}

\subsection{From pseudo-dimension to a uniform lower bound}
\label{sec:bump-packing}

This completes the capacity part of the lower bound.  We now show that any
real-valued function class with this capacity bound must fail to approximate
at least one member of the H\"older ball.  The mechanism is a family of
disjoint tents whose signs can be chosen independently: uniformly accurate
approximation of every such choice would force the zero-threshold class to
realize too many binary labelings.  The construction below yields an explicit
constant.  A related VC-dimension argument appears in
\cite[Theorem~2.4]{shijun:optimal:rate:in:width:and:depth}; the proof given here
is self-contained and does not invoke that result.

\begin{lemma}[H\"older bump packing]\label{lem:holder-bump}
Let $d\in\N^+$ and $0<\alpha\le1$, and let
$\calF\subseteq C(\R^d)$ be a real-valued function class.  Let
$V\in[1,\infty)$ satisfy
\[
 \VCdim(\calF)\le V.
\]
Then
\begin{equation}\label{eq:bump-lower}
 \sup_{f\in\calH_d^\alpha}\inf_{g\in\calF}
 \lVert f-g\rVert_{L^\infty([0,1]^d)}
 \ge \widetilde{c}_\alpha V^{-\alpha/d}.
\end{equation}
Here, as above, $\widetilde{c}_\alpha:=4^{-(\alpha+1)}2^{-\alpha}$.
\end{lemma}

\begin{proof}
We place disjoint H\"older tents on a regular grid and choose one binary label
at each center.  The label determines whether the corresponding tent is
positive or negative.  Every resulting sum remains in the unit H\"older
ball.  If every such target admitted an approximation with error less than
half its absolute value at the centers, the zero-threshold class would
realize every binary labeling of those centers.  Choosing more centers than
$V$ then gives a contradiction.

If $\calF$ is empty, the inner infimum in \eqref{eq:bump-lower} is
$+\infty$, so the result is immediate.  Assume henceforth that $\calF$ is
nonempty.  As specified in Section~\ref{sec:notation}, the displayed
$L^\infty([0,1]^d)$ norm is the pointwise uniform norm, and the error bound may
therefore be evaluated at every grid center.

\proofstep{1}{Choose more centers than the VC dimension at threshold zero.}

Set
\[
 R:=\lfloor V^{1/d}\rfloor+1,
 \qquad
 h:=\frac{1}{4R}.
\]
Since $V\ge1$, we have $R\ge2$.  Moreover,
$R=\lfloor V^{1/d}\rfloor+1>V^{1/d}$, so
\[
 R^d>V\ge\VCdim(\calF).
\]
Introduce the center index set
\[
 \calJ:=\{1,\ldots,R\}^d,
 \qquad \lvert\calJ\rvert=R^d.
\]
For every multi-index $\bmi=(i_1,\ldots,i_d)\in\calJ$, define
\begin{equation*}
 \bmx_{\bmi}
 :=\biggl(
  \frac{2i_1-1}{2R},\ldots,
  \frac{2i_d-1}{2R}
 \biggr).
\end{equation*}
There are exactly $\lvert\calJ\rvert=R^d$ centers.  Each center has distance at
least $1/(2R)=2h$ from the boundary of the cube in the maximum norm.  Two distinct
centers have distance at least $1/R=4h$ in that norm.  It follows that the
closed balls of radius $h$ in the maximum norm around the centers lie inside
$[0,1]^d$ and
that any two such balls are separated by distance at least $2h$ in the
maximum norm.  Since $\lVert\bmz\rVert_2\ge\lVert\bmz\rVert_\infty$, their Euclidean separation is
also at least $2h$.

\proofstep{2}{Define one tent at each center and verify its H\"older estimate.}

For each $\bmi\in\calJ$, define $\psi_{\bmi}:[0,1]^d\to\R$ by
\begin{equation*}
 \psi_{\bmi}(\bmx)
 :=\frac{1}{2}
 \left(h-\lVert\bmx-\bmx_{\bmi}\rVert_\infty\right)_+^\alpha.
\end{equation*}
The set on which $\psi_{\bmi}$ is nonzero lies in the open ball in the maximum norm
of radius $h$ centered at $\bmx_{\bmi}$, and its closed support lies in the
corresponding closed ball.  The supports are therefore contained in the cube
and are pairwise separated.

For $0<\alpha\le1$ and $u,v\in\R$, one has
\begin{equation}\label{eq:positive-part-holder}
 \lvert u_+^\alpha-v_+^\alpha\rvert\le\lvert u-v\rvert^\alpha.
\end{equation}
To see this, first suppose that $u\ge v\ge0$.  The subadditivity of
$t\mapsto t^\alpha$ gives
\[
 u^\alpha
 =[v+(u-v)]^\alpha
 \le v^\alpha+(u-v)^\alpha.
\]
For completeness, this subadditivity is elementary.  If $a,b\ge0$ and
$a+b>0$, put $\lambda:=a/(a+b)$.  Since $0<\alpha\le1$ and
$0\le\lambda\le1$, one has
$\lambda^\alpha\ge\lambda$ and
$(1-\lambda)^\alpha\ge1-\lambda$.  Multiplication by
$(a+b)^\alpha$ yields
\[
 a^\alpha+b^\alpha
 =(a+b)^\alpha\left(\lambda^\alpha+(1-\lambda)^\alpha\right)
 \ge(a+b)^\alpha.
\]
The case $a=b=0$ is immediate.
The case $v\ge u\ge0$ is symmetric.  If both numbers are nonpositive, both
positive parts vanish.  If exactly one is nonpositive, the positive part of
the other is at most $\lvert u-v\rvert$, which proves the remaining cases.

The reverse triangle inequality for the maximum norm follows by applying the
ordinary triangle inequality in both directions, and gives
\begin{equation*}
 \left\lvert
  \lVert\bmx-\bmx_{\bmi}\rVert_\infty
  -\lVert\bmy-\bmx_{\bmi}\rVert_\infty
 \right\rvert
 \le\lVert\bmx-\bmy\rVert_\infty
 \le\lVert\bmx-\bmy\rVert_2.
\end{equation*}
Combining this estimate with
\eqref{eq:positive-part-holder} yields
\begin{equation*}
 \lvert\psi_{\bmi}(\bmx)-\psi_{\bmi}(\bmy)\rvert
 \le\frac{1}{2}\lVert\bmx-\bmy\rVert_2^\alpha.
\end{equation*}

\proofstep{3}{Verify that every binary choice produces a target in the H\"older ball.}

For each binary vector indexed by the centers,
\[
 \bmv=(v_{\bmi})_{\bmi\in\calJ}
 \in\{0,1\}^{\calJ},
\]
define $f_{\bmv}:[0,1]^d\to\R$ by
\begin{equation*}
 f_{\bmv}(\bmx)
 :=\sum_{\bmi\in\calJ}
 (2v_{\bmi}-1)\psi_{\bmi}(\bmx).
\end{equation*}
The set of binary choices has cardinality
\[
 \left\lvert\{0,1\}^{\calJ}\right\rvert
 =2^{\lvert\calJ\rvert}=2^{R^d}.
\]
At any point of the cube, at most one tent is nonzero.  We verify the global
H\"older estimate by considering the following mutually exclusive
possibilities for $\bmx$ and $\bmy$.
\begin{enumerate}[label={\textup{(\roman*)}}]
 \item If no tent is nonzero at either point, then
 $f_{\bmv}(\bmx)=f_{\bmv}(\bmy)=0$.

 \item If the same tent is nonzero at both points, or if exactly one of the two
 points has any nonzero tent value, then the signed sum at both points reduces
 to the same signed tent, with value zero at the inactive point in the latter
 case.  The estimate from Step~2 therefore gives
 \[
  \lvert f_{\bmv}(\bmx)-f_{\bmv}(\bmy)\rvert
  \le\frac{1}{2}\lVert\bmx-\bmy\rVert_2^\alpha.
 \]

 \item If two different tents are nonzero at the two points, then
 \[
  \lvert f_{\bmv}(\bmx)-f_{\bmv}(\bmy)\rvert
  \le\frac{1}{2}h^\alpha+\frac{1}{2}h^\alpha
  =h^\alpha.
 \]
 Their Euclidean distance is at least $2h$, and hence
 \[
 h^\alpha
 \le(2h)^\alpha
 \le\lVert\bmx-\bmy\rVert_2^\alpha.
 \]
 This estimate from the triangle inequality also covers the case in which the two
 active tents carry opposite signs.
\end{enumerate}
Thus every $f_{\bmv}$ is a function on $[0,1]^d$ with H\"older constant at
most one.  Moreover,
\[
 \lVert f_{\bmv}\rVert_{L^\infty([0,1]^d)}
 \le\frac{1}{2}h^\alpha\le1.
\]
Therefore
\[
 f_{\bmv}\in\calH_d^\alpha
 \qquad\text{for every }\bmv\in\{0,1\}^{\calJ}.
\]

\proofstep{4}{Show that accurate approximation realizes every binary labeling.}

At every center, all tents except one vanish and the remaining tent attains
its maximum.  Hence
\begin{equation*}
 f_{\bmv}(\bmx_{\bmi})
 =\frac{1}{2}h^\alpha(2v_{\bmi}-1).
\end{equation*}

Write
\begin{equation*}
 E
 :=\sup_{f\in\calH_d^\alpha}
   \inf_{g\in\calF}
   \lVert f-g\rVert_{L^\infty([0,1]^d)}.
\end{equation*}
Suppose, for contradiction, that $E<h^\alpha/4$.  For every binary vector
$\bmv$, the definition of the infimum gives a function
$g_{\bmv}\in\calF$ satisfying
\[
 \lVert f_{\bmv}-g_{\bmv}\rVert_{L^\infty([0,1]^d)}
 <\frac{1}{4}h^\alpha.
\]
No attainment assumption is used: the strict inequality
$\inf_{g\in\calF}\lVert f_{\bmv}-g\rVert_{L^\infty([0,1]^d)}<h^\alpha/4$
is enough to select such a $g_{\bmv}$.

If $v_{\bmi}=1$, then
\[
 g_{\bmv}(\bmx_{\bmi})
 >\frac{1}{2}h^\alpha-\frac{1}{4}h^\alpha
 =\frac{1}{4}h^\alpha>0.
\]
If $v_{\bmi}=0$, then
\[
 g_{\bmv}(\bmx_{\bmi})
 <-\frac{1}{2}h^\alpha+\frac{1}{4}h^\alpha
 =-\frac{1}{4}h^\alpha<0.
\]
Thus the approximation margin is strict on both sides of the fixed threshold
$0$, and
\begin{equation*}
 \one_{\{g_{\bmv}(\bmx_{\bmi})>0\}}
 =v_{\bmi}
\end{equation*}
for every $\bmi\in\calJ$.  As $\bmv$ ranges over
$\{0,1\}^{\calJ}$, every binary labeling of the $R^d$ centers is realized by
$\operatorname{Thr}_0(\calF)$.  Hence those centers are shattered by the
zero-threshold class.  By \eqref{eq:real-vcdim}, this would give
$R^d\le\VCdim(\calF)$, contradicting
$R^d>V\ge\VCdim(\calF)$.

We conclude that
\[
 E\ge\frac{1}{4}h^\alpha.
\]

\proofstep{5}{Express the tent height in terms of $V$.}

By the definition of $R$,
\[
 R
 \le V^{1/d}+1
 \le2V^{1/d},
\]
where the last inequality uses $V\ge1$.
Since $h=1/(4R)$,
\begin{align*}
 E
 &\ge\frac{1}{4}h^\alpha
 =4^{-(\alpha+1)}R^{-\alpha}\\
 &\ge
 4^{-(\alpha+1)}2^{-\alpha}V^{-\alpha/d}.
\end{align*}
This is \eqref{eq:bump-lower} with the stated value of
$\widetilde{c}_\alpha$.
\end{proof}

\Needspace{8\baselineskip}
\subsection{Completion of the minimax lower bound}
\label{sec:lower-completion}

\begin{proof}[Proof of Theorem~\ref{thm:lower}]
We now combine the two independent parts of the argument.  The
pseudo-dimension estimate controls the binary capacity of each fixed pair
$(\Phi,\calA)$, while the bump lemma turns that capacity restriction into a
quantitative approximation obstruction.  The resulting
constant is uniform over all admissible pairs, so the estimate survives the
two outer infima in \eqref{eq:minimax}.

The outer infimum in \eqref{eq:minimax} is taken only over
architectures whose input dimension is $d$.  Under the convention of
Section~\ref{sec:notation}, such an architecture has at least one hidden unit
and hence at least $(d+1)+2=d+3$ parameter slots.  Thus the admissible set is
empty when $P<d+3$.  More generally, if
\[
 \left\{\Phi\in\Arch_{\mathrm{par}}(P):d_\Phi=d\right\}
 =\varnothing,
\]
then the convention $\inf\varnothing=+\infty$ makes the theorem immediate.
We may therefore assume that this $d$-input admissible set is nonempty.

Fix an arbitrary architecture
\[
 \Phi\in\Arch_{\mathrm{par}}(P)
 \qquad\text{with}\qquad d_\Phi=d,
\]
and an arbitrary map $\calA\in\aff(M,P_\Phi)$.  The image $\calA(\R^M)$, and
hence the realized family $\calF_{\Phi,\calA}$, is nonempty.  Under the
architectural convention fixed in
Section~\ref{sec:notation}, $L_\Phi\in\N^+$ and every hidden width is positive.
Thus $\Phi$ has at least one hidden unit, and
Lemma~\ref{lem:affine-pdim} applies.  The family
\[
 \calF_{\Phi,\calA}
 =\left\{\Phi_{\bmtheta}:
 \bmtheta\in\calA(\R^M)\right\}
\]
consists of continuous functions on $\R^d$.  Set
\[
 V_0:=8P\min\{M,P\}.
\]
Because $M,P\in\N^+$, we have $V_0\ge1$.  Combining
\eqref{eq:vc-below-pdim} and \eqref{eq:affine-pdim-budget} gives
\[
 \VCdim(\calF_{\Phi,\calA})
 \le\Pdim(\calF_{\Phi,\calA})
 \le V_0.
\]
Thus Lemma~\ref{lem:holder-bump} applies with $V=V_0$; the
actual VC or pseudo-dimension need not be positive.

By \eqref{eq:realized-family}, minimizing over functions in
$\calF_{\Phi,\calA}$ is equivalent to minimizing over complete parameters
$\bmtheta\in\calA(\R^M)$ or, equivalently, over latent vectors $\bmxi\in\R^M$.
Thus the pairwise error in
\eqref{eq:pair-risk} satisfies
\[
 \calR_{\alpha,d}(\Phi,\calA)
 =
 \sup_{f\in\calH_d^\alpha}
 \inf_{g\in\calF_{\Phi,\calA}}
 \lVert f-g\rVert_{L^\infty([0,1]^d)}.
\]
Lemma~\ref{lem:holder-bump} therefore gives
\[
 \calR_{\alpha,d}(\Phi,\calA)
 \ge
 \widetilde{c}_\alpha
 V_0^{-\alpha/d}
 =
 \widetilde{c}_\alpha 8^{-\alpha/d}
 [P\min\{M,P\}]^{-\alpha/d}.
\]
Thus the arbitrary pair $(\Phi,\calA)$ satisfies
\begin{equation*}
 \calR_{\alpha,d}(\Phi,\calA)
 \ge
 \underline{c}_{\alpha,d}
 [P\min\{M,P\}]^{-\alpha/d},
\end{equation*}
where one may take
\begin{equation*}
 \underline{c}_{\alpha,d}
 :=4^{-(\alpha+1)}2^{-\alpha}8^{-\alpha/d}.
\end{equation*}

Since $(\Phi,\calA)$ was arbitrary, the preceding estimate holds for every
$\Phi\in\Arch_{\mathrm{par}}(P)$ with $d_\Phi=d$ and every
$\calA\in\aff(M,P_\Phi)$.  Hence, for each such $\Phi$,
\[
 \inf_{\calA\in\aff(M,P_\Phi)}
 \calR_{\alpha,d}(\Phi,\calA)
 \ge \underline{c}_{\alpha,d}
 [P\min\{M,P\}]^{-\alpha/d}.
\]
Taking the infimum over
$\{\Phi\in\Arch_{\mathrm{par}}(P):d_\Phi=d\}$ and using
\eqref{eq:minimax} proves \eqref{eq:main-lower}.
\end{proof}

\section{Conclusion}\label{sec:conclusion}

The matching upper and lower bounds determine the sharp approximation rate of
fully connected ReLU networks with affine latent parameterizations.  Although
the minimax problem ranges over all architectures with input dimension $d$ in
$\Arch_{\mathrm{par}}(P)$, the upper bound is already attained by networks
whose width depends only on $d$.  More precisely, under the assumptions of
Theorem~\ref{thm:upper}, one fixed pair
$(\Phi,\calA)$ satisfies
\[
 \left\lVert f-\Phi_{\calA(\bmxi_f)}\right\rVert_{L^\infty([0,1]^d)}
 \le(d+2)\omega_f\bigl(
 C_d[P\min\{M,P\}]^{-1/d}\bigr),
\]
for every $f\in C([0,1]^d)$, with a suitable target-dependent latent vector
$\bmxi_f$.  For the unit H\"older ball, combining this construction with the
matching lower bound gives
\[
 \calE_{\alpha,d}(M,P)
 \asymp_{\alpha,d}
 [P\min\{M,P\}]^{-\alpha/d}.
\]
In particular, for fixed $M_0\ge4$, the sharp rate is $P^{-\alpha/d}$ once
$P\ge\max\{M_0,P_d^\star\}$.  Thus a latent space of constant dimension can
already yield a vanishing worst-case error, while the fully saturated regime
$M\ge P$ yields the faster rate $P^{-2\alpha/d}$.

This rate captures a tradeoff between latent information and decoding
capacity.  The $M$ coordinates of $\bmxi_f$ carry target-dependent
information, while the fixed architecture uses its $P_\Phi\le P$ parameter
slots and budget-dependent depth to decode and route that information.  The
construction relies on exact real arithmetic for the latent coordinates, a
discontinuous target-to-latent encoding, and a decoder whose size grows with
the budget but remains independent of the target.  Finite-precision coding,
stability, and efficient optimization lie outside the scope of this
expressivity result.  Establishing sharp rates under finite-precision or
bounded-weight restrictions, or under explicit stability and optimization
requirements, is a natural direction for future work.  A second direction is
to extend the joint analysis to nonlinear generators under explicit
complexity and regularity budgets.

\begingroup
\small
\setlength{\bibsep}{5pt}
\bibliographystyle{plainnat}
\bibliography{references}
\endgroup

\end{document}